\documentclass[11pt]{article}

\usepackage{amsmath} 
\usepackage{amsthm} 
\usepackage{amsfonts} 
\usepackage{amssymb} 
\usepackage{stmaryrd} 
\usepackage{mathtools} 
\usepackage{thm-restate} 

\usepackage[style=alphabetic, maxnames=100, maxalphanames=100]{biblatex} 
\DeclareFieldFormat{postnote}{\mknormrange{#1}} 

\usepackage{xifthen} 

\usepackage{xcolor} 
\usepackage{bm} 
\usepackage{bbm} 

\usepackage{array} 
\usepackage[inline]{enumitem} 

\usepackage{microtype} 
\usepackage{lmodern} 

\usepackage{pdflscape} 
\usepackage{afterpage} 

\usepackage{tikz} 
\usetikzlibrary{cd} 
\usetikzlibrary{calc} 

\usepackage[in]{fullpage} 

\usepackage[hypcap=true]{subcaption} 

\usepackage[%
  plainpages=false,
  pdfpagelabels,
  colorlinks,
  citecolor=black,
  linkcolor=black,
  urlcolor=black,
  filecolor=black,
  bookmarksopen=false
]{hyperref} 
\usepackage[all]{hypcap} 
\hypersetup{hypertexnames=false}

\makeatletter
\newcommand*{\newletterthm@internal}{}
\newcommand*{\newletterthm}[1]{%
  \def\newletterthm@name{#1}%
  \renewcommand*{\newletterthm@internal}[1][]{%
    \ifthenelse{\isempty{##1}}{%
      \expandafter\expandafter\expandafter\newtheorem%
      \expandafter\expandafter\expandafter{%
        \expandafter\newletterthm@name%
        \expandafter}%
      \expandafter{%
        \newletterthm@text}%
      \expandafter\renewcommand%
      \expandafter*%
      \expandafter{%
        \csname the#1\endcsname}{\Alph{#1}}%
    }{%
      \expandafter\expandafter\expandafter\newtheorem%
      \expandafter\expandafter\expandafter{%
        \expandafter\newletterthm@name%
        \expandafter}%
      \expandafter{%
        \newletterthm@text}[##1]%
    \expandafter\renewcommand%
        \expandafter*%
        \expandafter{%
          \csname the#1\endcsname}{\csname the##1\endcsname.\Alph{#1}}%
    }%
  }%
  \newletterthm@newthm%
}

\newcommand*{\newletterthm@newthm}[2][]{%
  \ifthenelse{\isempty{#1}}{%
    \def\newletterthm@text{#2}%
    \newletterthm@internal%
  }{%
    \expandafter\newtheorem\expandafter{\newletterthm@name}[#1]{#2}%
  }%
}
\makeatother

\newtheoremstyle{thmstyle}
  {\medskipamount}
  {\smallskipamount}
  {\slshape}
  {0pt}
  {\bfseries}
  {.}
  { }
  {\thmname{#1}\thmnumber{ #2}{\normalfont\thmnote{ (#3)}}}

\newtheoremstyle{plainstyle}
  {\medskipamount}
  {\smallskipamount}
  {\rmfamily}
  {0pt}
  {\bfseries}
  {.}
  { }
  {\thmname{#1}\thmnumber{ #2}{\normalfont\thmnote{ (#3)}}}

\theoremstyle{thmstyle}
\newtheorem{theorem}{Theorem}[section]
\newtheorem{lemma}[theorem]{Lemma}

\newtheorem{proposition}[theorem]{Proposition}

\newtheorem{claim}[theorem]{Claim}

\newletterthm{theoremLet}{Theorem}
\newletterthm{problemLet}{Problem}

\theoremstyle{plainstyle}
\newtheorem{definition}[theorem]{Definition}

\newtheorem{remark}[theorem]{Remark}

\newenvironment{proofof}[1]{\begin{proof}[Proof of #1.]}{\end{proof}}

\makeatletter
\def\refdescformat#1{%
  \phantomsection%
  \let\oldlabel\label%
  \let\label\@gobble%
  \edef\@currentlabel{\MakeLowercase{#1}}
  \let\label\oldlabel%
  #1.
}
\makeatother

\newlist{refdesc}{description}{1}
\setlist[refdesc]{format={\refdescformat}}

\newlist{enumdef}{enumerate}{1}
\setlist[enumdef]{before={\leavevmode}, label={\arabic*.}, ref={\thetheorem.\arabic*}}

\setlist[enumerate]{label={\roman*.}, ref={(\roman*)}} 

\makeatletter
\newcommand\notoc@internal[2][]{}
\newcommand{\notoc}[1]{
  \renewcommand{\notoc@internal}[2][]{%
    \let\old@addtocontents\addtocontents%
    \let\addtocontents\@gobbletwo%
    \ifthenelse{\isempty{##1}}{%
      #1{##2}%
    }{%
      #1[##1]{##2}%
    }%
    \let\addtocontents\old@addtocontents%
  }%
  \notoc@internal%
}
\makeatother

\newenvironment{noindquote}[1][]{%
  \begin{quotation}%
    \noindent%
    \ifthenelse{\isempty{#1}}{}{\textbf{#1.} }%
    \it%
}{%
  \end{quotation}%
}

\newenvironment{emphquote}[1][]{%
  \begin{quotation}%
    \noindent%
    \ifthenelse{\isempty{#1}}{}{(#1). }%
    \it%
}{%
  \end{quotation}%
}

\makeatletter
\newcommand{\biggg}{\bBigg@\thr@@}
\def\bigggl{\mathopen\biggg}

\def\bigggr{\mathclose\biggg}
\newcommand{\Biggg}{\bBigg@{3.5}}
\def\Bigggl{\mathopen\Biggg}

\def\Bigggr{\mathclose\Biggg}
\newcommand{\bigggg}{\bBigg@{4}}
\def\biggggl{\mathopen\bigggg}

\def\biggggr{\mathclose\bigggg}
\newcommand{\Bigggg}{\bBigg@{4.5}}
\def\Biggggl{\mathopen\Bigggg}

\def\Biggggr{\mathclose\Bigggg}
\newcommand{\biggggg}{\bBigg@{5}}

\newcommand{\Biggggg}{\bBigg@{5.5}}

\makeatother

\numberwithin{equation}{section} 

\let\epsilon\varepsilon

\newcommand{\rn}{\bm}
\newcommand{\df}{\stackrel{\text{def}}{=}}
\newcommand{\place}{\mathord{-}}

\newcommand{\comp}{\mathbin{\circ}}
\newcommand{\rest}{\mathord{\vert}}
\newcommand{\floor}[1]{\ensuremath{\lfloor#1\rfloor}}
\newcommand{\ceil}[1]{\ensuremath{\lceil#1\rceil}}
\newcommand{\Floor}[1]{\ensuremath{\left\lfloor#1\right\rfloor}}
\newcommand{\Ceil}[1]{\ensuremath{\left\lceil#1\right\rceil}}

\newcommand{\Given}{\;\middle\vert\;}
\newcommand{\unk}{\mathord{?}}

\DeclareMathOperator{\im}{im}

\DeclareMathOperator{\ag}{ag}
\DeclareMathOperator{\sym}{sym}
\DeclareMathOperator{\SC}{SC}
\DeclareMathOperator{\advSC}{advSC}
\DeclareMathOperator{\sSC}{sSC}
\DeclareMathOperator{\advsSC}{advsSC}
\DeclareMathOperator{\SUC}{SUC}
\newcommand{\hPHP}[1][h]{#1\operatorname{-PHP}}
\newcommand{\hSHP}[1][h]{#1\operatorname{-SHP}}
\DeclareMathOperator{\VC}{VC}
\DeclareMathOperator{\Nat}{Nat}
\DeclareMathOperator{\VCN}{VCN}

\newcommand{\kpart}[1][k]{#1\operatorname{-part}}

\DeclareMathOperator{\Bi}{Bi}
\DeclareMathOperator{\ev}{ev}
\DeclareMathOperator{\ex}{ex}
\DeclareMathOperator{\dist}{dist}
\DeclareMathOperator{\Heis}{Heis}

\newcommand{\EE}{\mathbb{E}}
\newcommand{\FF}{\mathbb{F}}
\newcommand{\NN}{\mathbb{N}}
\newcommand{\PP}{\mathbb{P}}

\newcommand{\RR}{\mathbb{R}}

\newcommand{\One}{\mathbbm{1}}

\newcommand{\cA}{\mathcal{A}}
\newcommand{\cB}{\mathcal{B}}

\newcommand{\cE}{\mathcal{E}}
\newcommand{\cF}{\mathcal{F}}
\newcommand{\cG}{\mathcal{G}}
\newcommand{\cH}{\mathcal{H}}

\newcommand{\cL}{\mathcal{L}}

\newcommand{\cU}{\mathcal{U}}

\def\Erdos{Erd\H{o}s}

\def\Kovari{K\H{o}v\'{a}ri}
\def\Sos{S\'{o}s}
\def\Turan{Tur\'{a}n}

\title{Sample completion, structured correlation, and Netflix problems}
\author{%
  Leonardo N.~Coregliano\thanks{Research partially supported by the 2024--2025 Suzuki Postdoctoral Fellowship} \and%
  Maryanthe Malliaris\thanks{Research partially supported by NSF-BSF 2051825.}%
}
\date{\today}

\begin{document}
\maketitle

\begin{abstract}
We develop a new high-dimensional statistical learning model which can take advantage of structured correlation in data even in
the presence of randomness. We completely characterize learnability in this model in terms of $\VCN_{k,k}$-dimension
(essentially $k$-dependence from Shelah's classification theory). This model suggests a theoretical explanation for the success
of certain algorithms in the 2006~Netflix Prize competition.
\end{abstract}

One of the most famous learning competitions of the early internet era was arguably the Netflix Prize competition~\cite{BL07} of
2006--2009. In this competition, the task was to predict user ratings of movies from partial information. Although some
reasonably successful algorithms were developed for this task, almost twenty years later, essentially no theoretical explanation
for their success is known. This is both a challenge to theory, and a blind spot for improving algorithms for these and related
problems (which we will refer to under the umbrella name of ``Netflix problems'').

Because these are well-known problems with a long history, because they have so far eluded a complete and satisfying
explanation, and because the kinds of learning tasks they describe are still central concerns today, Netflix problems are a
compelling test case for the question of how theory can contribute to the conversation around learning models.

In this paper we develop a statistical learning model for Netflix-type problems, which we call \emph{sample completion
learning}, and we completely characterize the problems it addresses. This is part of a program we are developing to deal with
certain kinds of intrinsic high dimensionality in learning, which can be described by the slogan:
\begin{emphquote}
  Learning problems arising in nature may hide ``structured correlation'' which may need to be leveraged if the learning task is
  to succeed.
\end{emphquote}
This model is inspired by, although independent from, the first two papers in this program~\cite{CM24+,CM25+} as explained
below. We are mathematicians, and part of what interests us in this work is what has always interested mathematicians about the
natural world: that nature (including, of course, AI and machine learning) provides a very interesting source of mathematical
problems and potentially new mathematical phenomena.

Readers may choose to begin with the informal exposition in Section~\ref{sec:informal}, the more technical exposition in
Section~\ref{sec:techexp}, the discussion of the Netflix Prize competition of Section~\ref{sec:Netflix} or the main technical
body of the paper in Section~\ref{sec:defs}.

\section{Informal exposition}
\label{sec:informal}

\subsection*{PAC learning}

Consider the following kind of problem (this paragraph is a simplified sketch of the celebrated PAC learning theory of
Valiant~\cite{Val84}). There is a set $X$ and a collection $\cH$ of subsets of $X$, both of which we know. Our adversary puts a
measure $\mu$ on $X$, which we do not know, and chooses one set $F\in\cH$, which we do not know. We receive a random
i.i.d.\ sample $\rn{x}_1,\ldots,\rn{x}_m$, according to $\mu$, and the adversary tells us which points in the sample belong to
their set $F$. Based on this, we guess some $H\in\cH$. We are judged according to the $\mu$-measure of the symmetric difference
of our $H$ and the correct $F$. How well can we do? PAC learning theory completely describes this problem as follows: call the
class $\cH$ learnable if for every $\epsilon$, $\delta$ there is an $m=m(\epsilon,\delta)$ such that for every adversarial
choice of $\mu$, $H$, given an i.i.d.\ sample of size $m$, we can with probability $1-\delta$ make a guess which is
$\epsilon$-close to being correct. Then ``$\cH$ is learnable'' has a purely combinatorial characterization: if and only if it
has finite Vapnik--Chervonenkis ($\VC$) dimension. This is part\footnote{In its modern statement, due to Vapnik--Chervonenkis,
Blumer--Ehrenfeucht--Haussler--Warmuth, and Natarajan, the fundamental theorem has other equivalent clauses including uniform
convergence and agnostic PAC learning. We will discuss these in due course. See~\cite[Theorem~A]{CM24+}; or~\cite[\S 6.4]{SB14}
for a full statement and references.} of the \emph{Fundamental Theorem of PAC Learning}.

\subsection*{Netflix problems and present work}

In this paper, we will address a class of collaborative filtering problems, whose most famous instance is arguably the Netflix
Prize competition~\cite{BL07} (see~\cite{SK09} for a survey on algorithmic techniques for collaborative filtering problems).
These collaborative filtering problems have been widely studied from an application standpoint and even some sufficient
conditions have been found from a theoretical standpoint. In this work we regard these problems as statistical learning problems
and provide a full theoretical characterization of their feasibility in terms of a combinatorial dimension of the hypothesis
class (in the language of classical PAC learning, we provide a complete fundamental theorem). Without further ado, here is a
simplified version of the Netflix Prize problem (see Section~\ref{sec:Netflix} for a more complete version and a discussion):
\begin{emphquote}[Netflix Prize competition, simplified version on a sample]
  Netflix has a finite set $A$ of users and a finite set $B$ of movies. This information we know. Netflix also has a
  confidential matrix $F\in\{0,1\}^{A\times B}$ whose $(a,b)$ entry is $1$ if user $a$ likes movie $b$ and $0$ otherwise.
  Netflix chooses randomly a $\rho$-proportion of the entries $(a,b)$ of the matrix and provides us with their labels (i.e.,
  with all such triples $(a,b,F(a,b))$). We are tasked with guessing the correct labels for all other pairs $(a,b)$ in the
  matrix $A\times B$.
\end{emphquote}

In fact, as we will see in Section~\ref{sec:Netflix}, in the actual Netflix Prize competition, the sets $A$ and $B$ are picked
at random from much larger sets of users and movies. This motivates the following framing of the problem in the language of
statistical learning: we consider sets of users $\cA$ and of movies $\cB$ which may possibly be infinite, we require the unknown
matrix $F\in\{0,1\}^{\cA\times\cB}$ to be an element of a known hypothesis class $\cH\subseteq\{0,1\}^{\cA\times\cB}$ and we are
provided with a finite portion of it chosen at random in the sense that the adversary picks probability measures on $\cA$ and
$\cB$ and randomly samples $m$ elements from each, independently, forming finite subsets $A\subseteq\cA$ and $B\subseteq\cB$ and
we then consider the Netflix problem on $A\times B$, that is, the adversary reveals to us a randomly chosen $\rho$-proportion of
the labels and we are tasked with completing the $A\times B$ matrix. Just as in PAC learning, our task is to provide an
approximate solution with high probability provided $m$ is large enough, where approximate means only an $\epsilon$ fraction of
the $m^2$ entries can be wrong. This will be an example of what we call \emph{sample completion high-arity PAC learning},
defined formally in Section~\ref{sec:defs} and abbreviated simply to \emph{sample completion learning}.

Beyond this, there are obvious key differences from this setup to classical PAC:
\begin{enumerate}
\item Our sample is not i.i.d. More specifically, even though the sampled users $\rn{a}_i$ and movies $\rn{b}_j$ are chosen
  i.i.d., the information of the problem, i.e., all triples $(\rn{a}_i, \rn{b}_j, F(\rn{a}_i,\rn{b}_j))_{i,j=1}^m$, is not
  i.i.d.\ at all, it features ``structured correlation'' as e.g., entry $(i,j_1)$ is correlated to entry $(i,j_2)$.
\item On top of this, some fraction of the information is randomly erased.
\item On the bright side, differently from classical PAC, once the triples are chosen, we no longer care about users and movies
  outside of the sample $A\times B$, we only need to retrieve the erased information on the sample.
\end{enumerate}

To hammer home the fact that this setup is different from PAC learning, consider the hypothesis class $\cH$ in which each user
$a$ has a favorite movie $b_a$ that they like and they do not like any other movie (i.e., $\cH$ is the set of all matrices in
$\{0,1\}^{A\times B}$ that have exactly one $1$ in each row). It is easy to see that $\cH$ has infinite $\VC$ dimension, hence
it is not learnable in the classic PAC sense. However, in the simplified Netflix Problem for $\cH$, if in a sample grid we see a
$1$ in some row $a_i$, then we immediately know $b_{a_i}$, otherwise, if the row of $a_i$ does not have a $1$ revealed to us, we
can simply guess that $a_i$ does not like any of the $b_j$ as this will only incur at most one error per row, hence less than
$\epsilon\cdot m^2$ errors in total if $m$ is large.

Before we proceed with the mathematical exposition in Section~\ref{sec:techexp}, we make some remarks (which the reader should
feel free to skip, as we do not define all terms) on how this compares with existing work. This is the second step in a larger
program of the authors on leveraging the aforementioned structured correlation. The first step came in the
papers~\cite{CM24+,CM25+} in the form of high-arity PAC learning theory, which already featured improved learning power through
structured correlation in the training data (more specifically, high-arity PAC is characterized by a slicewise $\VC$-dimension,
which is always at most the $\VC$-dimension and can actually be finite without the latter being finite). In the present work, we
leverage not only structured correlation in the training data, but also between the training data and test data. The dimension
that characterizes sample completion learning, which we call $\VCN_{k,k}$-dimension, is in turn at most the slicewise
$\VC$-dimension (and can be finite without the latter being finite); this means that there is a \emph{strict} hierarchy of
learnability: PAC $\implies$ high-arity PAC $\implies$ sample completion. From the definitions alone, it is not immediately
clear why there should be a strict hierarchy of learnability. To explain how this arises and how the different kinds of
correlation contribute to a learning advantage will be the work of the current paper.

To conclude this introduction, we now summarize what we believe to be the main contributions of the present paper:
\begin{itemize}
\item we define a new high-arity statistical learning model, which we call ``sample completion learning,'' which includes as a
  special case the problem of reconstructing randomly erased entries from finite matrices labeled from a finite set.
\item we prove a complete analogue of the fundamental theorem of PAC learning for this model.
\item in particular, we completely characterize learnability in this model in terms of a combinatorial dimension of independent
  mathematical interest.
\item we use this to suggest a theoretical explanation of learning in Netflix problems in Section~\ref{sec:Netflix}.
\end{itemize}

We now turn to a more technical exposition of the paper. The reader can also consult the table of contents on
page~\pageref{tableofcontentsmarker} for pointers to main definitions and theorems and Figure~\ref{fig:roadmap} on
page~\pageref{fig:roadmap} for a pictorial view of the implications involved in the main theorems.

\section{More technical exposition}
\label{sec:techexp}

This section exposits some of the main definitions and proofs in the case where $k=2$ and the learning is partite. Definitions
are mathematical but not completely formal, and we try for simplicity. The formal text will begin in Section~\ref{sec:defs}.

\begin{noindquote}[Convention]
Throughout this expository section, the arity is $k=2$. We also make the following slight simplifications: we fix sets $X_1$ and
$X_2$ and consider a family $\cH$ of hypotheses which are functions\footnote{So we can think about each $H\in\cH$ as a subset of
$X_1\times X_2$ identified with its characteristic function. It is also sometimes useful to think of these sets as the edge-set
of a bipartite graph with bipartition $(X_1,X_2)$, so that the function is the (bipartite) adjacency matrix.} from $X_1\times
X_2$ to $\Lambda = \{0,1\}$. Finally, in this expository section, we will use the $0/1$-loss function $\ell_{0/1}$; this means
that all incorrect guesses algorithms make get the same penalty of $1$ and correct guesses get $0$ penalty.
\end{noindquote}

\subsection*{Three examples}

We start with three examples illustrating a certain kind of \emph{structured correlation in data} which we shall leverage in our
learning model. Unlike previous forms of statistical learning, our model allows for a certain kind of randomness. The first
example is from combinatorics, the second from analysis/physics, and the third from linear algebra and as we will see in
Section~\ref{sec:Netflix}, connected to the ``the real world'' Netflix Prize competition. These examples all have infinite
$\VC$-dimension (hence escape the analysis of original PAC), have infinite slicewise $\VC$-dimension (hence escape the analysis
of the first two high-arity PAC papers), but do have what we will call \emph{bounded $\VCN_{k,k}$-dimension} (Shelah's
$k$-dependence, in the language of model theory). These examples are special in that the dimension itself will not generally
guarantee such a combinatorially basic analysis, however:
\begin{noindquote}
a consequence of our main theorem will be that for any class of finite $\VCN_{k,k}$-dimension, on any sufficiently large finite
grid, \emph{a relatively small set of values can determine the behavior of the hypothesis, and moreover such a representative
small set is statistically easy to find, or rather, statistically hard to erase}.
\end{noindquote}
Here then are the examples to keep in mind:

\begin{enumerate}[label={Example~\Roman*.}, ref={\Roman*}, itemsep=\medskipamount, wide]
\item\label{ex:Rado} First, let $G$ be the Rado graph\footnote{Here is one construction: let $\NN$ be the vertex set, and for
each $(i,j)\in\NN\times\NN$, flip a fair coin and put an edge if it comes up heads. The outcome will be the Rado graph (up to
isomorphism) with probability $1$.} (known to model-theorists simply as the countable random graph), with vertex set $\NN$. Let
  $X_1 = X_2 = \NN$. For each $c\in V(G) = \NN$, let $H_c(a,b) = 1$ if and only if $c$ has an edge to $a$ and not to $b$. Let
  $\cH = \{H_c\mid c\in G\}$. Observe that $\cH$ clearly does not have slicewise finite $\VC$-dimension, so the usual high-arity
  PAC won't apply.\footnote{Fix any $a\in X_1$ and any other $b_1,\ldots,b_n\in X_2$. For any
  $\sigma\subseteq\{b_1,\ldots,b_n\}$, there is some $c\in G$ connected to $a$ and to all $b_i\in\sigma$, but not to any element
  of $\{b_1,\ldots,b_n\}\setminus\sigma$. Letting $H_c$ vary in $\cH$, we shatter arbitrarily large subsets of $X_2$.}
  Nonetheless, $\cH$ has quite a bit of structure. For instance, if $H_c(a,b) = 1$, then necessarily $H_c(b,d) = 0$ and
  necessarily $H_c(e,a) = 0$, regardless of the values of $d$, $e$.
\item\label{ex:Heis} The second example draws on a recent analysis of some widely used matrix groups in the paper~\cite{DM21}.
  These are the discrete Heisenberg groups whose ``continuous'' analogues over $\RR$ are central in analysis and physics.

  Fix a prime $p > 2$ and let $\FF_p$ be the finite field with $p$ elements. For each $n\geq 1$, let $\Heis_n$ be the group of
  $(n+2)\times (n+2)$ matrices with $1$s on the diagonal, arbitrary elements of $\FF_p$ in the remaining entries of the top row
  and right column, and $0$s everywhere else. Let $X_1 = X_2 = \Heis_n$. The learning problem will be to determine the matrix
  $A\in\Heis_n$ as closely as possible based on information about which elements of $\Heis_n$ it does and does not commute
  with. That is, for each $A\in\Heis_n$, let
  \begin{equation*}
    H_A(B,C) = 1 \text{ if and only if $A$ commutes with $B$ and $A$ does not commute with $C$.}
  \end{equation*}
  Then let
  \begin{equation*}
    \cH = \{H_A \mid A\in\Heis_n\}.
  \end{equation*}
  This relates to the previous example in a nontrivial way: \cite{DM21} show that the sequence of commuting graphs of $\Heis_n$,
  for fixed $p$, as $n\to\infty$, are quasirandom. (They also show the unique countable limit, $\Heis_\omega$, is in some sense
  a random graph ``except for'' linear dependence.) Even without quoting these theorems, it may be plausible that $\cH$ cannot
  really act freely to shatter grids $A\times B$ since commuting depends on an underlying binary relation.\footnote{This example
  can also be seen terms of a vector space with a symplectic form, see~\cite{DM21} \S 2.5.}
\item\label{ex:boundedrank} Our third example comes from linear algebra and as we will see in Section~\ref{sec:Netflix} is
  related to the Netflix Prize competition: let $X_1=X_2=\NN$ and let us interpret hypotheses as infinite matrices with entries
  in $\FF_2$, i.e., functions $X_1\times X_2\to\FF_2$. For any given $r\in\NN$, we let $\cH_r\subseteq\FF_2^{\NN\times\NN}$ be
  the set of all infinite matrices of rank at most $r$ (i.e., matrices $M$ that can be written using exterior products as
  $M=\sum_{i=1}^r v_i\cdot v_i^\top$ for $v_i\in\FF_2^\NN$). Note that $\cH_r$ has structure that gets revealed exactly on grids
  $(r+1)\times(r+1)$: no such grid can span an identity matrix.
\end{enumerate}

These three examples exhibit the kind of behavior which our learning model will be able to leverage. In the case of $k=2$:

\subsection*{Bipartite $\VCN_{2,2}$-dimension}

Given $m$ and $A = \{a_1,\ldots,a_m\}$ from $X_1$, $B = \{b_1,\ldots,b_m\}$ from $X_2$, say that $\cH$ \emph{shatters} $A \times
B$ if every partial function $F\colon A\times B\to\{0,1\}$ is extended by some $H\in\cH$. (The ``N'' for ``Natarajan'' indicates
that we also allow a larger finite label set and a corresponding slighly more general notion of shattering.)  Bipartite
$\VCN_{2,2}$-dimension is essentially the largest integer $m$, if it exists, such that $\cH$ shatters some $A\times B$ where
$\lvert A\rvert = \lvert B\rvert = m$; and $\infty$ otherwise.

Again, to our knowledge, a dimension of this kind was first isolated by Shelah (see~\cite[Definition~5.63]{She14} and
references there) under the name of $k$-dependence of a first order formula.

For reference, Examples~\ref{ex:Rado} and~\ref{ex:Heis} have $\VCN_{2,2}$-dimension exactly $1$, while $\cH_r$ in
Example~\ref{ex:boundedrank} has $\VCN_{2,2}$-dimension exactly $r$.

\subsection*{Statement of a simplified main theorem}

We now state a simplified version of the paper's main theorem, a ``fundamental theorem'' for sample completion learning, in the
\emph{partite} case when $k=2$: after stating it, we will discuss the various new definitions and sketch proofs of the main
arrows in this special case. To our knowledge, all definitions except for the combinatorial dimension are new (and to our
knowledge this is the first time this combinatorial dimension has been used in statistical learning).

\begin{theorem}[Simplified version of Theorem~\ref{thm:SCpart} in the case $k=2$]\label{thm:expos}
  For $X_1$, $X_2$, $\cH$ (and using the $0/1$ loss function $\ell_{0/1}$), the following are equivalent:
  \begin{enumerate}[label={\arabic*.}, ref={(\arabic*)}]
  \item\label{thm:expos:VCN22} $\VCN_{2,2}(\cH) < \infty$.
  \item\label{thm:expos:SUC} $\cH$ satisfies sample uniform convergence.
  \item\label{thm:expos:advSC} $\cH$ is adversarial sample completion $2$-PAC learnable.
  \item\label{thm:expos:SC} $\cH$ is sample completion $2$-PAC learnable.
  \item\label{thm:expos:SHP} $\cH$ has the $m^2$-sample Haussler packing property.
  \item\label{thm:expos:PHP} $\cH$ has the $m^2$-probabilistic Haussler packing property.
  \end{enumerate}
\end{theorem}

Each of these items is formally defined in Section~\ref{def:SCpart} below, so our discussion here is informal (and in a
slightly different order). Item~\ref{thm:expos:VCN22} was already discussed.

\subsection*{Sample completion learning (item~\ref{thm:expos:SC})}

Suppose we have fixed our spaces $X_1$ and $X_2$ and a family $\cH$ of hypotheses, where each $H\in\cH$ is a function from
$X_1\times X_2$ to $\{0,1\}$. Suppose we are given in addition some $\epsilon,\delta,\rho > 0$. The setup is:

\begin{description}
\item[Input:] The adversary first fixes $F\in\cH$ and probability measures $\mu_1$ on $X_1$ and $\mu_2$ on $X_2$, respectively,
  all of these are unknown to the learner. The adversary then samples $\rn{a}_1,\ldots,\rn{a}_m$ i.i.d.\ from $\mu_1$ and
  $\rn{b}_1,\ldots,\rn{b}_m$ i.i.d.\ from $\mu_2$ (and independently from the $\rn{a}_i$), revealing these values to the
  learner. The adversary then forms an $m\times m$ grid as follows: first, they take a coin with probabiity of heads $\rho$ and,
  for each $(i,j)\in[m]\times[m]$, they flip the coin and label the $(i,j)$ entry of an $[m]\times[m]$ grid with $F(i,j)$ if the
  coin is heads, and ``$\unk$'' if the coin is tails. (Here ``$\unk$'' is a distinguished symbol that indicates to the learner
  that the label of the entry has been erased.) The adversary now gives this partially erased grid to the learner\footnote{Let
  us point out that this is not exactly the formulation of the Netflix Prize competition as in that one, we got a random
  $\rho$-proportion of all the entries instead of getting each entry independently with probability $\rho$; however, it is
  straightforward to translate between the two settings using concentration bounds and a small adjustment of parameters.}. The
  collection of names $\rn{a}_1,\ldots,\rn{a}_m,\rn{b}_1,\ldots,\rn{b}_m$ along with the labels of the $[m]\times[m]$ grid is
  referred to as \emph{partially erased sample}.
\item[Output:] Based on partially erased sample, the learner outputs some $H\in\cH$.
\item[Judgment:] We look at all $(a,b)\in A\times B$ and compare $H(a,b)$ to the true answer $F(a,b)$. The learner is
  successful in this instance if $H$ is different from $F$ in at most an $\epsilon$ proportion of all the $m^2$
  entries.\footnote{This setup corresponds to $0/1$-loss, but the theory also covers more general loss functions as long as they
  satisfy some mild natural assumptions.}

  Note that because we are simply trying to \emph{reconstruct the erased labels} on the given grid $A\times B$, there is no
  measure involved in calculating the loss, we are judged only on the entries of the $m\times m$ grid.
\end{description}

In general, we say that $\cH$ is \emph{sample completion learnable} if there exists a learning algorithm\footnote{As in the
usual PAC setup, ``algorithm'' just means that $\cA$ is a set-theoretic function from inputs to outputs.} $\cA$ such that for
every $\epsilon,\delta,\rho > 0$, there is $m^{\SC}_{\cH,\cA} = m^{\SC}_{\cH,\cA}(\epsilon,\delta,\rho)$ such that for every
$\mu_1$ on $X_1$ and $\mu_2$ on $X_2$, for every $F\in\cH$ and every integer $m\geq m^{\SC}_{\cH,\cA}$, with probability
$1-\delta$ over the choice of
\begin{itemize}
\item $m$ elements $\rn{a}_1,\ldots,\rn{a}_m$ from $X_1$ according to $\mu_1$,
\item $m$ elements $\rn{b}_1,\ldots,\rn{b}_m$ from $X_2$ according to $\mu_2$, and
\item the $(1-\rho)$-erasure of the labeling of the resulting grid,
\end{itemize}
our algorithm $\cA$, on receiving the partially erased sample, outputs some $H\in\cH$ whose fraction of errors on the grid is
less than $\epsilon$.

\subsection*{Adversarial sample completion learning (item~\ref{thm:expos:advSC})}

We now cover an adversarial version of sample completion learning. This differs from item~\ref{thm:expos:SC} in two aspects, one
expected by those familiar with agnostic PAC learning, and one a bit surprsing:
\begin{description}
\item[Non-realizability:] In the vein of agnostic learning, the adversary is not required to pick an element of $\cH$, but
  rather a general $F$. However, the learner's goal is not to attain small loss, but rather to be competitive, i.e., if the
  learner outputs $H\in\cH$ that differs from $F$ on an $L$ proportion of the $m^2$ entries of the sample, then they are
  successful in the learning task if $L < L_* + \epsilon$, where $L_*$ is the proportion of the difference from $F$ of the best
  element of $\cH$ (which is completely inaccessible to the learner as one needs to know the erased entries to compute $L_*$).
\item[Adversarial:] The choice of both the sample and the function $F$ by the adversary is completely free (justifying the name
  ``adversarial''). This means that the only randomness involved in the learning test is of the $(1-\rho)$-erasure. 
\end{description}

We now make the definition a bit more precise. Suppose we have fixed our spaces $X_1$ and $X_2$ and a family $\cH$ of
hypotheses, where each $H\in\cH$ is a function from $X_1\times X_2$ to $\{0,1\}$. Suppose we are given in addition some
$\epsilon,\delta,\rho > 0$. The setup is:

\begin{description}
\item[Input:] The adversary picks an arbitrary function $F\colon X_1\times X_2\to\{0,1\}$, this is unknown to the learner. The
  adversary then pick points $a_1,\ldots,a_m$ from $X_1$ and $b_1,\ldots,b_m$ from $X_2$ adversarially, revealing these to the
  learner (we emphasize a priori no measure is involved). The adversary then form an $m\times m$ grid as before: first, they
  take a coin with probabiity of heads $\rho$ and, for each $(i,j)\in[m]\times[m]$, they flip the coin and label the $(i,j)$
  entry of an $[m]\times[m]$ grid with $F(i,j)$ if the coin is heads, and ``$\unk$'' if the coin is tails. The adversary now
  gives this partially erased grid to the learner.
\item[Output:] Based on the partially erased sample, the learner outputs some $H\in\cH$.
\item[Judgment:] We look at all $(a,b)\in A\times B$ and compare $H(a,b)$ to the true answer $F(a,b)$. However, now the goal of
  the learner is to be competitive. Namely, if $H_*\in\cH$ minimizes its difference to $F$ on the $A\times B$, differing on a
  proportion $L_*$ of the $m^2$ entries, then the leaner is successful if $H$ differs from $F$ in at most an $L_*+\epsilon$
  proportion of the $m^2$ entries.
\end{description}

In general, we say that $\cH$ is \emph{adversarial sample completion learnable} if there exists a learning algorithm $\cA$ such
that for every $\epsilon,\delta,\rho>0$, there is $m^{\advSC}_{\cH,\cA} = m^{\advSC}_{\cH,\cA}(\epsilon,\delta,\rho)$ such that
for every $F\colon X_1\times X_2\to\{0,1\}$, every $m$ elements from $X_1$, and every $m$ elements from $X_2$, we have that with
probability at least $1-\delta$ over the choice of a $(1-\rho)$-erasure of the labeling of the resulting grid, our algorithm
$\cA$, on receiving the partially erased labeled grid, outputs some $H\in\cH$ whose fraction of errors on the grid is less than
$L_*+\epsilon$, where $L_*$ is the fraction of errors achieved by the best element $H_*$ of $\cH$.

\subsection*{Sample uniform convergence (item~\ref{thm:expos:SUC})}

Sample uniform convergence works for any sufficiently large grids $A\times B$. Unlike the parallel theorems in earlier forms of
PAC learning, this convergence statement is \emph{not} saying that if we fix $\mu_1$,$\mu_2$ then for most choices of $A\times
B$ something is likely to happen. Rather, here the random element is the erasure.

Sample uniform convergence says essentially that for every $\epsilon,\delta,\rho$ there exists $m\in\NN$ so that for every
$F\colon X_1\times X_2\to\{0,1\}$, on any $A\times B$ of size at least $m\times m$, with probability at least $1-\delta$ over
the $(1-\rho)$-erasure, for \emph{every} hypothesis $H\in\cH$, the ``empirical loss'' of $H$ (that is, the proportion of the
difference of $H$ and $F$ on the entire labeled sample before erasure) is within $\epsilon$ of the ``partially erased empirical
loss'' of $H$ (that is, the proportion of the difference on the part of the sample that was not erased).

This has a very similar flavor to uniform convergence of classic (and high-arity) PAC learning theory, which says that ``with
high probability, the actual and empirical losses are close''\footnote{Recall that the classical fundamental theorem of PAC
learning relies on a uniform convergence theorem for $\VC$ classes $\cH$ over a set $X$ which says, approximately, that for
every $\epsilon,\delta>0$ there is $m\in\NN$ so that given any measure $\mu$ on $X$ and any $F\colon X\to\{0,1\}$ an
i.i.d.\ sample $\rn{x}_1,\ldots, \rn{x}_m$, with probability at least $1-\delta$ over the choice of sample, we have that the
sample $\rn{x}_1,\ldots,\rn{x}_m$ is $\epsilon$-representative in the sense that for every $H\in\cH$, the ``empirical distance
between $F$ and $H$'' (calculated as the proportion of $\rn{x}_i$'s in which $F$ and $H$ differ) and the ``actual distance
between $F$ and $H$'' (calculated as $\mu(\{x \mid F(x)\neq H(x)\})$) are within $\epsilon$. Clearly this is a very useful
feature for a would-be learner.}. However, here the role of the actual loss is played by the empirical loss and the role of the
empirical loss is played by the partially erased empirical loss, so sample uniform convergence amounts to ``with high
probability, the empirical and partially erased empirical losses are close''.

\subsection*{The sample Haussler packing property (item~\ref{thm:expos:SHP})}

Recall that the classical Haussler packing property of a hypothesis class over $X$ asks for a bound on the size of the largest
$\epsilon$-separated set that depends only on $\epsilon$ (and the class itself), but not on the measure $\mu$ we put on $X$
(see~\cite[Corollary~1]{Hau95} or~\cite[\S 5.3]{Mat10} for a modern treatment). It is clear that if we require the same for a
hypothesis class over $X_1\times X_2$, then this is simply treating $X_1\times X_2$ as an $X$ and since the fundamental theorem
shows that classical Haussler packing is equivalent to finite $\VC$-dimension, it cannot apply in the sample completion setting.

Instead, we expect to have a high-arity version of Haussler packing. A natural candidate is to require a bound that depends only
on $\epsilon$ (and the class itself), but that only holds for product measures $\mu_1\otimes\mu_2$. However, this high-arity
Haussler packing property was shown in~\cite{CM25+} to be equivalent to finite slicewise $\VC$-dimension, hence it also cannot
apply in the sample completion setting.

So the present study demands reconsideration of what the correct ``packing phenomenon'' might be. Inspired by the fact that both
the learning and the uniform convergence notions are localized to the sample grid $A\times B$, we instead analyze when
hypotheses $H_1,\ldots,H_t\in\cH$ are \emph{$\epsilon$-separated} on $A\times B$, that is, if any two have distance greater than
$\epsilon$ \emph{on the grid}, that is, they differ on more than an $\epsilon$ proportion of the $m^2$ entries of $A\times B$.

Now for a function $h\colon\NN\to\NN_+$, say that $\cH$ has the \emph{$h$-sample Haussler packing property} if for every
$\epsilon,\delta,\rho > 0$, there exists $m^{\hSHP}_\cH = m^{\hSHP}_\cH(\epsilon,\delta,\rho)$ such that for every choice of
$a_1,\ldots,a_m\in X_1$ and $b_1,\ldots,b_m\in X_2$ with $m\geq m^{\hSHP}_\cH(\epsilon,\delta,\rho)$, the largest
$\epsilon$-separated collection $\cH'\subseteq\cH$ of hypotheses has size $\lvert\cH'\rvert < 2^{\rho\cdot h(m)}$; in a slightly
less formal language, the largest $\epsilon$-separated collections on $m\times m$ grids have size $2^{o(h(m))}$.

Note that the aforementioned high-arity Haussler packing of~\cite{CM25+} would instead say that the largest $\epsilon$-separated
collection on $m\times m$ grids has constant size (i.e., $O(1)$). In Theorem~\ref{thm:expos}\ref{thm:expos:SHP}, we consider the
$m^2$-sample Haussler packing property, which is drastically weaker than high-arity Haussler packing (and not surprisingly as
finite $\VCN_{2,2}$-dimension does not imply finite slicewise $\VC$-dimension).

\subsection*{The probabilistic Haussler packing property (item~\ref{thm:expos:PHP})}

We now further weaken (at least a priori) the sample Haussler packing property to a probabilistic version. Namely, for each
subcollection $\cH'\subseteq\cH$, we can let $S_{m,\epsilon}(\cH)\subseteq X_1^m\times X_2^m$ be the set of all $m\times m$
grids on which $\cH'$ is $\epsilon$-separated.

For a function $h\colon\NN\to\NN_+$, say that $\cH$ has the \emph{$h$-probabilistic Haussler packing property} if for every
$\epsilon,\delta,\rho > 0$, there exists $m^{\hPHP}_\cH = m^{\hPHP}_\cH(\epsilon,\delta,\rho)$ such that for every choice of
measures $\mu_1$ on $X_1$ and $\mu_2$ on $X_2$, every $m\geq m^{\hPHP}_\cH$ and every $\cH'\subseteq\cH$ where
$\lvert\cH'\rvert\geq 2^{\rho\cdot h(m)}$, with probability at least $1-\delta$ over sampling $\rn{a}_1,\ldots,\rn{a}_m$ from
$\mu_1$ and $\rn{b}_1,\ldots,\rn{b}_m$ from $\mu_2$, the collection $\cH'$ is \emph{not} $\epsilon$-separated on the
resulting grid $A\times B$.

This is clearly implied by the $h$-sample Haussler packing property, but note also that it is a priori even weaker than an
intermediate version saying that with probability at least $1-\delta$, a random $m\times m$ grid has largest
$\epsilon$-separated set smaller than $2^{\rho\cdot h(m)}$; instead it says that if $\cH'$ has size at least $2^{\rho\cdot
  h(m)}$, then it will \emph{not} be $\epsilon$-separated on a random grid with $1-\delta$ probability. A priori, it could be
that $\ceil{1/\delta}$ many different $\cH'$ of size $2^{\rho\cdot h(m)}$ together cover all the randomly picked grids without
any single $\cH'$ covering more than $\delta$ probability.

However, Theorem~\ref{thm:expos} says that even this apparently extremely weak $m^2$-probabilistic Haussler packing property is
equivalent to finite $\VCN_{2,2}$-dimension (hence also equivalent to the $m^2$-sample Haussler packing property and even to the
aforementioned intermediate version).

\subsection*{Brief notes on these arrows}

Here we comment on some of the arrows in the proof of Theorem~\ref{thm:expos}.

\begin{refdesc}
\item[\ref{thm:expos:advSC}$\implies$\ref{thm:expos:SC}] Adversarial sample completion implies sample completion a fortiori
  because in adversarial sample completion~\ref{thm:expos:advSC}, the adversary can choose the sample in any manner, and in
  sample completion~\ref{thm:expos:SC}, the adversary is restricted to fixing some pair of measures $(\mu_1,\mu_2)$ and sampling
  via them.
\item[Implicit arrow\label{it:implicitarrow}] Finite $\VCN_{2,2}$-dimension implies control of the growth function of the number
  of hypotheses over any square grid $A\times B$ of size $m^2$. A counting lemma for $k$-dependence was already known to
  Shelah~\cite[Conclusion~5.66]{She14}, who shows that in $k$-dependent theories, it must be the case that for infinitely many
  $m$, in an $m^k$ grid we see less than $2^{m^k}$ patterns. However, for our proof to go through, we need a much finer control
  of this growth function. Namely, by connecting the problem with the extremal problem in combinatorics of maximization of edges
  in graph without a complete bipartite graph with $\VCN_{2,2}(\cH)+1$ vertices in each part (in combinatorial notation
  $\ex(m,K_{\VCN_{2,2}(\cH)+1,\VCN_{2,2}(\cH)+1})$) and a using classical result by \Kovari--\Sos--\Turan\ (see
  Theorems~\ref{thm:KSTErdos:partite} and~\ref{thm:KSTErdos}, which also include the general $k$ case studied by \Erdos), we
  show in Lemma~\ref{lem:VCNkk->kgrowth} that the number of patterns in an $m^k$ grid is at most
  \begin{equation*}
    \exp\bigl(O(m^{2 - 1/(\VCN_{2,2}(\cH)+1)}\cdot\ln m)\bigr),
  \end{equation*}
  this is asymptotically much smaller than $2^{m^2}$ (and it holds for every sufficiently large $m$ as opposed to infinitely
  many $m$).
\item[\ref{thm:expos:VCN22}$\implies$\ref{thm:expos:SHP}] A direct consequence of the \ref{it:implicitarrow} is that in a
  (sufficiently large) $m^2$ grid, there are at most $2^{o(m^2)}$ many patterns (in fact, $m^{O(m^{2 - 1/\VCN_{2,2}(\cH)})}$
  many), which in particular means that all collections $\cH'\subseteq\cH$ of size larger than this bound must repeat a pattern
  on this grid, hence cannot be $\epsilon$-separated on the grid. Thus finite $\VCN_{2,2}$-dimension implies $m^2$-sample
  Haussler packing property.
\item[\ref{thm:expos:SHP}$\implies$\ref{thm:expos:PHP}] The fact that $m^2$-sample completion Haussler packing implies
  $m^2$-probabilistic Haussler packing is obvious from definitions.
\item[\ref{thm:expos:VCN22}$\implies$\ref{thm:expos:SUC}] Finite $\VCN_{2,2}$-dimension implies sample uniform convergence by an
  argument which has some parallels to the classical case. (To emphasize, this sketch covers the bipartite argument, which is
  simpler than the non-partite version.)

  A partite empirical loss function takes in: a tuple [a pair from $A\times B$], our guess $H$, and the adversary's
  labeling. It then returns a penalty. Call this penalty ``the loss on the tuple.''

  Fixing a hypothesis $H\in\cH$, we want to compare two quantities on $A\times B$. The first is the normalized sum of all
  errors: that is, $1/m^2$ times the sum over all tuples of the loss on the tuple. Let $\rn{\cU}$ be the set of tuples whose
  labels were not erased. The second quantity is the normalized sum of all errors made on tuples whose labels were not erased:
  that is, $1/\lvert\rn{\cU}\rvert$ times the sum over all tuples in $\rn{\cU}$ of the loss on the tuple. We aim to show that
  with high probability, the sup over all $H\in\cH$ of this difference is small.

  First, by a standard Chernoff bound, $\lvert\rn{\cU}\rvert$ will, with high probability, be close to its expected value $\rho
  m^2$. For our fixed $H\in\cH$, then, the difference looks like
  \begin{equation*}
    \frac{1}{m^2}
    \cdot
    \left\lvert\text{sum of all losses} - \frac{1}{\rho}\cdot\text{sum of losses on non erased tuples}\right\rvert.
  \end{equation*}
  Informally, still for a fixed $H\in\cH$, weight the loss on a given tuple by $1$ if it is erased (which happens with
  probability $1-\rho$) and $1-1/\rho$ if it is not (which happens with probability $\rho$). In particular, the expected value
  of each weight is $0$ and we are interested in showing that with high probability, the sum of all $m^2$ weights is
  $\epsilon$-small \emph{for all $H\in\cH$}.

  With a standard Hoeffding bound, for each $H\in\cH$, with probability $1 - \exp(-O(\rho^2\cdot\epsilon^2\cdot m^2))$, the
  weight corresponding to $H$ is $\epsilon$-small. On the other hand, by the \ref{it:implicitarrow}, we know that in the $m^2$
  grid, there are at most $2^{o(m^2)}$ many patterns, so if $m$ is large enough, we can apply a union bound to conclude that
  with probability at least $1-\delta$ the weights of all $H\in\cH$ are $\epsilon$-small.
\item[\ref{thm:expos:SUC}$\implies$\ref{thm:expos:advSC}] This crucial arrow is a direct consequence of the definition of sample
  uniform convergence being the correct version of uniform convergence for sample completion learning. It simply uses sample
  uniform convergence (along with an appropriate notion of empirical risk minimizer for sample completion) to obtain adversarial
  sample completion learnability.
\item[\ref{thm:expos:SC}$\implies$\ref{thm:expos:PHP}] Sample 2-PAC learnability implies $m^2$-probabilistic Haussler packing:
  suppose we have $m = m(\epsilon,\delta,\rho)$ and a collection $\cH'=\{H_1,\ldots,H_t\}\subseteq\cH$ with
  $t=\lvert\cH'\rvert\geq 2^{\rho\cdot m^2}$. We want to show that for any choice of measures $\mu_1$ on $X_1$, $\mu_2$ on $X_2$,
  the set
  \begin{equation*}
    S_{m,\epsilon}(\cH)
    =
    \{(x^1,x^2)\in X_1^m\times X_2^m \mid
    \cH'\text{ is $\epsilon$-separated on } (x^1,x^2)\}
  \end{equation*}
  has product measure at most $\delta$. To prove this, let us assume that $\cH$ is learnable with parameters
  $(\epsilon/2,\delta/2,\rho/2)$, say by some learning algorithm $\cA$ and hence, very informally, we find three points of
  leverage. Define, for each $1\leq i\leq t$ and each appropriate sequence $w$ of $0$s and $1$s (where the $0$s encode the
  labels to be erased: call $w$ an erasure rule; it will have length $m^2$), the set $G_i$ of pairs $(x,w)$ where $x$ is an
  $m$-sample, $w$ is an erasure rule, and if $\cA$ receives this sample labeled by $H_i$ and erased according to $w$, then $\cA$
  returns a hypothesis $\epsilon/2$-close to $H_i$.

  First, learning says that these $G_i$ are large: if we randomly choose the sample\footnote{In our running sense: $\rn{x}$
  involves choosing elements from $X_1$ and from $X_2$ and forming the resulting finite grid.} and the erasure rule $\rn{w}$
  then the probability of belonging to $G_i$ is at least say $1-\delta/2$.

  Second, for each fixed sample $x$ and each fixed erasure rule $w$ with $s$-many $1$s, if $x\in S_{m,\epsilon}(\cH)$ (i.e.,
  $\cH'$ is $\epsilon$-separated on the grid generated by $x$), our learning algorithm $\cA$ can only receive one of $2^s$ many
  possible inputs. On the other hand, since $\cH'$ is $\epsilon$-separated on $x$, if several $H_i$ provide the same input to
  $\cA$ with respect to $(x,w)$, then $\cA$ can only be successful in one of them (as being successful for $H_i$ means its
  answer is $\epsilon/2$-close to $H_i$ on $x$). This means that $(x,w)$ is in at most $2^s$ of the $G_i$.

  The final point of leverage is that, informally, we expect most outcomes of the erasure rule $\rn{w}$ to have approximately
  $(\rho/2)\cdot m^2$ many $1$s (the formal argument actually is via expectation and not a concentration bound); this yields an
  inequality of the form
  \begin{equation*}
    \left(1-\frac{\delta}{2}\right)\cdot t
    \leq
    \EE_{\rn{x},\rn{w}}\left[\sum_{i=1}^t \One_{G_i}(\rn{x},\rn{w})\right]
    \lesssim
    (\mu_1\otimes\mu_2)\bigl(S_{m,\epsilon}(\cH)\bigr)\cdot 2^{\rho\cdot m^2/2}
    + \Bigl(1 - (\mu_1\otimes\mu_2)\bigl(S_{m,\epsilon}(\cH)\bigr)\Bigr)\cdot t
  \end{equation*}
  hence
  \begin{equation*}
    (\mu_1\otimes\mu_2)\bigl(S_{m,\epsilon}(\cH)\bigr)
    \lesssim
    \frac{\delta}{2}\cdot\frac{t}{t - 2^{\rho\cdot m^2/2}}
    \leq
    \frac{\delta}{2}\cdot\frac{2^{\rho\cdot m^2}}{2^{\rho\cdot m^2} - 2^{\rho\cdot m^2/2}}
  \end{equation*}
  so if $m$ is sufficiently large, the above is at most $\delta$.
\item[\ref{thm:expos:PHP}$\implies$\ref{thm:expos:VCN22}] $m^2$-Probabilistic Haussler packing implies finite
  $\VCN_{2,2}$-dimension: This key implication is responsible for closing the loop of equivalences. It goes by the
  contrapositive. Suppose we take small $\epsilon,\delta,\rho$ and hence $m=m(\epsilon,\delta,\rho)$ for probabilistic Haussler
  is given. Choose $n$ to be large enough relative to $m$.

  Since we assume $\VCN_{2,2}$-dimension is infinite, we can find $A = \{a_1,\ldots,a_n\}\subseteq X_1$ and $B =
  \{b_1,\ldots,b_n\}\subseteq X_2$ so that $A\times B$ is shattered by $\cH$ (i.e., every one of the $2^{n^2}$ possible
  labelings of the points in the grid is extended by some hypothesis from $\cH$). Identify hypotheses with their restrictions to
  $A\times B$ and consider their domain to be $n^2$. Instead of considering their range to be $\{0,1\}$, we may view them as
  functions from $[n]^2$ to $\FF_2$, that is, as elements of the $\FF_2$-vector space $\FF_2^{[n]^2}$.

  We would like to generate our contradiction by putting the uniform probability measures $\mu_1$ and $\mu_2$ on $A$ and $B$,
  respectively, and finding a subcollection $C$ of $\FF_2^{[n]^2}$ (i.e., of $\cH$) of size at least $2^{\rho\cdot m^2}$ such
  that when we sample $m$ points from $\mu_1$ and $\mu_2$, the subcollection is $\epsilon$-separated on the sample.

  Here it becomes extremely convenient to frame the problem in coding theory language. First, given functions
  $\gamma_1,\gamma_2\colon[n]\to[m]$, let $\gamma^*\colon\FF_2^{[n]^2}\to\FF_2^{[m]^2}$ be the projection given by
  \begin{equation*}
    \gamma^*(x)_{(i,j)} \df x_{(\gamma_1(i),\gamma_2(j))}.
  \end{equation*}
  We are looking for a code, i.e., a subset $C\subseteq\FF_2^{[n]^2}$ such that for independently uniformly randomly picked
  functions $\rn{\gamma}_1,\rn{\gamma}_2\colon [n]\to [m]$ with high probability, $C$ will have large ``projected distance''
  defined by
  \begin{equation*}
    \dist_{\rn{\gamma}}(C)
    \df
    \min_{\substack{w,w'\in C\\ w\neq w'}} d_H\bigl(\rn{\gamma}^*(w), \rn{\gamma}^*(w')\bigr),
  \end{equation*}
  where $d_H(z,z')\df\lvert\{i\in[m]^2 \mid z_i\neq z'_i\}\rvert$ is the Hamming distance (on $\FF_2^{[m]^2}$). Namely, our goal
  is to get the projected distance above to be larger than $\epsilon\cdot m^2$ with probability greater than $\delta$, so that
  this generates a contradiction with the $m^2$-probabilistic Haussler packing property.

  To find such a $C$, it is convenient to restrict oneself to \emph{linear codes}, i.e., $\FF_2$-linear subspaces of
  $\FF_2^{[n]^2}$. The convenience comes from the fact that the projection maps $\gamma^*$ are linear and the projected distance
  of a linear code can be more easily computed as
  \begin{equation*}
    \dist_{\rn{\gamma}}(C)
    =
    \min_{w\in C\setminus\{0\}} d_H(\rn{\gamma}^*(w), 0),
  \end{equation*}
  i.e., the minimum Hamming weight of the $\rn{\gamma}^*$-projection of a non-zero element of $C$.

  In turn, to prove that such a large linear code $C$, we use (as is common in coding theory) a probabilistic method: we fix
  $d\df\ceil{\rho\cdot m^2}$ and we pick a random linear code of dimension at most $d$; more specifically, we let $\rn{C}$ be
  the image of a uniformly at random $[m]^2\times[d]$-matrix with entries in $\FF_2$. With standard concentration techniques, we
  show that with such a random code $\rn{C}$ with high probability satisfies the properties required above (and has dimension
  exactly $d$), provided $m$ is sufficiently large and $n$ is sufficiently large with respect to $m$.
\end{refdesc}

\subsection*{Bridge to the main proofs}

To conclude this expository section let us call attention to some of the main points and extensions not present in the above
sketch.

\begin{itemize}
\item \emph{Values of $k$ greater than two.} As mentioned before, our results are actually proved for general $k\in\NN_+$ and the
  exposition above remains reasonably indicative of the arguments for general $k$.
\item \emph{Larger $\Lambda$.} In generality the label set $\Lambda$ can be finite but larger than $\{0,1\}$. To reflect
  this we add ``Natarajan'' to ``Vapnik-Chervonenkis'' in our dimension acronym. Note that this comes with an interesting
  upgrade to shattering (following Natarajan): we ask essentially that there are two functions $f,g$ which take different values
  everywhere on the set to be shattered, and then for every partition of the set into two pieces, there is a hypothesis agreeing
  with $f$ on one piece and with $g$ on the complement.
\item \emph{Partite versus non-partite}. This is a central conceptual and technical feature which arises in high-arity
  statistical learning, including in our present work. (It doesn't appear in the classic PAC theory, though it will be familiar
  to readers of~\cite{CM24+}.) Briefly:
  \begin{itemize}
  \item \emph{Partite:} In the exposition just given, we kept track of two separate axes, $X_1$ and $X_2$, we sampled a set of
    $m$ points $A$ from $X_1$ and $B$ from $X_2$ according to two possibly different measures, and we only asked the hypothesis
    to label pairs from $A\times B$, not $A\times A$ or $B\times B$. (Of course, if $X_1 = X_2$ our sets $A$ and $B$ might
    possibly have overlapped, but the quantification in the learning problem ranges over all $\mu_1$ and $\mu_2$ and all
    randomly chosen $A$ and $B$ so overlap cannot be counted on.) In other words, the problem appeared ``bipartite,'' and for
    $k\geq 2$ and $X_1,\ldots,X_k$ we could simply say ``partite'' or ``$k$-partite.''
  \item \emph{Non-partite:} Suppose instead we had been trying to learn a class of graphs $\cG$ all on the same vertex set $X$.
    In this case the natural learning problem would be receiving $m$ vertices from $X$ along with the (partially erased) induced
    subgraph on those vertices arising from the adversary's choice of $G\in\cG$. This is a ``non-partite'' problem. In
    particular, a key difference is that in the non-partite, there is only one measure $\mu$ (which is over $X$), regardless of
    the arity $k$ of the hypothesis class.
  \item \emph{Comparison:} How does this non-partite learning problem compare to the partite learning problem we obtain by
    turning each $G\in\cG$ into a bipartite graph in the natural way by doubling its vertex set (or more generally, turning
    $k$-hypergraphs into $k$-partite by $k$-fold repeating the vertex set)? Our main theorem shows that a non-partite class is
    sample completion learnable if and only if its partization is sample completion learnable (in the partite sense).
  \end{itemize}
  At the scale of a learning problem, a priori, the ``partite'' and ``non-partite'' sample completion learning paradigms may
  appear to involve different kinds and amounts of information. A central feature of the theory is the entanglement of these two
  paradigms. For instance, non-partite learning is extremely natural in practice since ``induced substructure'' and ``induced
  subgraph'' are basic mathematical carriers of information. On the other hand, the $\VCN_{k,k}$-dimension is basically defined
  in a partite way.
\item \emph{Some features of the non-partite.} In the partite case, given $A\times B$, there is a clear order on any tuple we
  receive: its first coordinate comes from $A$ and its second from $B$. Here are several subtleties of the non-partite case.
  First, recalling that the intent is that we receive a set of vertices and the information about induced substructure on that
  set, we have to specify an order on the vertices in order to make sense of this, but we then need to be able to reference
  different sub-orders. For instance, if we are learning a family of colored directed graphs with binary edge $E$ and colors
  $P$, $Q$, given a sequence of vertices $\langle v_1,\ldots,v_r\rangle$ we need to input whether $E$ holds on any $(v_i, v_j)$
  for $\langle i,j\rangle$ a function from $\{0,1\}$ to $\{1,\ldots,r\}$; and we need to input whether $P$, as well as
  $Q$, hold on any $v_i$. And we also have to \emph{output} this quantity of information. Second, loss functions may not give
  the same loss when presented with the ``same information'' in two different ways.\footnote{Is our guess $(v_1,v_2)$ with the
  information that there is a directed edge from the first coordinate to the second but not from the second to the first? Or is
  our guess $(v_2,v_1)$ with the information that there is a directed edge from the second coordinate to the first but not from
  the first to the second? These obviously present the same structure, but the loss function may penalize them differently for
  its own reasons. Why not simply require that the loss function behaves well? We may; this is ``symmetric''; but often a more
  robust result can be proved.} Third, when we are erasing labels in the non-partite case, there is a possibility of erasing
  just part of the label associated to a given set of vertices (i.e., part of the information about its induced structure). So
  when computing the set $\cU$ of tuples which are not erased, we gather those where \emph{no} information has been lost.
\item \emph{No-free-lunch Theorem for sample completion.} The reader familiar with classical PAC theory might have been
  expecting to see a ``No-free-lunch Theorem'' for sample completion, i.e., a direct proof that sample completion learnability
  implies finite $\VCN_{2,2}$-dimension. While it is straightforward to adapt the No-free-lunch Theorem of classical PAC theory
  to sample completion, we opted to go via $m^2$-probabilistic Haussler packing property for two main reasons:
  \begin{itemize}
  \item a No-free-lunch Theorem would not be enough to include the (apparently weak) $m^2$-probabilistic Haussler packing
    property (not even the $m^2$-sample Haussler packing property) in the list of equivalent properties; and
  \item more importantly, this adaptation would only work seamlessly in the partite. This is because even in the non-partite,
    the $\VCN_{2,2}$-dimension has an inherently partite definition: for it to be at least $n$, we need to find $a_1,\ldots,a_n$
    distinct and $b_1,\ldots,b_n$ distinct such that the set $\{\{a_i,b_j\} \mid i,j\in[n]\}$ is shattered. However, this does
    \emph{not} say that the edges between two of the $a_i$ or between two of the $b_j$ are free either from each other or from
    the ones of the form $\{a_i,b_j\}$. A priori, a sample completion algorithm in the non-partite could use information on how
    the $a_i$ relate to each other and how the $b_j$ relate to each other to deduce some information about the ``crossing
    edges'' $\{a_i,b_j\}$. In the partite this issue is not present as the setup itself makes it so that there is no information
    on how the $a_i$ relate to each other nor any information on how the $b_j$ relate to each other. An analogous difficulty in
    lifting the No-free-lunch Theorem in the non-partite had already happened in the high-arity PAC theory of~\cite{CM24+} and
    was circumvented exactly by closing the equivalence via a high-arity Haussler packing property~\cite{CM25+} (which is what
    prompted the authors to look for a Haussler packing property compatible sample completion).
  \end{itemize}
\end{itemize}

This concludes the introductory material.

\section{Connection to Netflix Prize competition}
\label{sec:Netflix}

In this section we explain how our model contributes to understanding of the Netflix Prize competition. While this is likely to
be the most read part of the paper, we caution that it is also in some sense the least mathematical. Obviously, we do not claim
explanation has the same status as a theorem. Nonetheless, we find the parallels compelling enough to set out for discussion.
Note to the reader: in this section we will occasionally refer to Section~\ref{sec:techexp}.
  
The Netflix Prize competition, from contemporary reports, functioned as follows. We quote from Bennett--Lanning~\cite{BL07}:
\begin{quote}
  Netflix provided over 100 million ratings (and their dates) from over 480 thousand randomly-chosen, anonymous subscribers on
  nearly 18 thousand movie titles. The data were collected between October, 1998 and December, 2005 and reflect the distribution
  of all ratings received by Netflix during this period. The ratings are on a scale from 1 to 5 (integral) stars. It withheld
  over 3 million most-recent ratings from those same subscribers over the same set of movies as a competition qualifying set.

  Contestants are required to make predictions for all 3 million withheld ratings in the qualifying set. The RMSE [root mean
    squared error] is computed immediately and automatically for a fixed but unknown half of the qualifying set (the “quiz”
  subset). This value is reported to the contestant and posted to the leader board, if appropriate. The RMSE for the other half
  of the qualifying set (the “test” subset) is not reported and is used by Netflix to identify potential winners of a Prize.

  [\textellipsis]

  In addition to providing the baseline Cinematch performance on the quiz subset, Netflix also identified a "probe" subset of
  the complete training set and the Cinematch RMSE value to permit off-line comparison with systems before submission[.]

  \textbf{3. Formation of the Training Set}

  Two separate random sampling processes were employed to compose first the entire Prize dataset and then the quiz, test, and
  probe subsets used to evaluate the performance of contestant systems.

  The complete Prize dataset (the training set, which contains the probe subset, and the qualifying set, which comprises the
  quiz and test subsets) was formed by randomly selecting a subset of all users who provided at least 20 ratings between
  October, 1998 and December, 2005. All their ratings were retrieved. To protect some information about the Netflix subscriber
  base [5], a perturbation technique was then applied to the ratings in that dataset. The perturbation technique was designed to
  not change the overall statistics of the Prize dataset. However, the perturbation technique will not be described since that
  would defeat its purpose.

  The qualifying set was formed by selecting, for each of the randomly selected users in the complete Prize dataset, a set of
  their most recent ratings. These ratings were randomly assigned, with equal probability, to three subsets: quiz, test, and
  probe. Selecting the most recent ratings reflects the Netflix business goal of predicting future ratings based on past
  ratings. The training set was created from all the remaining (past) ratings and the probe subset; the qualifying set was
  created from the quiz and test subsets. The training set ratings were released to contestants; the qualifying ratings were
  withheld and form the basis of the contest scoring system.
\end{quote}

Here is our formal interpretation of the above:
\begin{emphquote}[Netflix Prize competition, full version on a sample]
  Netflix has a finite set $A$ of users and a finite set $B$ of movies. This information we know. Netflix also has a
  confidential partial function $F\colon A\times B\rightharpoonup\{0,1,\ldots,5\}\times T$ (i.e., a partially filled $A\times B$
  matrix), where $T$ is a set of possible ``dates of rating''. Netflix chooses randomly a $\rho$-proportion of the filled
  entries $(a,b)$ of the matrix and provides us with their labels (i.e., with all such triples $(a,b,F(a,b))$). We are tasked
  with guessing the correct rating (but not the date of rating)\footnote{In fact, in the actual competition, per
  Bennett--Lanning, even in for erased entries, we are provided with the date of rating.} for all other pairs $(a,b)$ in the
  matrix $A\times B$. We are allowed to answer fractional values and we are judged according to the mean square distance of our
  guess from the actual values of the matrix.
\end{emphquote}
Recall both from the simplified version and the account of Bennett--Lanning~\cite{BL07} that we see the problem above as
happening after $A$ and $B$ got randomly sampled from much larger sets of users $\cA$ and movies $\cB$, respectively.

Let us now comment on the differences of the above to the simplified version and why they should not matter:
\begin{itemize}
\item The ratings are not $0$ or $1$, but rather one of finitely many values; this is actually covered by our theory.
\item The fact that we are allowed to guess fractional values should also not affect learnability, since rounding them to the
  nearest integer value is plausible to only incur small error (say, if it the error was $\epsilon$ before rounding, then with
  high probability the error should be $\epsilon^{\Omega(1)}$ after rounding).
\item The fact that the matrix is partially filled can be encoded by simply adding an extra label that means ``rating not
  known'', which does not incur any penalty if guessed incorrectly.
\item The date of rating can be considered a part of the label that is ignored by the loss function, but is provided to us and
  can be used to improve learning. We will elaborate on that at the end of this section.
\end{itemize}

If we accept that this is a correct formulation of the Netflix Prize competition, then our main theorem has a strong prediction
about algorithms succeeding in the competition. Namely, their underlying hypothesis class must have finite
$\VCN_{2,2}$-dimension. We now examine this prediction.

\paragraph*{Why do the winning algorithms have finite $\VCN_{2,2}$-dimension?} According to the accounts
of~\cite{Kor09,TJ09,PC09}, the best algorithms are actually a blend of several different algorithms and remarkably, for all of
those that we investigated, we can provide a reasonable explanation of why the underlying hypothesis class has finite
$\VCN_{2,2}$-dimension. In the paragraph below, we provide details so that interested readers can contribute further to the
picture.

First, let us address the blend of algorithms itself: any kind of weighted combination of algorithms all of which have finite
$\VCN_{2,2}$-dimension has itself finite $\VCN_{2,2}$-dimension (which is at most the sum of the dimensions). Second, several of
the algorithms fall under variations of the following principle: we assume that each user $a$ has a fixed number of features
$v_{a,1},\ldots,v_{a,t}$, each of which is a vector in some fixed dimensional space $v_{a,i}\in\RR^{d_i}$ and the same for each
movie $b$ having features $w_{b,i}\in\RR^{d_i}$. The rating is then determined by a formula of the form
\begin{equation}\label{eq:boundedrankext}
  F(a,b) \df \sum_{i=1}^t v_{a,i}\cdot w_{b,i}.
\end{equation}
As is, such rating clearly only generates matrices of rank at most $r\df\sum_{i=1}^t d_i$, i.e., all such $F$ are in the
hypothesis class $\cH_r$ of Example~\ref{ex:boundedrank}, which has $\VCN_{2,2}$-dimension $r$.

Variations of these classifications involve the following:
\begin{itemize}
\item We might have fixed functions $g_i\colon\RR\to\RR$ and the rating is determined by
  \begin{equation*}
    F(a,b) \df \sum_{i=1}^t g(v_{a,i}\cdot w_{b,i}).
  \end{equation*}
  While it is no longer true that the resulting matrices have bounded rank, since the $g_i$ are fixed, it is not hard to argue
  that the resulting class still has $\VCN_{k,k}$-dimension at most $r\df\sum_{i=1}^t d_i$.
\item One way we can interpret the classifier in~\eqref{eq:boundedrankext} is by forming matrices $V_i\in\RR^{A\times d_i}$ for
  the features of all users in the sample and matrices $W_i\in\RR^{d_i\times B}$ for the features of all movies in the sample
  and the classfier is given by the sum of matrix products
  \begin{equation*}
    F \df \sum_{i=1}^t V_i\cdot W_i.
  \end{equation*}

  Another variation is to instead compute a different matrix product based on $K$-nearest neighbors ($K\in\NN_+$ here is fixed).
  We conjecture that the resulting hypothesis class still has bounded $\VCN_{2,2}$-dimension based on the fact that $K$-nearest
  neighbors should have a local effect. It may be interesting for a reader of this paper to investigate further.
\end{itemize}
\paragraph*{Implicit assumptions that some algorithms seem to be making, which abstractly guarantees finite
  $\VCN_{2,2}$-dimension.} One of the recurring themes in the algorithms above (and in the overall treatment of the problem
in~\cite{Kor09,TJ09,PC09}) is the belief that there are a fixed amount of features that a user can have and a fixed amount of
features that a movie can have and once one knows the features, there is a global rule that maps them to a rating. One way to
interpret this is that they expect all hypotheses to actually factor as
\begin{equation*}
  F(a,b) = h(g_1(a),g_2(b)),
\end{equation*}
where $g_1$ is a function in some unary hypothesis class $\cH_1\subseteq Y_1^{\cA}$, $g_2$ is a function in some unary
hypothesis class $\cH_2\subseteq Y_2^{\cB}$ (with both $Y_1$ and $Y_2$ finite) and $h\colon Y_1\times Y_2\to\{0,\ldots,5\}$ is a
fixed rule of how the rating is deduced from the hidden features.

By appealing to the connection of $\VCN_{2,2}$-dimension to growth functions (see Section~\ref{sec:techexp}), one can show that
all such hypothesis class have finite $\VCN_{2,2}$-dimension.
\paragraph*{What about the timestamps?} In the Netflix Prize competition, we are actually provided the timestamps of when the rating
actually happened and several algorithms in~\cite{Kor09,TJ09,PC09} indeed use these timestamps (a common usage is to reweight
ratings, giving priority to newer ones). For the purposes of estimating how this affects the $\VCN_{2,2}$-dimension, we might
interpret this usage of timestamps as follows: the timestamps $t(a,b)$ of each user-movie pair determines an underlying linear
order of the user-movie pairs in $A\times B$ and the algorithms have access to questions of the form ``Is $t(a,b)\leq r_i$?''
for a finite collection of times $r_1,\ldots,r_s$. Even if we add this extra layer to the basically unary strategies described
in the previous item, it is not too difficult to see that the $\VCN_{2,2}$-dimension remains finite.
\paragraph*{How does the theory help us go further?} So far we have seen how the theory developed in this paper can explain the
success of existing algorithms. However, can we actually use this theory to suggest how to improve the learning power and design
better algorithms? Indeed, the theory provides an actual ceiling of sample completion learnability, namely, that of finite
$\VCN_{2,2}$-dimension of the background hypothesis class. This is much more power than what existing algorithms seem to use.
Let us give here an example of a hypothesis class that
\begin{enumerate*}[label={(\roman*)}]
\item has finite $\VCN_{2,2}$-dimension,
\item is not a combination of a basically unary strategy along with the linear order of timestamps,
\item does not seem to have been explored in any of the algorithms in~\cite{Kor09,TJ09,PC09}, and
\item seems natural to the Netflix problem.
\end{enumerate*}

Since the timestamps include day and month, they also induce a natural cyclic order on the ratings corresponding to the year
cycles. This means that our algorithm could have access to the (365-valued) question ``On which day of the year was this
rated?''. This still generates a finite $\VCN_{2,2}$-dimension hypothesis class, which is not of any of the previously discussed
forms (but can be combined with them). Furthermore, it seems natural that ratings might have some underlying seasonality to them
which could be exploited to improve algorithms.

To conclude, the theory can potentially contribute to practice in at least two ways: first by suggesting larger hypothesis
classes that are still finite $\VCN_{2,2}$-dimension, hence guaranteed to be at least qualitatively learnable, so can serve as
guides for the development of better algorithms; second, before one actually implements an algorithm, which can be
time-consuming and costly, instead one can use the characterization proved in this paper and first investigate the plausibility
that the algorithmic ideas have an underlying hypothesis class that has finite $\VCN_{2,2}$-dimension.

\clearpage

\phantomsection
\label{tableofcontentsmarker}
\tableofcontents

\afterpage{%
  \begingroup
\def\smallbend{14}
\begin{landscape}
  \vfill
  \begin{figure}[p]
    \centering
    \begin{small}
      \begin{tikzcd}[ampersand replacement=\&]
        \begin{tabular}{c}
          Non-partite finite\\
          $\VCN_{k,k}$-dimension
        \end{tabular}
        \arrow[dddd, twoheadleftarrow, two heads, dashed, "{\scriptsize\shortstack{\ref{prop:VCNkk}}}"']
        \arrow[r, tail,
          "{\scriptsize\shortstack{\ref{prop:VCNkk->SHP}}}",
          "{\scriptsize\shortstack{$\ell$ separated}}"']
        \arrow[dr, tail,
          "{\scriptsize\shortstack{\ref{prop:VCNkk->SUC}}}" pos={0.7},
          "{\scriptsize\shortstack{$\ell$ bounded, local\\$\ell$ symmetric}}"' pos={0.9}]
        \&
        \begin{tabular}{c}
          Non-partite\\
          $\omega(m^{k-1/(d+1)^{k-1}}\cdot\ln m)$-sample\\
          Haussler packing property
        \end{tabular}
        \arrow[r, "{\scriptsize\shortstack{\ref{rmk:SHP->PHP}}}"]
        \&
        \begin{tabular}{c}
          Non-partite $m^k$-sample\\
          Haussler packing property
        \end{tabular}
        \arrow[r, two heads, "{\scriptsize\shortstack{\ref{rmk:SHP->PHP}}}"]
        \&
        \begin{tabular}{c}
          Non-partite\\
          $m^k$-probabilistic\\
          Haussler packing property
        \end{tabular}
        \arrow[lll, bend right={\smallbend},
          "{\scriptsize\shortstack{\ref{prop:PHP->VCNkk}}}"',
          "{\scriptsize\shortstack{$\ell$ separated}}"]
        \\
        \&
        \begin{tabular}{c}
          Non-partite sample\\
          uniform convergence
        \end{tabular}
        \arrow[dr,
          "{\scriptsize\shortstack{\ref{prop:SUC->advSC}}}",
          "{\scriptsize\shortstack{Existence of completion\\(almost) empirical\\ risk minimizers}}"' pos={0.8}]
        \&
        \begin{tabular}{c}
          Non-partite adversarial\\
          sample completion\\
          $k$-PAC learnable
        \end{tabular}
        \arrow[r, two heads, dashed, "{\scriptsize\shortstack{\ref{rmk:advSC->agSC->SC}}}"']
        \&
        \begin{tabular}{c}
          Non-partite\\
          sample completion\\
          $k$-PAC learnable
        \end{tabular}
        \arrow[u,
          "{\scriptsize\shortstack{\ref{prop:SC->PHP}}}"',
          "{\scriptsize\shortstack{$\ell$ metric or\\$\ell$ separated and bounded}}"]
        \\
        \&
        \&
        \begin{tabular}{c}
          Non-partite adversarial\\
          symmetric sample\\
          completion $k$-PAC learnable
        \end{tabular}
        \arrow[r, two heads, dashed, "{\scriptsize\shortstack{\ref{rmk:advSC->agSC->SC}}}"']
        \arrow[u, "{\scriptsize\shortstack{\ref{rmk:symm->nonsymm}}}"']
        \&
        \begin{tabular}{c}
          Non-partite\\
          symmetric sample\\
          completion $k$-PAC learnable
        \end{tabular}
        \arrow[u, "{\scriptsize\shortstack{\ref{rmk:symm->nonsymm}}}"']
        \\[0.5cm]
        \&
        \begin{tabular}{c}
          Partite sample\\
          uniform convergence
        \end{tabular}
        \arrow[r,
          "{\scriptsize\shortstack{\ref{prop:SUC->advSC}}}"',
          "{\scriptsize\shortstack{Existence of completion\\(almost) empirical\\ risk minimizers}}"]
        \&
        \begin{tabular}{c}
          Partite adversarial\\
          sample completion\\
          $k$-PAC learnable
        \end{tabular}
        \arrow[r, two heads, dashed, "{\scriptsize\shortstack{\ref{rmk:advSC->agSC->SC}}}"]
        \&
        \begin{tabular}{c}
          Partite\\
          sample completion\\
          $k$-PAC learnable
        \end{tabular}
        \arrow[d,
          "{\scriptsize\shortstack{\ref{prop:SC->PHP}}}",
          "{\scriptsize\shortstack{$\ell$ metric or\\$\ell$ separated and bounded}}"']
        \\
        \begin{tabular}{c}
          Partite finite\\
          $\VCN_{k,k}$-dimension
        \end{tabular}
        \arrow[ur, tail,
          "{\scriptsize\shortstack{\ref{prop:VCNkk->SUC}}}"',
          "{\scriptsize\shortstack{$\ell$ bounded and local}}" pos={0.9}]
        \arrow[r, tail,
          "{\scriptsize\shortstack{\ref{prop:VCNkk->SHP}}}"',
          "{\scriptsize\shortstack{$\ell$ separated}}"]
        \&
        \begin{tabular}{c}
          Partite\\
          $\omega(m^{k-1/(d+1)^{k-1}}\cdot\ln m)$-sample\\
          Haussler packing property
        \end{tabular}
        \arrow[r, "{\scriptsize\shortstack{\ref{rmk:SHP->PHP}}}"']
        \&
        \begin{tabular}{c}
          Partite $m^k$-sample\\
          Haussler packing property
        \end{tabular}
        \arrow[r, two heads, "{\scriptsize\shortstack{\ref{rmk:SHP->PHP}}}"']
        \&
        \begin{tabular}{c}
          Partite\\
          $m^k$-probabilistic\\
          Haussler packing property
        \end{tabular}
        \arrow[lll, bend left={\smallbend},
          "{\scriptsize\shortstack{\ref{prop:PHP->VCNkk}}}",
          "{\scriptsize\shortstack{$\ell$ separated}}"']
      \end{tikzcd}
      \captionof{figure}{Diagram of results proved in this document. Labels on arrows contain the number of the
        proposition/remark that contains the proof of the implication and extra hypotheses needed. Arrows with two heads
        ($\twoheadrightarrow$) are tight in some sense with a straightforward proof of tightness. Dashed arrows involve a
        construction (meaning that either the hypothesis class changes and/or the loss function changes) due to being in
        different settings; this also means that objects in one of the sides of the implication might not be completely general
        (as they are required to be in the image of the construction). Arrows with tails ($\rightarrowtail$) mean that exactly
        one of the sides involves a loss function (so when composing a solid arrow with a tailed arrow, the result might involve
        a construction that changes the loss function and thus be a dashed arrow). Under appropriate hypotheses, all items are
        proved equivalent.}
      \label{fig:roadmap}
    \end{small}
  \end{figure}
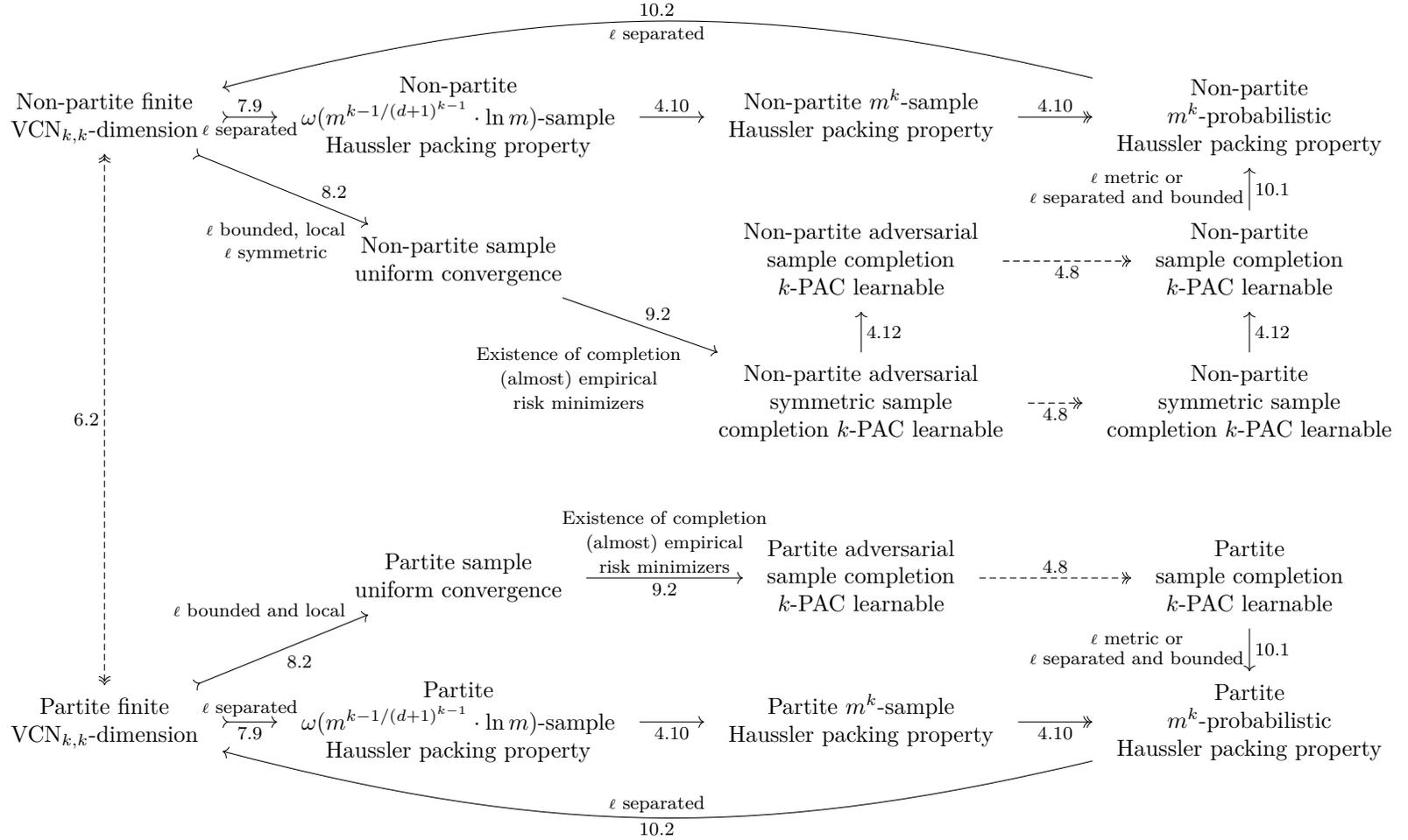
  \vfill
\end{landscape}
\endgroup
}

\clearpage

\section{Definitions}
\label{sec:defs}

In this section, we collect the definitions of the high-arity PAC theory of~\cite{CM24+,CM25+} that we will need as well as the
main definitions of the current work so that we can formally state our main theorems in Section~\ref{sec:main}. Before we start,
let us comment on a notational convention of both~\cite{CM24+,CM25+} and the current work: we will have two settings, the
partite and non-partite, and we will aim to use the same notation for concepts that are analogous to each other; this will both
make the analogy self-evident and make it easier to write proofs whose arguments are the same on both settings. Let us point out
that even though the symbols used are the same, there is no ambiguity in the notation as, e.g., for notation such as
$\cE_V(\Omega)$ (Definition~\ref{def:part:cEV} in the partite and Definition~\ref{def:nonpart:cEV} in the non-partite), the
underlying $\Omega$ is different in the settings: it is a $k$-tuple of non-empty Borel spaces in the partite and it is a single
non-empty Borel space in the non-partite (furthermore, the collision when $k=1$ is intentional as both settings coincide when
$k=1$).

We start with general notation: we denote the set of non-negative integers by $\NN$ and the set of positive integers by
$\NN_+\df\NN\setminus\{0\}$. For $m\in\NN$, we let $[m]\df\{1,\ldots,m\}$ and for a set $V$, we let $\binom{V}{m}$ be the
collection of all subsets of $V$ of cardinality $m$ and we let $(V)_m$ be the set of all injections $[m]\to V$; in particular,
we view $([m])_m$ as the symmetric group $S_m$ on $[m]$.

\subsection{Definitions from high-arity PAC}
\label{subsec:highPAC}

In this subsection, we collect the definitions from the high-arity PAC theory of~\cite{CM24+,CM25+} that we will need. The
definitions here are simplified versions do not cover ``higher-order variables'' as in~\cite{CM24+,CM25+}; for the full versions
of these definitions, we refer the interested reader to those works (and give specific pointers on where each concept can be
found).

Let us also comment on the measurability assumptions that we impose here: since the definitions of Section~\ref{subsec:highPAC}
are imported from~\cite{CM24+,CM25+}, we make the same measurability assumptions, but we point out right now that sample
completion learning requires much fewer measurability assumptions. We will discuss further in Remark~\ref{rmk:measurable}, but
the reader unfamiliar with measure theory should just interpret these assumptions as ``all probabilities and expectations need
to make sense'' and should know right away that if they only use the $0/1$-loss function (and its agnostic counterpart), then
all measurability assumptions are satisfied.

\begin{definition}[Definitions in the partite, simplified]\label{def:part}
  By a Borel space, we mean a standard Borel space, i.e., a measurable space that is Borel-isomorphic to a Polish space when
  equipped with the $\sigma$-algebra of Borel sets. The space of probability measures on a Borel space $\Lambda$ is denoted
  $\Pr(\Lambda)$.

  Let $k\in\NN_+$, let $\Omega=(\Omega_i)_{i=1}^k$ be a $k$-tuple of non-empty Borel spaces and let $\Lambda$ be a non-empty
  Borel space.
  \begin{enumdef}
  \item\label{def:part:cEV} \cite[4.1.4]{CM24+} For a finite set $V$, we let $\cE_V(\Omega)\df\prod_{i=1}^k \Omega_i^V$ be
    equipped with the product $\sigma$-algebra. We will also use the shorthand notation $\cE_m(\Omega)\df\cE_{[m]}(\Omega)$ when
    $m\in\NN$ (recall that $[m]\df\{1,\ldots,m\}$). For the particular case of $\cE_1(\Omega)$, we will simply view it as
    $\prod_{i=1}^k \Omega_i$ (as opposed to $\prod_{i=1}^k \Omega_i^{[1]}$).

    With a slight abuse of notation, we let $\Pr(\Omega)$ be the space of $k$-tuples $\mu=(\mu_i)_{i=1}^k$ where
    $\mu_i\in\Pr(\Omega_i)$ is a probability measure on $\Omega_i$. For $\mu\in\Pr(\Omega)$ and $m\in\NN$, we let
    $\mu^m\in\Pr(\cE_m(\Omega))$ be the product measure
    \begin{equation*}
      \mu^m \df \bigotimes_{i=1}^k \mu_i^m
    \end{equation*}
    (where each $\mu_i^m$ is itself the product measure of $m$-many copies of $\mu_i$).
  \item\label{def:part:alpha*} \cite[4.1.5]{CM24+} For a finite set $V$ and $\alpha\in V^k$ (i.e., a function $\alpha\colon[k]\to
    V$), we define the map $\alpha^*\colon\cE_V(\Omega)\to\cE_1(\Omega)$ by
    \begin{equation}\label{eq:part:alpha*}
      \alpha^*(x)_i \df (x_i)_{\alpha(i)} \qquad \bigl(x\in\cE_V(\Omega), i\in[k]\bigr).
    \end{equation}
  \item\cite[4.2.1]{CM24+} The set of \emph{$k$-partite hypotheses} from $\Omega$ to $\Lambda$, denoted
    $\cF_k(\Omega,\Lambda)$, is the set of (Borel) measurable functions from $\cE_1(\Omega)$ (i.e., $\prod_{i=1}^k \Omega_i$) to
    $\Lambda$.
  \item\cite[4.2.2]{CM24+}\label{def:part:hypclass} A \emph{$k$-partite hypothesis class} is a subset $\cH$ of
    $\cF_k(\Omega,\Lambda)$ equipped with a $\sigma$-algebra such that:
    \begin{enumerate}
    \item the evaluation map $\ev\colon\cH\times\cE_1(\Omega)\to\Lambda$ given by $\ev(H,x)\df H(x)$ is measurable;
    \item for every $H\in\cH$, the set $\{H\}$ is measurable;
    \item for every Borel space $\Upsilon$ and every measurable set $A\subseteq\cH\times\Upsilon$, the projection of $A$ onto
      $\Upsilon$, i.e., the set
      \begin{equation*}
        \{\upsilon\in\Upsilon \mid \exists H\in\cH, (H,\upsilon)\in A\}
      \end{equation*}
      is universally measurable\footnote{\label{ftn:measurability} This assumption about hypothesis classes is not made
      in~\cite{CM24+}, but for uniform convergence there to make sense, one needs that this is true. As we will see in
      Remark~\ref{rmk:measurable}, this measurability assumption is not necessary for sample completion learning.}.
    \end{enumerate}
  \item\label{def:part:F*V} \cite[4.2.3]{CM24+} Given $F\in\cF_k(\Omega,\Lambda)$ and a finite set $V$, we define the function
    $F^*_V\colon\cE_V(\Omega)\to\Lambda^{V^k}$ by
    \begin{equation*}
      F^*_V(x)_\alpha \df F\bigl(\alpha^*(x)\bigr)
      \qquad \bigl(x\in\cE_V(\Omega), \alpha\in V^k\bigr).
    \end{equation*}
    For $m\in[m]$, we use the shorthand $F^*_m\df F^*_{[m]}$.
  \item\label{def:part:alphasharp} \cite[4.3]{CM24+} Given a $k$-tuple $\alpha=(\alpha_i)_{i=1}^k$ of injections $\alpha_i\colon
    U\to V$ between finite sets $U$ and $V$, we contra-variantly define the map $\alpha^\#\colon\cE_V(\Omega)\to\cE_U(\Omega)$
    by
    \begin{equation*}
      \bigl(\alpha^\#(x)_i\bigr)_u \df (x_i)_{\alpha_i(u)} \qquad \bigl(x\in\cE_V(\Omega), i\in[k], u\in U\bigr)
    \end{equation*}
    and the map $\alpha^\#\colon\Lambda^{V^k}\to\Lambda^{U^k}$ by
    \begin{equation*}
      \alpha^\#(y)_\beta \df y_{\alpha_1(\beta_1),\ldots,\alpha_k(\beta_k)} \qquad (\beta\in U^k).
    \end{equation*}
    The overload of notation here is intention as these definitions make the one in~\ref{def:part:F*V} above equivariant in the
    sense that the diagram
    \begin{equation}\label{eq:F*Vequiv:part}
      \begin{tikzcd}
        \cE_V(\Omega)
        \arrow[r, "F^*_V"]
        \arrow[d, "\alpha^\#"']
        &
        \Lambda^{V^k}
        \arrow[d, "\alpha^\#"]
        \\
        \cE_U(\Omega)
        \arrow[r, "F^*_U"]
        &
        \Lambda^{U^k}
      \end{tikzcd}
    \end{equation}
    is commutative (this is a straightforward proof that can be found in~\cite[Lemma~4.3]{CM24+}).
  \item\cite[4.7.1, 4.7.2, 4.7.3, 4.7.4]{CM24+}, \cite[A.12]{CM25+} A \emph{$k$-partite loss function} over $\Lambda$ is a
    measurable function $\ell\colon\cE_1(\Omega)\times\Lambda\times\Lambda\to\RR_{\geq 0}$. We further define
    \begin{align*}
      \lVert\ell\rVert_\infty & \df \sup_{\substack{x\in\cE_1(\Omega)\\y,y'\in\Lambda}} \ell(x,y,y'),
      &
      s(\ell) & \df \inf_{\substack{x\in\cE_1(\Omega)\\y,y'\in\Lambda\\y\neq y'}} \ell(x,y,y'),
    \end{align*}
    and we say that $\ell$ is:
    \begin{description}[format={\normalfont\textit}]
    \item[bounded] if $\lVert\ell\rVert_\infty < \infty$.
    \item[separated] if $s(\ell) > 0$ and $\ell(x,y,y)=0$ for every $x\in\cE_1(\Omega)$ and every $y\in\Lambda$.
    \item[metric] if for every $x\in\cE_1(\Omega)$, the function $\ell(x,\place,\place)$ is a metric on $\Lambda$ in the usual
      sense, that is, the following hold for every $x\in\cE_1(\Omega)$ and $y,y',y''\in\Lambda$:
      \begin{enumerate}
      \item We have $\ell(x,y,y') = \ell(x,y',y)$.
      \item We have $\ell(x,y,y') = 0$ if and only if $y = y'$.
      \item We have $\ell(x,y,y'')\leq\ell(x,y,y')+\ell(x,y',y'')$.
      \end{enumerate}
    \end{description}

    If we are further given $k$-partite hypotheses $F,H\in\cF_k(\Omega,\Lambda)$ and $\mu\in\Pr(\Omega)$, then we define the
    \emph{total loss} of $H$ with respect to $\mu$, $F$ and $\ell$ as
    \begin{equation*}
      L_{\mu,F,\ell}(H)
      \df
      \EE_{\rn{x}\sim\mu^1}\Bigl[\ell\bigl(\rn{x}, H(\rn{x}), F(\rn{x})\bigr)\Bigr].
    \end{equation*}
  \item\cite[4.7.5]{CM24+} We say that $F\in\cF_k(\Omega,\Lambda)$ is \emph{realizable} in a $k$-partite hypothesis class
    $\cH\subseteq\cF_k(\Omega,\Lambda)$ with respect to a $k$-partite loss function $\ell$ and $\mu\in\Pr(\Omega)$ if
    $\inf_{H\in\cH} L_{\mu,F,\ell}(H) = 0$.
  \item\cite[4.7.6]{CM24+} The \emph{$k$-partite $0/1$-loss function} over $\Lambda$ is defined as
    $\ell_{0/1}(x,y,y')\df\One[y\neq y']$.
  \item\cite[4.10.1, 4.10.2, 4.10.3, 4.12]{CM24+} A \emph{$k$-partite agnostic loss function} over $\Lambda$ with respect to a
    $k$-partite hypothesis class $\cH$ is a measurable function $\ell\colon\cH\times\cE_1(\Omega)\times\Lambda\to\RR_{\geq 0}$.
    We further define
    \begin{equation*}
      \lVert\ell\rVert_\infty \df \sup_{\substack{H\in\cH\\ x\in\cE_1(\Omega)\\ y\in\Lambda}} \ell(H,x,y)
    \end{equation*}
    and we say that $\ell$ is:
    \begin{description}[format={\normalfont\textit}]
    \item[bounded] if $\lVert\ell\rVert_\infty < \infty$.
    \item[local] if there exists a function $r\colon\cH\to\RR$ such that for every $F,H\in\cH$, every $x\in\cE_1(\Omega)$ and
      every $y\in\Lambda$, we have
      \begin{equation*}
        F(x) = H(x) \implies \ell(F,x,y) - r(F) = \ell(H,x,y) - r(H) \geq 0.
      \end{equation*}
      A function $r$ satisfying the above is called a \emph{regularization term} of $\ell$. Equivalently, $\ell$ is local if and
      only if it can be factored as
      \begin{equation}\label{eq:localellr:part}
        \ell(H,x,y) = \ell_r\bigl(x,H(x),y\bigr) + r(H) \qquad \bigl(H\in\cH, x\in\cE_1(\Omega), y\in\Lambda\bigr)
      \end{equation}
      for some (non-agnostic) $k$-partite loss function $\ell_r\colon\cE_1(\Omega)\times\Lambda\times\Lambda\to\RR_{\geq 0}$ and
      some regularization term $r\colon\cH\to\RR$.
    \end{description}
  \item\cite[4.10.5]{CM24+} The \emph{$k$-partite agnostic $0/1$-loss function} over $\Lambda$ with respect to $\cH$ is defined
    as $\ell_{0/1}(H,x,y)\df\One[H(x)\neq y]$.
  \item\cite[4.17.1, 4.17.2]{CM24+} For $m\in\NN_+$, $x\in\cE_m(\Omega)$, $y\in\Lambda^{[m]^k}$ and
    $H\in\cF_k(\Omega,\Lambda)$, we define the \emph{empirical loss} (or \emph{empirical risk}) of $H$ with respect to $(x,y)$ and
    a $k$-partite loss function $\ell\colon\cE_1(\Omega)\times\Lambda\times\Lambda\to\RR_{\geq 0}$ as
    \begin{equation}\label{eq:emploss:part}
      L_{x,y,\ell}(H)
      \df
      \frac{1}{m^k}\sum_{\alpha\in[m]^k}
      \ell\bigl(\alpha^*(x), H^*_m(x)_\alpha, y_\alpha\bigr)
    \end{equation}
    (we define the above to be $0$ when $m=0$).

    We also define the \emph{empirical loss} (or \emph{empirical risk}) of $H$ with respect to $(x,y)$ and a $k$-partite
    agnostic loss function $\ell\colon\cH\times\cE_1(\Omega)\times\Lambda\to\RR_{\geq 0}$ as
    \begin{equation*}
      L_{x,y,\ell}(H) \df \frac{1}{m^k}\sum_{\alpha\in[m]^k}\ell\bigl(H,\alpha^*(x),y_\alpha\bigr)
    \end{equation*}
    (and define the above to be $0$ when $m=0$). 
  \end{enumdef}
\end{definition}

\begin{remark}\label{rmk:localbounded}
  Note that if $\ell$ is a local loss function that is bounded, then there is always a choice of functions $\ell_r$ and $r$ that
  satisfy~\eqref{eq:localellr:part} and are \emph{both} non-negative, which in particular implies
  \begin{align*}
    \lVert r\rVert_\infty \df \sup_{H\in\cH} \lvert r(H)\rvert & \leq \lVert\ell\rVert_\infty,
    &
    \lVert\ell_r\rVert_\infty & \leq \lVert\ell\rVert_\infty.
  \end{align*}

  Indeed, if $\ell_r$ and $r$ satisfy~\eqref{eq:localellr}, then we must have $\lVert\ell_r\rVert_\infty < \infty$ (otherwise
  fixing one $H$ and varying $x,y,y'$ would make $\ell$ unbounded). In turn, since $\ell$ is non-negative, we get that $R \df
  \inf_{H\in\cH} r(H) > -\infty$ and $\ell_r\geq \max\{0,-R\}$ everywhere. Let $C\df\max\{0,-R\}$ and note that replacing
  $(\ell_r,r)$ with $(\ell_r - C, r + C)$ also satisfies~\eqref{eq:localellr:part}, but both $\ell_r - C$ and $r + C$ are
  non-negative. A similar remark applies in the non-partite case in~\eqref{eq:localellr} below.
\end{remark}

\begin{definition}[Definitions in the non-partite, simplified]\label{def:nonpart}
  Let $\Omega=(X,\cB)$ and $\Lambda=(Y,\cB')$ be non-empty Borel spaces and $k\in\NN_+$.
  \begin{enumdef}
  \item\label{def:nonpart:cEV} \cite[3.1.4]{CM24+} For a finite set $V$, we let $\cE_V(\Omega)\df\Omega^V$ be equipped with the
    product $\sigma$-algebra. We will also use the shorthand notation $\cE_m(\Omega)\df\cE_{[m]}(\Omega)$ when $m\in\NN$, where
    $[m]\df\{1,\ldots,m\}$.
  \item\label{def:nonpart:alpha*} \cite[3.1.5]{CM25+} For an injective function $\alpha\colon U\to V$ between finite sets, we
    contra-variantly define the map $\alpha^*\colon\cE_V(\Omega)\to\cE_U(\Omega)$ by
    \begin{equation}\label{eq:alpha*}
      \alpha^*(x)_u \df x_{\alpha(u)} \qquad \bigl(x\in\cE_V(\Omega), u\in U\bigr).
    \end{equation}
  \item\cite[3.2.1]{CM24+} The set of \emph{$k$-ary hypotheses} from $\Omega$ to $\Lambda$, denoted $\cF_k(\Omega,\Lambda)$, is
    the set of (Borel) measurable functions from $\cE_k(\Omega)$ to $\Lambda$.
  \item\cite[3.2.2]{CM24+} A \emph{$k$-ary hypothesis class} is a subset $\cH$ of $\cF_k(\Omega,\Lambda)$ equipped with a
    $\sigma$-algebra such that:
    \begin{enumerate}
    \item the evaluation map $\ev\colon\cH\times\cE_k(\Omega)\to\Lambda$ given by $\ev(H,x)\df H(x)$ is measurable;
    \item for every $H\in\cH$, the set $\{H\}$ is measurable;
    \item for every Borel space $\Upsilon$ and every measurable set $A\subseteq\cH\times\Upsilon$, the projection of $A$ onto
      $\Upsilon$, i.e., the set
      \begin{equation*}
        \{\upsilon\in\Upsilon \mid \exists H\in\cH, (H,\upsilon)\in A\}
      \end{equation*}
      is universally measurable.
    \end{enumerate}
  \item\label{def:nonpart:F*V} \cite[3.2.3]{CM24+} Given $F\in\cF_k(\Omega,\Lambda)$ and a finite set $V$, we define the function
    $F^*_V\colon\cE_V(\Omega)\to\Lambda^{(V)_k}$ by
    \begin{equation*}
      F^*_V(x)_\alpha \df F\bigl(\alpha^*(x)\bigr)
      \qquad \bigl(x\in\cE_V(\Omega), \alpha\in (V)_k\bigr)
    \end{equation*}
    (recall that $(V)_k$ is the set of injections $[k]\to V$). For $m\in\NN$, we use the shorthand $F^*_m\df F^*_{[m]}$; note that when
    $k=m$, we have $F^*_k\colon\cE_k(\Omega)\to\Lambda^{S_k}$, where $S_k\df([k])_k$ is the symmetric group on $[k]$.
  \item\cite[3.2.4]{CM24+} For an injective function $\alpha\colon U\to V$ between finite sets, we also contra-variantly
    define the map $\alpha^*\colon\Lambda^{(V)_k}\to\Lambda^{(U)_k}$ by
    \begin{equation*}
      \alpha^*(y)_\beta \df y_{\alpha\comp\beta} \qquad \bigl(y\in\Lambda^{(V)_k}, \beta\in(U)_k\bigr).
    \end{equation*}
    This is intentionally the same notation as Definition~\ref{def:nonpart:alpha*} to make explicit the fact that the definition
    in~\ref{def:nonpart:F*V} above is equivariant in the sense that the diagram
    \begin{equation}\label{eq:F*Vequiv}
      \begin{tikzcd}
        \cE_V(\Omega)
        \arrow[r, "F^*_V"]
        \arrow[d, "\alpha^*"']
        &
        \Lambda^{(V)_k}
        \arrow[d, "\alpha^*"]
        \\
        \cE_U(\Omega)
        \arrow[r, "F^*_U"]
        &
        \Lambda^{(U)_k}
      \end{tikzcd}
    \end{equation}
    is commutative (this is a one-line proof that can be found in~\cite[Lemma~3.3]{CM24+}).
  \item\cite[3.7.1, 3.7.2, 3.7.3, 3.7.4]{CM24+}, \cite[A.12]{CM25+} A \emph{$k$-ary loss function} over $\Lambda$ is a
    measurable function $\ell\colon\cE_k(\Omega)\times\Lambda^{S_k}\times\Lambda^{S_k}\to\RR_{\geq 0}$. We further define
    \begin{align*}
      \lVert\ell\rVert_\infty & \df \sup_{\substack{x\in\cE_k(\Omega)\\y,y'\in\Lambda^{S_k}}} \ell(x,y,y'),
      &
      s(\ell) & \df \inf_{\substack{x\in\cE_k(\Omega)\\y,y'\in\Lambda^{S_k}\\y\neq y'}} \ell(x,y,y'),
    \end{align*}
    and we say that $\ell$ is:
    \begin{description}[format={\normalfont\textit}]
    \item[bounded] if $\lVert\ell\rVert_\infty < \infty$.
    \item[separated] if $s(\ell) > 0$ and $\ell(x,y,y)=0$ for every $x\in\cE_k(\Omega)$ and every $y\in\Lambda^{S_k}$.
    \item[symmetric] if it is $S_k$-invariant in the sense that
      \begin{equation*}
        \ell\bigl(\sigma^*(x),\sigma^*(y),\sigma^*(y')\bigr) = \ell(x,y,y')
      \end{equation*}
      for every $x\in\cE_k(\Omega)$, every $y,y'\in\Lambda^{S_k}$ and every $\sigma\in S_k$.
    \item[metric] if for every $x\in\cE_k(\Omega)$, the function $\ell(x,\place,\place)$ is a metric on $\Lambda^{S_k}$ in the
      usual sense, that is, the
      following hold for every $x\in\cE_k(\Omega)$ and $y,y',y''\in\Lambda^{S_k}$:
      \begin{enumerate}
      \item We have $\ell(x,y,y') = \ell(x,y',y)$.
      \item We have $\ell(x,y,y') = 0$ if and only if $y = y'$.
      \item We have $\ell(x,y,y'')\leq\ell(x,y,y')+\ell(x,y',y'')$.
      \end{enumerate}
    \end{description}

    If we are further given $k$-ary hypotheses $F,H\in\cF_k(\Omega,\Lambda)$ and a probability measure $\mu\in\Pr(\Omega)$, then
    we define the \emph{total loss} of $H$ with respect to $\mu$, $F$ and $\ell$ as
    \begin{equation*}
      L_{\mu,F,\ell}(H)
      \df
      \EE_{\rn{x}\sim\mu^k}\Bigl[\ell\bigl(\rn{x}, H^*_k(\rn{x}), F^*_k(\rn{x})\bigr)\Bigr].
    \end{equation*}
  \item\cite[3.7.5]{CM24+} We say that $F\in\cF_k(\Omega,\Lambda)$ is \emph{realizable} in a $k$-ary hypothesis class
    $\cH\subseteq\cF_k(\Omega,\Lambda)$ with respect to a $k$-ary loss function $\ell$ and $\mu\in\Pr(\Omega)$ if
    $\inf_{H\in\cH} L_{\mu,F,\ell}(H) = 0$.
  \item\cite[3.7.6]{CM24+} The \emph{$k$-ary $0/1$-loss function} over $\Lambda$ is defined as
    $\ell_{0/1}(x,y,y')\df\One[y\neq y']$.
  \item\cite[3.10.1, 3.10.2, 3.10.3, 3.12]{CM24+} A \emph{$k$-ary agnostic loss function} over $\Lambda$ with respect to a $k$-ary
    hypothesis class $\cH$ is a measurable function $\ell\colon\cH\times\cE_k(\Omega)\times\Lambda^{S_k}\to\RR_{\geq 0}$. We
    further define
    \begin{equation*}
      \lVert\ell\rVert_\infty \df \sup_{\substack{H\in\cH\\ x\in\cE_k(\Omega)\\ y\in\Lambda^{S_k}}} \ell(H,x,y)
    \end{equation*}
    and we say that $\ell$ is:
    \begin{description}[format={\normalfont\textit}]
    \item[bounded] if $\lVert\ell\rVert_\infty < \infty$.
    \item[symmetric] if it is $S_k$-invariant in the sense that
      \begin{equation*}
        \ell\bigl(H,\sigma^*(x),\sigma^*(y)\bigr) = \ell(H,x,y)
      \end{equation*}
      for every $H\in\cH$, every $x\in\cE_k(\Omega)$ and every $y\in\Lambda^{S_k}$.
    \item[local] if there exists a function $r\colon\cH\to\RR$ such that for every $F,H\in\cH$, every $x\in\cE_k(\Omega)$ and
      every $y\in\Lambda^{S_k}$, we have
      \begin{equation*}
        F^*_k(x) = H^*_k(x) \implies \ell(F,x,y) - r(F) = \ell(H,x,y) - r(H) \geq 0.
      \end{equation*}
      A function $r$ satisfying the above is called a \emph{regularization term} of $\ell$. Equivalently, $\ell$ is local if and
      only if it can be factored as
      \begin{equation}\label{eq:localellr}
        \ell(H,x,y) = \ell_r\bigl(x,H^*_k(x),y\bigr) + r(H) \qquad \bigl(H\in\cH, x\in\cE_k(\Omega), y\in\Lambda^{S_k}\bigr)
      \end{equation}
      for some (non-agnostic) $k$-ary loss function $\ell_r\colon\cE_k(\Omega)\times\Lambda^{S_k}\times\Lambda^{S_k}\to\RR_{\geq
        0}$ and some regularization term $r\colon\cH\to\RR$.
    \end{description}
  \item\cite[3.10.5]{CM24+} The \emph{$k$-ary agnostic $0/1$-loss function} over $\Lambda$ with respect to $\cH$ is defined
    as $\ell_{0/1}(H,x,y)\df\One[H^*_k(x)\neq y]$.
  \item\cite[7.1.1, 7.1.2, 7.1.3, 7.1.4]{CM24+} For $m\in\NN$, a \emph{($k$-ary) order choice} for $[m]$ is a sequence
    $\alpha=(\alpha_U)_{U\in\binom{[m]}{k}}$ such that for each $U\in\binom{[m]}{k}$, $\alpha_U\in([m])_k$ is an injection with
    $\im(\alpha_U) = U$.

    Any such order choice $\alpha$ defines a natural Borel-isomorphism
    $b_\alpha\colon\Lambda^{([m])_k}\to(\Lambda^{S_k})^{\binom{[m]}{k}}$ by
    \begin{equation}\label{eq:balpha}
      \bigl(b_\alpha(y)_U\bigr)_\pi \df y_{\alpha_U\comp\pi}
      \qquad \left(y\in\Lambda^{([m])_k}, U\in\binom{[m]}{k}, \pi\in S_k\right).
    \end{equation}

    If we are further given $x\in\cE_m(\Omega)$, $y\in\Lambda^{([m])_k}$ and $H\in\cF_k(\Omega,\Lambda)$, we define the
    \emph{empirical loss} (or \emph{empirical risk}) of $H$ with respect to $(x,y)$, a $k$-ary loss function
    $\ell\colon\cE_k(\Omega)\times\Lambda^{S_k}\times\Lambda^{S_k}\to\RR_{\geq 0}$ and $\alpha$ as
    \begin{equation*}
      L_{x,y,\ell}^\alpha(H)
      \df
      \frac{1}{\binom{m}{k}}\sum_{U\in\binom{[m]}{k}}
      \ell\Bigl(\alpha_U^*(x), b_\alpha\bigl(H^*_m(x)\bigr)_U, b_\alpha(y)_U\Bigr)
    \end{equation*}
    (when $m\geq k$, and defined to be $0$ if $m < k$).

    We also define the \emph{empirical loss} (or \emph{empirical risk}) of $H$ with respect to $(x,y)$, a $k$-ary agnostic loss
    function $\ell\colon\cH\times\cE_k(\Omega)\times\Lambda^{S_k}\to\RR_{\geq 0}$ and $\alpha$ as
    \begin{equation*}
      L_{x,y,\ell}^\alpha(H) \df \frac{1}{\binom{m}{k}}\sum_{U\in\binom{[m]}{k}} \ell\bigl(H,\alpha_U^*(x),b_\alpha(y)_U\bigr)
    \end{equation*}
    (when $m\geq k$, and defined to be $0$ if $m < k$).
  \end{enumdef}
\end{definition}

\begin{definition}[Partization, simplified]\label{def:kpart}
  Let $\Omega=(X,\cB)$ and $\Lambda=(Y,\cB')$ be non-empty Borel spaces and $k\in\NN_+$.
  \begin{enumdef}
  \item\cite[4.20.1]{CM24+} The \emph{$k$-partite version} of $\Omega$ is the constant $k$-tuple
    $\Omega^{\kpart}\df(\Omega,\ldots,\Omega)$ consisting of $k$ copies of $\Omega$.
  \item\cite[4.20.2]{CM24+} For $\mu\in\Pr(\Omega)$, the \emph{$k$-partite version} of $\mu$ is the constant $k$-tuple
    $\mu^{\kpart}\df(\mu,\ldots,\mu)\in\Pr(\Omega^{\kpart})$ consisting of $k$ copies of $\mu$.
  \item\cite[4.20.3]{CM24+} For a $k$-ary hypothesis $F\in\cF_k(\Omega,\Lambda)$, the \emph{$k$-partite version} of $F$ is
    the $k$-partite hypothesis $F^{\kpart}\in\cF_k(\Omega^{\kpart},\Lambda^{S_k})$ given by
    \begin{equation*}
      F^{\kpart}(x) \df F^*_k(\bigl(\iota_{\kpart}(x)\bigr) \qquad \bigl(x\in\cE_1(\Omega^{\kpart})\bigr),
    \end{equation*}
    where $\iota_{\kpart}\colon\cE_1(\Omega^{\kpart})\to\cE_k(\Omega)$ is given by
    \begin{equation}\label{eq:iotakpart:simplified}
      \iota_{\kpart}(x)_i \df x_i \qquad \bigl(x\in\cE_1(\Omega^{\kpart}), i\in[k]\bigr)
    \end{equation}
    (recall that $\cE_1(\Omega^{\kpart})$ is simply viewed as $\prod_{i=1}^k \Omega^{\kpart}_i = \prod_{i=1}^k \Omega =
    \Omega^k$).
  \item\label{def:kpart:cH} \cite[4.20.4]{CM24+} For a $k$-ary hypothesis class $\cH\subseteq\cF_k(\Omega,\Lambda)$, the
    \emph{$k$-partite version} of $\cH$ is $\cH^{\kpart}\df\{H^{\kpart} \mid H\in\cH\}$, equipped with the pushforward
    $\sigma$-algebra of the one of $\cH$. It is clear that $\iota_{\kpart}$ is a Borel-isomorphism, which in turn implies that
    $\cH\ni F\mapsto F^{\kpart}\in\cH^{\kpart}$ is a bijection and $\cH\mapsto\cH^{\kpart}$ is an injection. We denote by
    $\cH^{\kpart}\ni G\mapsto G^{\kpart,-1}\in\cH$ the inverse of $\cH\ni F\mapsto F^{\kpart}\in\cH^{\kpart}$.
  \item\cite[4.20.5]{CM24+} For a $k$-ary loss function
    $\ell\colon\cE_k(\Omega)\times\Lambda^{S_k}\times\Lambda^{S_k}\to\RR_{\geq 0}$ over $\Lambda$, the \emph{$k$-partite
  version} of $\ell$ is the $k$-partite loss function
    $\ell^{\kpart}\colon\cE_1(\Omega^{\kpart})\times\Lambda^{S_k}\times\Lambda^{S_k}\to\RR_{\geq 0}$ given by
    \begin{equation*}
      \ell^{\kpart}(x,y,y') \df \ell\bigl(\iota_{\kpart}(x),y,y')
      \qquad \bigl(\cE_1(\Omega^{\kpart}),y,y'\in\Lambda^{S_k}\bigr).
    \end{equation*}
  \item\cite[4.20.6]{CM24+} For a $k$-ary agnostic loss function
    $\ell\colon\cH\times\cE_k(\Omega)\times\Lambda^{S_k}\to\RR_{\geq 0}$, the \emph{$k$-partite version} of $\ell$ is the
    $k$-partite loss function $\ell^{\kpart}\colon\cH^{\kpart}\times\cE_1(\Omega^{\kpart})\times\Lambda^{S_k}\to\RR_{\geq 0}$
    given by
    \begin{equation*}
      \ell^{\kpart}(H,x,y) \df \ell\bigl(H^{\kpart,-1},\iota_{\kpart}(x),y\bigr)
      \qquad \bigl(H\in\cH^{\kpart}, \cE_1(\Omega^{\kpart}), y\in\Lambda^{S_k}\bigr).
    \end{equation*}
  \end{enumdef}
\end{definition}

\begin{definition}[Natarajan dimension~\cite{Nat89}]\label{def:Nat}
  Let $\cF$ be a collection of functions of the form $X\to Y$ and let $A\subseteq X$.
  \begin{enumdef}
  \item We say that $\cF$ \emph{Natarajan-shatters} $A$ if there exist functions $f_0,f_1\colon A\to Y$ such that
    \begin{enumerate}
    \item for every $a\in A$, we have $f_0(a)\neq f_1(a)$,
    \item for every $U\subseteq A$, there exists $F_U\in\cF$ such that
      \begin{equation*}
        F_U(a) = f_{\One[a\in U]}(a) =
        \begin{dcases*}
          f_0(a), & if $a\notin U$,\\
          f_1(a), & if $a\in U$
        \end{dcases*}
      \end{equation*}
      for every $a\in A$. (We will typically summarize this as $F_U(a) = f_{\One[a\in U]}(a)$.)
    \end{enumerate}
  \item The \emph{Natarajan dimension} of $\cF$ is defined as
    \begin{equation*}
      \Nat(\cF)\df\sup\{\lvert A\rvert \mid A\subseteq X\land\cF\text{ Natarajan-shatters } A\}.
    \end{equation*}
  \end{enumdef}
\end{definition}

\subsection{Sample completion versions of high-arity PAC}
\label{subsec:defSC}

This subsection contains the main definitions of the current work.

\begin{definition}[Sample completion definitions in the partite]\label{def:SCpart}
  Let $k\in\NN_+$, let $\Omega=(\Omega_i)_{i=1}^k$ be a $k$-tuple of non-empty Borel spaces, let $\Lambda=(Y,\cB')$ be a
  non-empty Borel space and let $\cH\subseteq\cF_k(\Omega,\Lambda)$ be a $k$-partite hypothesis class.
  \begin{enumdef}
  \item For $m\in\NN$, a \emph{($k$-partite) $[m]$-sample} (with respect to $\Omega$ and $\Lambda$) is an element of
    $\cE_m(\Omega)\times\Lambda^{[m]^k}$. A \emph{partially erased ($k$-partite) $[m]$-sample} (with respect to $\Omega$ and
    $\Lambda$) is an element of $\cE_m(\Omega)\times(\Lambda\cup\{\unk\})^{[m]^k}$, where $\unk$ is a special symbol assumed to
    \emph{not} be an element of $\Lambda$ (and is meant to represent that the original symbol of this entry got erased).
  \item For $m\in\NN$ and a partially erased $[m]$-sample $(x,y)\in\cE_m(\Omega)\times(\Lambda\cup\{\unk\})^{[m]^k}$, the
    \emph{partially erased empirical loss} (or \emph{partially erased empirical risk}) of a $k$-partite hypothesis
    $H\in\cF_k(\Omega,\Lambda)$ with respect to $(x,y)$ and a $k$-partite loss function
    $\ell\colon\cE_1(\Omega)\times\Lambda\times\Lambda\to\RR_{\geq 0}$ is\footnote{We use the same notation as the empirical
    loss (see~\eqref{eq:emploss:part}) intentionally: if $y$ does not have any entries $\unk$, then the partially erased
    empirical loss amounts simply to the empirical loss.}
    \begin{equation*}
      L_{x,y,\ell}(H)
      \df
      \frac{1}{\lvert\cU_y\rvert}\sum_{\alpha\in\cU_y}
      \ell\bigl(\alpha^*(x), H^*_m(y)_\alpha, y_\alpha\bigr),
    \end{equation*}
    where
    \begin{equation*}
      \cU_y
      \df
      \{\alpha\in[m]^k \mid y_\alpha\neq\unk\}.
    \end{equation*}
    If $\cU_y=\varnothing$, we set $L_{x,y,\ell}(H)\df 0$ instead.

    If we are given instead a $k$-partite agnostic loss function $\ell\colon\cH\times\cE_1(\Omega)\times\Lambda\to\RR_{\geq 0}$,
    then we define the \emph{partially erased empirical loss} (or \emph{partially erased empirical risk}) of $H\in\cH$ with
    respect to $(x,y)$ and $\ell$ similarly:
    \begin{equation*}
      L_{x,y,\ell}(H)
      \df
      \frac{1}{\lvert\cU_y\rvert}\sum_{\alpha\in\cU_y}
      \ell\bigl(H, \alpha^*(x), y_\alpha\bigr).
    \end{equation*}
  \item If $y\in\Lambda^{[m]^k}$ and $y'\in(\Lambda\cup\{\unk\})^{[m]^k}$, then we say that $y$ \emph{extends} $y'$ if $y_\alpha
    = y'_\alpha$ for every $\alpha\in([m])_k$ such that $y'_\alpha\neq\unk$.
  \item Given $y\in\Lambda^{[m]^k}$ and $\rho\in[0,1]$, the \emph{$(1-\rho)$-erasure}\footnote{It might seem weird to define in
  terms of $1-\rho$, but this is done so that $\rho$ small will correspond to an a priori harder learning task.} is the random
    element $\rn{E}_\rho(y)$ of $(\Lambda\cup\{\unk\})^{[m]^k}$ in which each entry of $y$ is replaced with $\unk$ independently
    with probability $1-\rho$.

    By construction, $y$ always extends $\rn{E}_\rho(y)$.
  \item A \emph{($k$-partite) completion algorithm}\footnote{Similarly to~\cite{SB14}, even though we use the term ``algorithm''
  here, we make no assumptions about the complexity of the function, in fact, not even about its computability. Furthermore, our
  algorithm notion here is proper: namely, it is required to return an element of $\cH$ rather than simply an arbitrary
  function. However, we point out that it is straightforward to adapt the proofs here to show that the improper version of
  sample completion learning is also equivalent to its proper counterpart.} is a measurable function
    \begin{equation*}
      \cA\colon\bigcup_{m\in\NN} \bigl(\cE_m(\Omega)\times(\Lambda\cup\{\unk\})^{[m]^k}\bigr)\to\cH,
    \end{equation*}
    where $\unk\notin\Lambda$ and $\Lambda\cup\{\unk\}$ is equipped with co-product $\sigma$-algebra.

    We want to interpret $\cA$ as receiving a $k$-partite $[m]$-sample that has been partially erased and outputting what it
    thinks was the original hypothesis from $\cH$ that generated the sample (or more generally, the hypothesis of $\cH$ that
    best explains the sample).
  \item We say that a completion algorithm $\cA$ is a \emph{(completion) empirical risk minimizer} with respect to an (agnostic
    or not) loss function $\ell$ if for every $m\in\NN$ and every partially erased $[m]$-sample
    $(x,y)\in\cE_m(\Omega)\times(\Lambda\cup\{\unk\})^{[m]^k}$, we have
    \begin{equation}\label{eq:ERM:part}
      L_{x,y,\ell}\bigl(\cA(x,y)\bigr) = \inf_{H\in\cH} L_{x,y,\ell}(H).
    \end{equation}
  \item\label{def:SCpart:SC} We say that $\cH$ is \emph{sample completion $k$-PAC learnable} with respect to a $k$-partite loss
    function $\ell\colon\cE_1(\Omega)\times\Lambda\times\Lambda\to\RR_{\geq 0}$ if there exist a completion algorithm $\cA$ and
    a function $m^{\SC}_{\cH,\ell,\cA}\colon(0,1)^3\to\RR_{\geq 0}$ such that for every $\epsilon,\delta,\rho\in(0,1)$, every
    $\mu\in\Pr(\Omega)$ and every $F\in\cF_k(\Omega,\Lambda)$ that is realizable in $\cH$ with respect to $\ell$ and $\mu$, we
    have
    \begin{equation}\label{eq:SCpart:SC}
      \PP_{\rn{x}\sim\mu^m,\rn{E}_\rho}\Biggl[
        L_{\rn{x},F^*_m(\rn{x}),\ell}\biggl(
        \cA\Bigl(\rn{x}, \rn{E}_\rho\bigl(F^*_m(\rn{x})\bigr)\Bigr)
        \biggr)
        \leq \epsilon
        \Biggr]
      \geq 1 - \delta
    \end{equation}
    for every integer $m\geq m^{\SC}_{\cH,\ell,\cA}(\epsilon,\delta,\rho)$. A remark on the notation above: the probability is
    computed as a total probability over both $\rn{x}$ picked according to $\mu^m$ and the $(1-\rho)$-erasure $\rn{E}_\rho$,
    which is done independently from $\rn{x}$.

    A completion algorithm $\cA$ satisfying the above is called a \emph{sample completion $k$-PAC learner} for $\cH$ with
    respect to $\ell$.
  \item We say that $\cH$ is \emph{adversarial sample completion $k$-PAC learnable} with respect to $\ell$ if there exist a
    completion algorithm $\cA$ and a function $m^{\advSC}_{\cH,\ell,\cA}\colon(0,1)^3\to\RR_{\geq 0}$ such that for every
    $\epsilon,\delta,\rho\in(0,1)$ and every $[m]$-sample $(x,y)\in\cE_m(\Omega)\times\Lambda^{[m]^k}$, we have
    \begin{equation*}
      \PP_{\rn{E}_\rho}\biggl[
        L_{x,y,\ell}\Bigl(
        \cA\bigl(x,\rn{E}_\rho(y)\bigr)
        \Bigr)
        \leq \inf_{H\in\cH} L_{x,y,\ell}(H) + \epsilon
        \biggr]
      \geq 1 - \delta.
    \end{equation*}

    A completion algorithm $\cA$ satisfying the above is called an \emph{adversarial sample completion $k$-PAC learner} for
    $\cH$ with respect to $\ell$.
  \item Let $(x,y)\in\cE_m(\Omega)\times(\Lambda\cup\{\unk\})^{[m]^k}$ be a partially erased $[m]$-sample and let
    $y'\in\Lambda^{[m]^k}$ extend $y$. For $\epsilon > 0$, we say that $(x,y)$ is \emph{$\epsilon$-representative} with respect
    to $\cH$, $y'$ and $\ell$ if
    \begin{equation*}
      \bigl\lvert L_{x,y,\ell}(H) - L_{x,y',\ell}(H)\bigr\rvert \leq \epsilon
    \end{equation*}
    for every $H\in\cH$. Note that in the above, the first $L$ is the partially erased empirical loss, while the second is the
    (usual) empirical loss\footnote{This also highlights a big difference between sample completion notions and classical
    PAC notions (high-arity or not): in sample completion, the partially erased empirical loss plays the role that (usual)
    empirical loss plays in classical PAC, whereas the (usual) empirical loss plays the role that total loss plays in classical
    PAC.}.
  \item\label{def:SCpart:SUC} We say that $\cH$ has the \emph{sample uniform convergence property} with respect to $\ell$ if
    there exists a function $m^{\SUC}_{\cH,\ell}\colon(0,1)^3\to\RR_{\geq 0}$ such that for every
    $\epsilon,\delta,\rho\in(0,1)^3$, every integer $m\geq m^{\SUC}_{\cH,\ell}(\epsilon,\delta,\rho)$ and every $[m]$-sample
    $(x,y)\in\cE_m(\Omega)\times\Lambda^{[m]^k}$, we have
    \begin{equation*}
      \PP_{\rn{E}_\rho(y)}\Bigl[
        \bigl(x, \rn{E}_\rho(y)\bigr)
        \text{ is $\epsilon$-representative w.r.t.\ $\cH$, $y$ and $\ell$}
        \Bigr]
      \geq 1 - \delta.
    \end{equation*}
  \item For $\epsilon > 0$, $m\in\NN$ and $x\in\cE_m(\Omega)$, we say that a (finite) sequence $(H_1,\ldots,H_t)$ of $k$-partite
    hypotheses is \emph{$\epsilon$-separated on $x$} with respect to a $k$-partite loss function
    $\ell\colon\cE_1(\Omega)\times\Lambda\times\Lambda\to\RR_{\geq 0}$ if
    \begin{equation*}
      L_{x,(H_i)^*_m(x),\ell}(H_j) > \epsilon
    \end{equation*}
    for every $i,j\in[m]$ with $i < j$.
  \item For a function $h\colon\NN\to\NN$, we say that $\cH$ has the \emph{($k$-partite) $h$-sample Haussler packing property}
    with respect to a $k$-partite loss function $\ell\colon\cE_1(\Omega)\times\Lambda\times\Lambda\to\RR_{\geq 0}$ if there
    exists a function $m^{\hPHP}_{\cH,\ell}\colon(0,1)^3\to\RR_{\geq 0}$ such that for every $\epsilon,\delta,\rho\in(0,1)$ and
    every integer $m\geq m^{\hSHP}_{\cH,\ell}(\epsilon,\delta,\rho)$, if $(H_1,\ldots,H_t)\in\cH^t$ with $t\geq 2^{\rho\cdot
      h(m)}$, then $(H_1,\ldots,H_t)$ is \emph{not} $\epsilon$-separated on $x$ w.r.t.\ $\ell$. In plain English, this means
    that we cannot pack $t$ many elements of $\cH$ so that they are pairwise $\epsilon$-far apart on $x$.
  \item For a function $h\colon\NN\to\NN$, we say that $\cH$ has the \emph{($k$-partite)
  $h$-probabilistic Haussler packing property} with respect to a $k$-partite loss function
    $\ell\colon\cE_1(\Omega)\times\Lambda\times\Lambda\to\RR_{\geq 0}$ if there exists a function
    $m^{\hPHP}_{\cH,\ell}\colon(0,1)^3\to\RR_{\geq 0}$ such that for every $\epsilon,\delta,\rho\in(0,1)$, every integer $m\geq
    m^{\hPHP}_{\cH,\ell}(\epsilon,\delta,\rho)$ and every $(H_1,\ldots,H_t)\in\cH^t$ with $t\geq 2^{\rho\cdot h(m)}$, we have
    \begin{equation*}
      \PP_{\rn{x}\sim\mu^m}\bigl[
        (H_1,\ldots,H_t)\text{ is $\epsilon$-separated on $\rn{x}$ w.r.t.\ $\ell$}
        \bigr]
      \leq \delta.
    \end{equation*}
  \item\label{def:SCpart:VCNkk} For $m\in\NN$, we write\footnote{A small remark on the notation: the first $k$ denotes the
  arity of the hypothesis class, while the second $k$ denotes the ``level'' of the learning task. This is both to differentiate
  from the $\VCN_k$-dimension that controls the (non-sample) $k$-PAC learning notion of~\cite{CM24+}, to connect to
  $k$-dependence~\cite{She14} and the $\VC_k$-dimension of~\cite{CT20,TW22} that controls hypergraphs regularity lemmas that are
  tame in the top level and to anticipate future work that will provide $k$-ary/$k$-partite learning theories of all ``levels''
  $\ell\in[k]$ (in particular, the $\VCN_k$-dimension of \cite{CM24+} will then be rebaptized as $\VCN_{k,1}$-dimension).}
    $\VCN_{k,k}(\cH)\geq m$ if there exists $x\in\cE_m(\Omega)$ such that
    \begin{equation}\label{eq:cHx:partite}
      \cH_x \df \{H^*_m(x) \mid H\in\cH\} \subseteq \Lambda^{[m]^k}
    \end{equation}
    Natarajan-shatters $[m]^k$ (see Definition~\ref{def:Nat}).

    The \emph{Vapnik--Chervonenkis--Natarajan $(k,k)$-dimension} of $\cH$, denoted $\VCN_{k,k}(\cH)$, is the largest $m\in\NN$
    such that $\VCN_{k,k}(\cH)\geq m$ (and if this holds for every $m\in\NN$, then we write $\VCN_{k,k}(\cH)=\infty$).
  \end{enumdef}
\end{definition}

\begin{remark}\label{rmk:ERM}
  A technicality regarding empirical risk minimizers analogous to the one in~\cite[Remark~4.18]{CM24+} happens here: completion
  empirical risk minimizers might not exist due to the infimum in~\eqref{eq:ERM:part} not being attained. Again, it will be
  clear from the proofs that for sample completion learnability, it will suffice to consider \emph{almost} empirical risk
  minimizers in the sense that~\eqref{eq:ERM:part} holds with an extra additive term $f(m)$ on the right-hand side for some
  function $f\colon\NN_+\to\RR_{\geq 0}$ with $\lim_{m\to\infty} f(m) = 0$. Nonetheless, even (completion) almost empirical risk
  minimizers might not exist due to measurability issues if the loss function and hypothesis class are too wild. Nevertheless,
  in most applications, the fact that algorithms are (efficiently) implemented implicitly gives measurability.

  However, we point out one major difference between sample high-arity PAC and high-arity PAC regarding empirical risk
  minimizers: if $\Lambda$ is finite and $\ell$ is a non-agnostic $k$-ary loss function, then for a fixed $(x,y)$, the partially
  erased empirical loss $L_{x,y,\ell}(H)$ can only take at most $\lvert\Lambda\rvert^{m^k}$ values, i.e., finitely many, which
  means that the infimum in~\eqref{eq:ERM:part} is indeed attained. A similar observation also holds in the non-partite case.
\end{remark}

\begin{remark}\label{rmk:advSC->agSC->SC}
  It is clear that if $\ell$ is a $k$-partite loss function and we define the $k$-partite agnostic loss function
  $\ell^{\ag}\colon\cH\times\cE_k(\Omega)\times\Lambda\to\RR_{\geq 0}$ by
  \begin{equation}\label{eq:ellag}
    \ell^{\ag}(H,x,y) \df \ell\bigl(x, H(x), y\bigr) \qquad \bigl(H\in\cH, x\in\cE_1(\Omega), y\in\Lambda\bigr),
  \end{equation}
  then adversarial sample completion $k$-PAC learnability of $\cH$ with respect to $\ell^{\ag}$ implies sample completion
  $k$-PAC learnability of $\cH$ with respect to $\ell$ (with the same learner $\cA$ and same bounds $m^{\SC}_{\cH,\ell,\cA}\df
  m^{\advSC}_{\cH,\ell^{\ag},\cA}$). This follows simply by conditioning on the outcome $\rn{x}\sim\mu^m$ of the sample in the
  non-adversarial version.

  We also point out that $\ell^{\ag}$ is clearly local and if $\ell$ is bounded, then so is $\ell^{\ag}$ (the proof is
  straightforward, but it is made explicit in~\cite[Proposition~6.3]{CM24+}).

  Finally, we could have also defined a notion of agnostic sample completion $k$-PAC learnability which is a priori in between
  adversarial and (standard) sample completion $k$-PAC learnability: namely, the loss is agnostic, but the adversary is not
  allowed to pick an $[m]$-sample adversarially and must instead sample it at random from an ``agnostic distribution''. The
  precise meaning of ``agnostic distribution'' here is a finite marginal of a separately exchangeable distribution;
  see~\cite[Propostion~4.9, Definitions~4.10 and~4.11]{CM24+} for more details. We would then have a chain of trivial
  implications adversarial $\implies$ agnostic $\implies$ standard. Since a consequence of the main result of this paper is that
  standard sample completion $k$-PAC learnability also implies the adversarial version, the agnostic one is then also equivalent
  to the other two. However, differently from~\cite{CM24+}, we do not currently have any particular application/result
  that requires specifically this agnostic version, so we refrain from stating it formally here. As expected, a similar
  observation applies in the non-partite case, in which ``agnostic distribution'' means a finite marginal of a (jointly)
  exchangeable distribution in the non-partite case; see~\cite[Proposition~3.9, Definitions~3.10 and~3.11 and paragraphs that
    precede them]{CM24+} for more details.
\end{remark}

\begin{remark}\label{rmk:SUC}
  Pedantically, it would be more correct to call the notion in Definition~\ref{def:SCpart:SUC}, adversarial sample uniform
  convergence and allow for an agnostic variant and a non-agnostic variant defined in analogy to Definition~\ref{def:SCpart:SC}
  and Remark~\ref{rmk:advSC->agSC->SC}. However, similarly to Remark~\ref{rmk:advSC->agSC->SC}, we would trivially have the
  implications adversarial $\implies$ agnostic $\implies$ standard. In turn, we will show in Proposition~\ref{prop:SUC->advSC}
  that adversarial sample uniform convergence implies adversarial sample completion $k$-PAC learnability and its proof is easily
  adapted to yield agnostic and standard versions of this implication. Finally, since a consequence of the main result of this
  paper is that standard sample completion $k$-PAC learnability implies adversarial sample uniform convergence, we refrain from
  stating formally all these variations of sample uniform convergence here. Again, an analogous observation holds in the
  non-partite case.
\end{remark}

\begin{remark}\label{rmk:SHP->PHP}
  We will abuse notation slightly by writing, for example, $m^k$-sample Haussler packing property for when we mean $h$-sample
  Haussler packing property for $h(m)\df m^k$.

  It is clear from definitions that the $h$-sample Haussler packing property implies the $h$-probabilistic Haussler packing
  property by simply conditioning on the outcome $\rn{x}\sim\mu^m$ of the sample in the probabilistic version. It is also clear
  that if $h_1,h_2\colon\NN\to\NN$ are such that $h_1\leq O(h_2)$ (i.e., we have $\limsup_{m\to\infty}
  h_1(m)/h_2(m)\leq\infty$), then the $h_1$-sample Haussler packing property implies the $h_2$-sample Haussler packing property
  with
  \begin{equation*}
    m^{\hPHP[h_2]}_{\cH,\ell}(\epsilon,\delta,\rho)
    \df
    \min_{C}
    \min\left\{m_0\in\NN
    \;\middle\vert\;
    m_0\geq m^{\hPHP[h_1]}\left(\epsilon,\delta,\frac{\rho}{C}\right)
    \land\forall m\geq m_0, \frac{h_1(m)}{h_2(m)}\leq C\right\},
  \end{equation*}
  where the outer minimum is over
  \begin{equation*}
    C > \max\left\{\limsup_{m\to\infty} \frac{h_1(m)}{h_2(m)}, 1\right\}.
  \end{equation*}
  A similar remark holds for the probabilistic Haussler packing property.

  As expected, similar observations apply in the non-partite case.
\end{remark}

\begin{definition}[Sample completion definitions in the non-partite]\label{def:SC}
  Let $k\in\NN_+$, let $\Omega=(X,\cB)$ and $\Lambda=(Y,\cB')$ be non-empty Borel spaces and let
  $\cH\subseteq\cF_k(\Omega,\Lambda)$ be a $k$-ary hypothesis class.
  \begin{enumdef}
  \item For $m\in\NN$, a \emph{($k$-ary) $[m]$-sample} (with respect to $\Omega$ and $\Lambda$) is an element of
    $\cE_m(\Omega)\times\Lambda^{([m])_k}$. A \emph{partially erased ($k$-ary) $[m]$-sample} (with respect to $\Omega$ and
    $\Lambda$) is an element of $\cE_m(\Omega)\times(\Lambda\cup\{\unk\})^{([m])_k}$, where $\unk$ is a special symbol assumed
    to \emph{not} be an element of $\Lambda$ (and is meant to represent that the original symbol of this entry got erased).
  \item For $m\in\NN$ and a partially erased $[m]$-sample $(x,y)\in\cE_m(\Omega)\times(\Lambda\cup\{\unk\})^{([m])_k}$, the
    \emph{partially erased empirical loss} (or \emph{partially erased empirical risk}) of a $k$-ary hypothesis
    $H\in\cF_k(\Omega,\Lambda)$ with respect to $(x,y)$ a $k$-ary loss function
    $\ell\colon\cE_k(\Omega)\times\Lambda^{S_k}\times\Lambda^{S_k}\to\RR_{\geq 0}$ and an order choice $\alpha$ for $[m]$
    is
    \begin{equation*}
      L_{x,y,\ell}^\alpha(H)
      \df
      \frac{1}{\lvert\cU_y\rvert}\sum_{U\in\cU_y}
      \ell\Bigl(\alpha_U^*(x), b_\alpha\bigl(H^*_m(x)\bigr)_U, b_\alpha(y)_U\Bigr),
    \end{equation*}
    where
    \begin{equation}\label{eq:cUy}
      \begin{aligned}
        \cU_y
        & \df
        \left\{U\in\binom{[m]}{k} \;\middle\vert\;
        \forall\beta\in([m])_k, \bigl(\im(\beta)=U\to y_\beta\neq\unk\bigr)\right\}
        \\
        & =
        \left\{U\in\binom{[m]}{k} \;\middle\vert\;
        \unk\notin\im\bigl(b_\alpha(y)_U\bigr)\right\}.
      \end{aligned}
    \end{equation}
    If $\cU_y=\varnothing$, we set $L_{x,y,\ell}^\alpha(H)\df 0$ instead.

    If we are given instead a $k$-ary agnostic loss function $\ell\colon\cH\times\cE_k(\Omega)\times\Lambda^{S_k}\to\RR_{\geq
      0}$, then we define the \emph{partially erased empirical loss} (or \emph{partially erased empirical risk}) of $H\in\cH$
    with respect to $(x,y)$, $\ell$ and an order choice $\alpha$ for $[m]$ similarly:
    \begin{equation*}
      L_{x,y,\ell}^\alpha(H)
      \df
      \frac{1}{\lvert\cU_y\rvert}\sum_{U\in\cU_y}
      \ell\bigl(H, \alpha_U^*(x), b_\alpha(y)_U\bigr).
    \end{equation*}
  \item If $y\in\Lambda^{([m])_k}$ and $y'\in(\Lambda\cup\{\unk\})^{([m])_k}$, then we say that $y$ \emph{extends} $y'$ if
    $y_\alpha = y'_\alpha$ for every $\alpha\in([m])_k$ such that $y'_\alpha\neq\unk$.
  \item Given $y\in\Lambda^{([m])_k}$ and $\rho\in[0,1]$, the \emph{$(1-\rho)$-erasure} is the random element $\rn{E}_\rho(y)$
    of $(\Lambda\cup\{\unk\})^{([m])_k}$ in which each entry of $y$ is replaced with $\unk$ independently with probability
    $1-\rho$.

    Similarly, the \emph{symmetric $(1-\rho)$-erasure} is the random element $\rn{E}^{\sym}_\rho(y)$ of
    $(\Lambda\cup\{\unk\})^{([m])_k}$ obtained from $y$ through the following procedure: for each $U\in\binom{[m]}{k}$, with
    probability $1-\rho$, independently from other elements of $\binom{[m]}{k}$, we replace all entries of $y$ indexed by all
    $\beta\in([m])_k$ with $\im(\beta)=U$ with $\unk$.

    By construction, $y$ always extends $\rn{E}_\rho(y)$ and $\rn{E}_\rho^{\sym}(y)$.
  \item A \emph{($k$-ary) completion algorithm} is a measurable function
    \begin{equation*}
      \cA\colon\bigcup_{m\in\NN}\bigl(\cE_m(\Omega)\times(\Lambda\cup\{\unk\})^{([m])_k}\bigr)\to\cH,
    \end{equation*}
    where $\unk\notin\Lambda$ and $\Lambda\cup\{\unk\}$ is equipped with co-product $\sigma$-algebra.

    We want to interpret $\cA$ as receiving a $k$-ary $[m]$-sample that has been partially erased and outputting what it thinks
    was the original hypothesis from $\cH$ that generated the sample (or more generally, the hypothesis of $\cH$ that best
    explains the sample).
  \item We say that a completion algorithm $\cA$ is a \emph{(completion) empirical risk minimizer} with respect to an (agnostic
    or not) loss function $\ell$ if for every $m\in\NN$ and every partially erased $[m]$-sample
    $(x,y)\in\cE_m(\Omega)\times(\Lambda\cup\{\unk\})^{([m])_k}$, we have
    \begin{equation}\label{eq:ERM}
      L_{x,y,\ell}^\alpha\bigl(\cA(x,y)\bigr)
      =
      \inf_{H\in\cH} L_{x,y,\ell}^\alpha(H)
    \end{equation}
    for every order choice $\alpha$ for $[m]$.
  \item We say that $\cH$ is \emph{sample completion $k$-PAC learnable} with respect to a $k$-ary loss function
    $\ell\colon\cE_k(\Omega)\times\Lambda^{S_k}\times\Lambda^{S_k}\to\RR_{\geq 0}$ if there exist a completion algorithm $\cA$
    and a function $m^{\SC}_{\cH,\ell,\cA}\colon(0,1)^3\to\RR_{\geq 0}$ such that for every $\epsilon,\delta,\rho\in(0,1)$,
    every $\mu\in\Pr(\Omega)$ and every $F\in\cF_k(\Omega,\Lambda)$ that is realizable in $\cH$ with respect to $\ell$ and
    $\mu$, we have
    \begin{equation*}
      \PP_{\rn{x}\sim\mu^m,\rn{E}_\rho}\Biggl[
        L_{\rn{x}, F^*_m(\rn{x}),\ell}^\alpha\biggl(
        \cA\Bigl(\rn{x}, \rn{E}_\rho\bigl(F^*_m(\rn{x})\bigr)\Bigr)
        \biggr)
        \leq \epsilon
        \Biggr]
      \geq 1 - \delta
    \end{equation*}
    for every integer $m\geq m^{\SC}_{\cH,\ell,\cA}(\epsilon,\delta,\rho)$ and every order choice $\alpha$ for $[m]$.

    A completion algorithm $\cA$ satisfying the above is called a \emph{sample completion $k$-PAC learner} for $\cH$ with
    respect to $\ell$.

    We define the notions of \emph{symmetric sample completion $k$-PAC learnability}, $m^{\sSC}_{\cH,\ell,\cA}$ and of a
    \emph{symmetric sample completion $k$-PAC learner} analogously to the non-symmetric case, but replacing the
    $(1-\rho)$-erasure $\rn{E}_\rho$ with the symmetric $(1-\rho)$-erasure $\rn{E}^{\sym}_\rho$.
  \item We say that $\cH$ is \emph{adversarial sample completion $k$-PAC learnable} with respect to $\ell$ if there exist a
    completion algorithm $\cA$ and a function $m^{\advSC}_{\cH,\ell,\cA}\colon(0,1)^3\to\RR_{\geq 0}$ such that for every
    $\epsilon,\delta,\rho\in(0,1)$ and every $[m]$-sample $(x,y)\in\cE_m(\Omega)\times\Lambda^{([m])_k}$, we have
    \begin{equation*}
      \PP_{\rn{E}_\rho}\biggl[
        L_{x,y,\ell}^\alpha\Bigl(
        \cA\bigl(x,\rn{E}_\rho(y)\bigr)
        \Bigr)
        \leq \inf_{H\in\cH} L_{x,y,\ell}^\alpha(H) + \epsilon
        \biggr]
      \geq 1 - \delta
    \end{equation*}
    for every order choice $\alpha$ for $[m]$.

    A completion algorithm $\cA$ satisfying the above is called an \emph{adversarial sample completion $k$-PAC learner} for
    $\cH$ with respect to $\ell$.

    We define the notions of \emph{adversarial symmetric sample completion $k$-PAC learnability}, $m^{\advsSC}_{\cH,\ell,\cA}$
    and of a \emph{adversarial symmetric sample completion $k$-PAC learner} analogously to the non-symmetric case, but replacing
    the $(1-\rho)$-erasure $\rn{E}_\rho$ with the symmetric $(1-\rho)$-erasure $\rn{E}^{\sym}_\rho$.
  \item Let $(x,y)\in\cE_m(\Omega)\times(\Lambda\cup\{\unk\})^{([m])_k}$ be a partially erased $[m]$-sample and let
    $y'\in\Lambda^{([m])_k}$ extend $y$. For $\epsilon > 0$, we say that $(x,y)$ is \emph{$\epsilon$-representative} with respect
    to $\cH$, $y'$ and $\ell$ if
    \begin{equation*}
      \bigl\lvert L_{x,y,\ell}^\alpha(H) - L_{x,y',\ell}^\alpha(H)\bigr\rvert \leq \epsilon
    \end{equation*}
    for every $H\in\cH$ and every order choice $\alpha$ for $[m]$.
  \item\label{def:SC:SUC} We say that $\cH$ has the \emph{sample uniform convergence property} with respect to $\ell$ if there
    exists a function $m^{\SUC}_{\cH,\ell}\colon(0,1)^3\to\RR_{\geq 0}$ such that for every $\epsilon,\delta,\rho\in(0,1)^3$,
    every integer $m\geq m^{\SUC}_{\cH,\ell}(\epsilon,\delta,\rho)$ and every $[m]$-sample
    $(x,y)\in\cE_m(\Omega)\times\Lambda^{([m])_k}$, we have
    \begin{equation*}
      \PP_{\rn{E}^{\sym}_\rho(y)}\Bigl[
        \bigl(x, \rn{E}^{\sym}_\rho(y)\bigr)
        \text{ is $\epsilon$-representative w.r.t.\ $\cH$, $y$ and $\ell$}
        \Bigr]
      \geq 1 - \delta.
    \end{equation*}
  \item For $\epsilon > 0$, $m\in\NN$, $x\in\cE_m(\Omega)$ and an order choice $\alpha$ for $[m]$, we say that a (finite)
    sequence $(H_1,\ldots,H_t)$ of $k$-ary hypotheses is \emph{$\epsilon$-separated on $x$} with respect to a $k$-ary loss
    function $\ell\colon\cE_k(\Omega)\times\Lambda^{S_k}\times\Lambda^{S_k}\to\RR_{\geq 0}$ and $\alpha$ if
    \begin{equation*}
      L_{x,(H_i)^*_m(x),\ell}^\alpha(H_j) > \epsilon
    \end{equation*}
    for every $i,j\in[m]$ with $i < j$.
  \item For a function $h\colon\NN\to\NN$, we say that $\cH$ has the \emph{($k$-ary) $h$-sample Haussler packing
  property} with respect to a $k$-ary loss function $\ell\colon\cE_k(\Omega)\times\Lambda^{S_k}\times\Lambda^{S_k}\to\RR_{\geq
    0}$ if there exists a function $m^{\hSHP}_{\cH,\ell}\colon(0,1)^3\to\RR_{\geq 0}$ such that for every
    $\epsilon,\delta,\rho\in(0,1)$, every integer $m\geq m^{\hSHP}_{\cH,\ell}(\epsilon,\delta,\rho)$ and every order choice
    $\alpha$ for $[m]$, if $(H_1,\ldots,H_t)\in\cH^t$ with $t\geq 2^{\rho\cdot h(m)}$, then $(H_1,\ldots,H_t)$ is \emph{not}
    $\epsilon$-separated on $x$ w.r.t.\ $\ell$ and $\alpha$.
  \item For a function $h\colon\NN\to\NN$, we say that $\cH$ has the \emph{($k$-ary) $h$-probabilistic Haussler packing
  property} with respect to a $k$-ary loss function $\ell\colon\cE_k(\Omega)\times\Lambda^{S_k}\times\Lambda^{S_k}\to\RR_{\geq
    0}$ if there exists a function $m^{\hPHP}_{\cH,\ell}\colon(0,1)^3\to\RR_{\geq 0}$ such that for every
    $\epsilon,\delta,\rho\in(0,1)$, every integer $m\geq m^{\hPHP}_{\cH,\ell}(\epsilon,\delta,\rho)$ and every
    $(H_1,\ldots,H_t)\in\cH^t$ with $t\geq 2^{\rho\cdot h(m)}$, we have
    \begin{equation*}
      \PP_{\rn{x}\sim\mu^m}\bigl[
        (H_1,\ldots,H_t)\text{ is $\epsilon$-separated on $\rn{x}$ w.r.t.\ $\ell$ and $\alpha$}
        \bigr]
      \leq \delta
    \end{equation*}
    for every order choice $\alpha$ for $[m]$.
  \item\label{def:SC:VCNkk} For $m\in\NN$, we let
    \begin{equation}\label{eq:Tkm}
      T_{k,m} \df \left\{U\in\binom{[k\cdot m]}{k} \;\middle\vert\; \lvert U\cap[(i-1)m + 1, im]\rvert=1\right\}
    \end{equation}
    be the set of all $k$-subsets of $[k\cdot m]$ that are transversal to the equipartition of $[k\cdot m]$ into $k$ intervals.

    If we are further given $x\in\cE_{k\cdot m}(\Omega)$ and $H\in\cH$, we define the function $H_x\colon T_{k,m}\to\Lambda^{S_k}$ by
    \begin{equation}\label{eq:Hx}
      H_x(U)_\tau \df H^*_{k\cdot m}(x)_{\iota_{U,k\cdot m}\comp\tau}
      \qquad (U\in T_{k,m}, \tau\in S_k),
    \end{equation}
    where $\iota_{U,k\cdot m}\colon[k]\to[k\cdot m]$ is the unique increasing function with $\im(\iota_{U,k\cdot m})=U$.

    We then write $\VCN_{k,k}(\cH)\geq m$ if there exists $x\in\cE_{k\cdot m}(\Omega)$ such that
    \begin{equation}\label{eq:cHx}
      \cH_x \df \{H_x \mid H\in\cH\} \subseteq (\Lambda^{S_k})^{T_{k,m}}
    \end{equation}
    Natarajan-shatters $T_{k,m}$.

    The \emph{Vapnik--Chervonenkis--Natarajan $(k,k)$-dimension}\footnote{We will prove in Proposition~\ref{prop:VCNkk} that the
    non-partite $\VCN_{k,k}$-dimension can be computed is in terms of the partization operation of Definition~\ref{def:kpart:cH}
    and the partite version of the $\VCN_{k,k}$-dimension as $\VCN_{k,k}(\cH)=\VCN_{k,k}(\cH^{\kpart})$.} of $\cH$, denoted
    $\VCN_{k,k}(\cH)$, is the largest $m\in\NN$ such that $\VCN_{k,k}(\cH)\geq m$ (and if this holds for every $m\in\NN$, then
    we write $\VCN_{k,k}(\cH)=\infty$).
  \end{enumdef}
\end{definition}

\begin{remark}\label{rmk:symm->nonsymm}
  It is clear that the symmetric version of (adversarial, resp.) sample completion $k$-PAC learnability implies its
  non-symmetric counterpart with a simple adjustment of parameters. Namely, to produce a sample completion learner $\cA'$ using
  a symmetric sample completion learner $\cA$, we can simply start by erasing all entries indexed by $\beta\in([m])_k$ such that
  there exists an entry indexed by some $\beta'\in([m])_k$ with $\im(\beta)=\im(\beta')$ that was erased. If our sample was
  indeed of the form $\rn{E}_\rho(y)$, then the result of this operation has the same distribution as $\rn{E}^{\sym}_{\rho'}(y)$
  for $\rho'\df\rho^{k!}$, so we get $m^{\SC}_{\cH,\ell,\cA'}(\epsilon,\delta,\rho)\df
  m^{\sSC}_{\cH,\ell,\cA}(\epsilon,\delta,\rho^{k!})$ (and similarly for the adversarial variant). A consequence of the main
  result of this paper is that the converse implication also holds.

  Furthermore, regarding the definition of sample uniform convergence in the non-partite (Definition~\ref{def:SC:SUC}),
  pedantically, it would be more accurate to call this the symmetric notion. However, note that it does not make sense to
  compute a $k$-ary (agnostic or not) loss function if we do not know all $S_k$-labels (i.e., if we have an element of
  $(\Lambda\cup\{\unk\})^{S_k}\setminus\Lambda^{S_k}$); this is reflected in the definition of $\cU_y$ in~\eqref{eq:cUy}. Thus
  the non-symmetric version of sample uniform convergence is trivially equivalent to its symmetric counterparts (except for the
  same change in the parameter $\rho$ to $\rho^{k!}$) and as such, for sample uniform convergence, we will simply use the
  symmetric version and omit ``symmetric'' from the terminology.
\end{remark}

\begin{remark}\label{rmk:measurable}
  Let us formalize why the measurability conditions that we impose make all probabilities and expectations make sense. We will
  also argue that when we only use the $0/1$-loss function and its agnostic counterpart, essentially all measurability
  conditions immediately hold. We will discuss only the partite case, but the non-partite case is completely analogous.

  First, when compute total losses
  \begin{equation*}
    L_{\mu,F,\ell}(H) \df \EE_{\rn{x}\sim\mu^1}\Bigl[\ell\bigl(\rn{x}, H(\rn{x}), F(\rn{x})\bigr)\Bigr],
  \end{equation*}
  the expectation above makes sense since the evaluation map $\ev\colon\cH\times\cE_1(\Omega)\ni(H,x)\mapsto H(x)\in\Lambda$ is
  measurable and $\ell$ is measurable.

  Similarly, for a fixed loss function $\ell$, $m\in\NN$ and $y\in\Lambda^{[m]^k}$, the function
  $\cL_{y,\ell}\colon\cH\times\cE_m(\Omega)\to\RR_{\geq 0}$ that maps $(H,x)\in\cH\times\cE_m(\Omega)$ to the empirical loss
  \begin{equation}\label{eq:measurable:cL}
    \cL_{y,\ell}(H,x)
    \df
    L_{x,y,\ell}(H)
    \df
    \frac{1}{m^k}\sum_{\alpha\in[m]^k}
    \ell\bigl(\alpha^*(x), H^*_m(x)_\alpha, y_\alpha\bigr)
  \end{equation}
  is also measurable due to $\ev$ and $\ell$ being measurable. A similar argument holds for the corresponding function
  $\cL_{y,\ell}$ defined from an agnostic loss function.

  We also need to reason about the erasure operation $\rn{E}_\rho$. For this, given $m\in\NN$ and $y\in\Lambda^{[m]^k}$, let
  $\Upsilon_m\df\{0,1\}^{[m]^k}$ be equipped with discrete $\sigma$-algebra, let $\nu_m\in\Pr(\Upsilon_m)$ be the distribution
  in which each entry is $1$ independently with probability $\rho$ and let $E_y\colon\Upsilon\to(\Lambda\cup\{\unk\})^{[m]^k}$
  be given by
  \begin{equation*}
    E_y(w)_\beta \df
    \begin{dcases*}
      y_\beta, & if $w_\beta=1$,\\
      \unk, & if $w_\beta=0$.
    \end{dcases*}
  \end{equation*}
  This successfully encodes the $\rho$-erasure operation as if $\rn{w}\sim\nu_m$, then $\rn{E}_\rho(y)\sim E_y(\rn{w})$.

  This means that the probability in the definition of adversarial sample completion learning is encoded as:
  \begin{equation*}
    \PP_{\rn{w}\sim\nu_m}\biggl[
      L_{x,y,\ell}\Bigl(
      \cA\bigl(x,E_y(\rn{w})\bigr)
      \Bigr)
      \leq \inf_{H\in\cH} L_{x,y,\ell}(H) + \epsilon
      \biggr].
  \end{equation*}
  Since $\Upsilon_m$ is equipped with discrete $\sigma$-algebra, the probability above makes sense.

  For the non-adversarial version, the probability is encoded as:
  \begin{equation*}
    \PP_{\rn{x}\sim\mu^m, \rn{w}\sim\nu_m}\biggl[
      L_{\rn{x},F^*_m(\rn{x}),\ell}\Bigl(
      \cA\bigl(\rn{x},E_{F^*_m(\rn{x})}(\rn{w})\bigr)
      \Bigr)
      \leq \epsilon
      \biggr].
  \end{equation*}
  Using the fact that the map $\cL_{y,\ell}$ of~\eqref{eq:measurable:cL} and the algorithm $\cA$ are measurable (and that
  $\Upsilon_m$ is equipped with discrete $\sigma$-algebra), the probability above is also well-defined.

  We now consider sample uniform convergence, which involves the following probability:
  \begin{equation*}
    \PP_{\rn{w}\sim\nu_m}[
      \bigl(x, E_y(\rn{w}))\bigr)
      \text{ is $\epsilon$-representative w.r.t.\ $\cH$, $y$ and $\ell$}
      \Bigr].
  \end{equation*}
  Again, this is well-defined since $\Upsilon_m$ is equipped with discrete $\sigma$-algebra\footnote{This is in sharp contrast
  with the definition of high-arity uniform convergence of~\cite{CM24+}: in that setting, the fact that
  $\epsilon$-representativeness involves quantifying over all $H\in\cH$ is what leads to the requirement mentioned in
  Footnote~\ref{ftn:measurability}.}.

  The fact that $h$-sample Haussler property makes sense does not require any measurability.

  For the $h$-probabilistic Haussler property to make sense, we need to compute the probability
  \begin{equation*}
    \PP_{\rn{x}\sim\mu^m}\bigl[
      (H_1,\ldots,H_t)\text{ is $\epsilon$-separated on $\rn{x}$ w.r.t.\ $\ell$}
      \bigr].
  \end{equation*}
  Since $(H_1,\ldots,H_t)$ is fixed in the above, the fact that the set of $x\in\cE_m(\Omega)$ in which $(H_1,\ldots,H_t)$ is
  $\epsilon$-separated is measurable follows from $\ev$ and $\ell$ being measurable.

  \medskip

  Let us now mention which of these measurability assumptions can be relaxed in sample completion. First, note that at no point
  we used that $\Omega$ is a tuple of Borel spaces. Indeed, sample completion learning makes sense in the setting of tuples of
  general measurable spaces\footnote{In the high-arity setting of~\cite{CM24+,CM25+}, Borelness was required to invoke theorems
  from exchangeability theory to cover agnostic learning; these are not required here.}. Furthermore, if we consider only the
  $0/1$-loss function $\ell_{0/1}$ and its agnostic counterpart and equip $\Lambda$ with the discrete $\sigma$-algebra, then
  $\ell_{0/1}$ is immediately measurable. We can then equip our hypotheses classes $\cH$ with the discrete $\sigma$-algebra as
  well and the evaluation map $\ev$ immediately becomes measurable (we do not necessarily satisfy the universal measurability of
  projections of Footnote~\ref{ftn:measurability}, but sample completion learning does not require it).
\end{remark}

\clearpage

\section{Statements of the main theorems}
\label{sec:main}

In this section, we formally state our main theorems. Figure~\ref{fig:roadmap} contains a pictorial image of the structure of
the proof and the location of the proofs of the specific implications.

\begin{restatable}[Fundamental theorem of sample PAC learning, partite version]{theorem}{thmSCpart}
  \label{thm:SCpart}
  Let $k\in\NN_+$, let $\Omega=(\Omega_i)_{i=1}^k$ be a $k$-tuple of non-empty Borel spaces, let $\Lambda$ be a non-empty finite
  Borel space, let $\cH\subseteq\cF_k(\Omega,\Lambda)$ be a $k$-partite hypothesis class, let
  $\ell\colon\cE_1(\Omega)\times\Lambda\times\Lambda\to\RR_{\geq 0}$ be a $k$-partite loss function that is separated and
  bounded. Suppose completion (almost) empirical risk minimizers exist (see Remark~\ref{rmk:ERM}). Let further
  $\ell^{\ag}\colon\cH\times\cE_1(\Omega)\times\Lambda\to\RR_{\geq 0}$ be the $k$-partite agnostic loss function given by
  \begin{equation*}
    \ell^{\ag}(H,x,y) \df \ell\bigl(x, H(x), y\bigr) \qquad \bigl(H\in\cH, x\in\cE_1(\Omega), y\in\Lambda\bigr).
  \end{equation*}
  Then the following are equivalent:
  \begin{enumerate}[label={\arabic*.}, ref={(\arabic*)}]
  \item\label{thm:SCpart:VCNkk} $\VCN_{k,k}(\cH) < \infty$.
  \item\label{thm:SCpart:SUC} $\cH$ has the sample uniform convergence with respect to $\ell^{\ag}$.
  \item\label{thm:SCpart:advSC} $\cH$ is adversarial sample completion $k$-PAC learnable with respect to $\ell^{\ag}$.
  \item\label{thm:SCpart:SC} $\cH$ is sample completion $k$-PAC learnable with respect to $\ell$.
  \item\label{thm:SCpart:SHP} $\cH$ has the $m^k$-sample Haussler packing property with respect to $\ell$.
  \item\label{thm:SCpart:SHPbootstrap} $\VCN_{k,k}(\cH) = d < \infty$ and $\cH$ has the $h$-sample Haussler packing
    property with respect to $\ell$ for every $h(m) = \omega(m^{k - 1/(d+1)^{k-1}}\cdot\ln m)$.
  \item\label{thm:SCpart:PHP} $\cH$ has the $m^k$-probabilistic Haussler packing property with respect to $\ell$.
  \end{enumerate}
\end{restatable}

\begin{restatable}[Fundamental theorem of sample PAC learning, non-partite version]{theorem}{thmSC}
  \label{thm:SC}
  Let $\Omega$ and $\Lambda$ be non-empty Borel spaces with $\Lambda$ finite, let $k\in\NN_+$, let
  $\cH\subseteq\cF_k(\Omega,\Lambda)$ be a $k$-ary hypothesis class, let
  $\ell\colon\cE_k(\Omega)\times\Lambda^{S_k}\times\Lambda^{S_k}\to\RR_{\geq 0}$ be a $k$-ary loss function that is symmetric,
  separated and bounded. Suppose completion (almost) empirical risk minimizers exist (see Remark~\ref{rmk:ERM}). Let further
  $\ell^{\ag}\colon\cH\times\cE_k(\Omega)\times\Lambda^{S_k}\to\RR_{\geq 0}$ be the $k$-ary agnostic loss function given by
  \begin{equation*}
    \ell^{\ag}(H,x,y) \df \ell\bigl(x, H^*_k(x), y\bigr) \qquad \bigl(H\in\cH, x\in\cE_k(\Omega), y\in\Lambda^{S_k}\bigr).
  \end{equation*}
  Then the following are equivalent:
  \begin{enumerate}[label={\arabic*.}, ref={(\arabic*)}]
  \item\label{thm:SC:VCNkk} $\VCN_{k,k}(\cH) < \infty$.
  \item\label{thm:SC:VCNkkpart} $\VCN_{k,k}(\cH^{\kpart}) < \infty$.
  \item\label{thm:SC:SUC} $\cH$ has the sample uniform convergence with respect to $\ell^{\ag}$.
  \item\label{thm:SC:SUCpart} $\cH^{\kpart}$ has the sample uniform convergence with respect to $(\ell^{\ag})^{\kpart}$.
  \item\label{thm:SC:advSCsymm} $\cH$ is adversarial symmetric sample completion $k$-PAC learnable with respect to
    $\ell^{\ag}$.
  \item\label{thm:SC:advSC} $\cH$ is adversarial sample completion $k$-PAC learnable with respect to $\ell^{\ag}$.
  \item\label{thm:SC:advSCpart} $\cH^{\kpart}$ is adversarial sample completion $k$-PAC learnable with respect to
    $(\ell^{\ag})^{\kpart}$.
  \item\label{thm:SC:SCsymm} $\cH$ is symmetric sample completion $k$-PAC learnable with respect to $\ell$.
  \item\label{thm:SC:SC} $\cH$ is sample completion $k$-PAC learnable with respect to $\ell$.
  \item\label{thm:SC:SCpart} $\cH^{\kpart}$ is sample completion $k$-PAC learnable with respect to $\ell^{\kpart}$.
  \item\label{thm:SC:SHP} $\cH$ has the $m^k$-sample Haussler packing property with respect to $\ell$.
  \item\label{thm:SC:SHPpart} $\cH^{\kpart}$ has the $m^k$-sample Haussler packing property with respect to $\ell^{\kpart}$.
  \item\label{thm:SC:SHPbootstrap} $\VCN_{k,k}(\cH) = d < \infty$ and $\cH$ has the $h$-sample Haussler packing property with
    respect to $\ell$ for every $h(m) = \omega(m^{k - 1/(d+1)^{k-1}}\cdot\ln m)$.
  \item\label{thm:SC:SHPbootstrappart} $\VCN_{k,k}(\cH^{\kpart}) = d < \infty$ and $\cH$ has the $h$-sample Haussler
    packing property with respect to $\ell^{\kpart}$ for every $h(m) = \omega(m^{k - 1/(d+1)^{k-1}}\cdot\ln m)$.
  \item\label{thm:SC:PHP} $\cH$ has the $m^k$-probabilistic Haussler packing property with respect to $\ell$.
  \item\label{thm:SC:PHPpart} $\cH^{\kpart}$ has the $m^k$-probabilistic Haussler packing property with respect to $\ell^{\kpart}$.
  \end{enumerate}
\end{restatable}

We also state quotable versions of the theorems above for the $0/1$-loss function (and its agnostic counterpart), in which
almost all measurability conditions can be dropped (see Remark~\ref{rmk:measurable}):

\begin{theorem}[Fundamental theorem of sample PAC learning, partite version, $0/1$-loss]
  Let $k\in\NN_+$, let $\Omega=(\Omega_i)_{i=1}^k$ be a $k$-tuple of non-empty measurable spaces, let $\Lambda$ be a non-empty
  finite measurable space, equipped with discrete $\sigma$-algebra, let $\cH\subseteq\cF_k(\Omega,\Lambda)$ be a $k$-partite
  hypothesis class, equipped with discrete $\sigma$-algebra. Suppose completion (almost) empirical risk minimizers exist (see
  Remark~\ref{rmk:ERM}). Then the following are equivalent:
  \begin{enumerate}[label={\arabic*.}, ref={(\arabic*)}]
  \item $\VCN_{k,k}(\cH) < \infty$.
  \item $\cH$ has the sample uniform convergence with respect to the agnostic $0/1$-loss function.
  \item $\cH$ is adversarial sample completion $k$-PAC learnable with respect to the agnostic $0/1$-loss function.
  \item $\cH$ is sample completion $k$-PAC learnable with respect to the $0/1$-loss function.
  \item $\cH$ has the $m^k$-sample Haussler packing property with respect to the $0/1$-loss function.
  \item $\VCN_{k,k}(\cH) = d < \infty$ and $\cH$ has the $h$-sample Haussler packing property with respect to the $0/1$-loss
    function for every $h(m) = \omega(m^{k - 1/(d+1)^{k-1}}\cdot\ln m)$.
  \item $\cH$ has the $m^k$-probabilistic Haussler packing property with respect to the $0/1$-loss function.
  \end{enumerate}
\end{theorem}

\begin{theorem}[Fundamental theorem of sample PAC learning, non-partite version, $0/1$-loss]
  Let $\Omega$ and $\Lambda$ be measurable spaces with $\Lambda$ finite and equipped with discrete $\sigma$-algebra, let
  $k\in\NN_+$, let $\cH\subseteq\cF_k(\Omega,\Lambda)$ be a $k$-ary hypothesis class, equipped with discrete $\sigma$-algebra.
  Suppose completion (almost) empirical risk minimizers exist (see Remark~\ref{rmk:ERM}). Then the following are equivalent:
  \begin{enumerate}[label={\arabic*.}, ref={(\arabic*)}]
  \item $\VCN_{k,k}(\cH) < \infty$.
  \item $\VCN_{k,k}(\cH^{\kpart}) < \infty$.
  \item $\cH$ has the sample uniform convergence with respect to the agnostic $0/1$-loss function.
  \item $\cH^{\kpart}$ has the sample uniform convergence with respect to the agnostic $0/1$-loss function.
  \item $\cH$ is adversarial symmetric sample completion $k$-PAC learnable with respect to the agnostic $0/1$-loss function.
  \item $\cH$ is adversarial sample completion $k$-PAC learnable with respect to the agnostic $0/1$-loss function.
  \item $\cH^{\kpart}$ is adversarial sample completion $k$-PAC learnable with respect to the agnostic $0/1$-loss function.
  \item $\cH$ is symmetric sample completion $k$-PAC learnable with respect to the $0/1$-loss function.
  \item $\cH$ is sample completion $k$-PAC learnable with respect to the $0/1$-loss function.
  \item $\cH^{\kpart}$ is sample completion $k$-PAC learnable with respect to the $0/1$-loss function.
  \item $\cH$ has the $m^k$-sample Haussler packing property with respect to the $0/1$-loss function..
  \item $\cH^{\kpart}$ has the $m^k$-sample Haussler packing property with respect to the $0/1$-loss function.
  \item $\VCN_{k,k}(\cH) = d < \infty$ and $\cH$ has the $h$-sample Haussler packing property with respect to the $0/1$-loss
    function for every $h(m) = \omega(m^{k - 1/(d+1)^{k-1}}\cdot\ln m)$.
  \item $\VCN_{k,k}(\cH^{\kpart}) = d < \infty$ and $\cH$ has the $h$-sample Haussler packing property with respect to the
    $0/1$-loss function for every $h(m) = \omega(m^{k - 1/(d+1)^{k-1}}\cdot\ln m)$.
  \item $\cH$ has the $m^k$-probabilistic Haussler packing property with respect to the $0/1$-loss function.
  \item $\cH^{\kpart}$ has the $m^k$-probabilistic Haussler packing property with respect to the $0/1$-loss function.
  \end{enumerate}
\end{theorem}

\section{Partite versus non-partite $\VCN_{k,k}$-dimension}

In this section, we prove that the partization operation (see Definition~\ref{def:kpart}) preserves the
$\VCN_{k,k}$-dimension. For this, the following lemma from~\cite{CM24+} will be useful:

\begin{lemma}[Partization basics~\protect{\cite[Lemma~8.1]{CM24+}}, simplified]\label{lem:kpartbasics}
  Let $\Omega$ and $\Lambda$ be non-empty Borel spaces and $k\in\NN_+$. Then the following hold:
  \begin{enumerate}
  \item\label{lem:kpartbasics:phi} For $\mu\in\Pr(\Omega)$ and $m\in\NN$, the function
    $\phi_m\colon\cE_m(\Omega)\to\cE_{\floor{m/k}}(\Omega^{\kpart})$ given by
    \begin{equation}\label{eq:kpartbasics:phi}
      (\phi_m(x)_i)_v
      \df
      x_{(i-1)\floor{m/k} + v}
      \qquad \left(i\in[k], v\in\Floor{\frac{m}{k}}\right)
    \end{equation}
    is measure-preserving with respect to $\mu^m$ and $(\mu^{\kpart})^{\floor{m/k}}$. Furthermore, if $m$ is divisible by $k$,
    then $\phi_m$ is a measure-isomorphism.

    Moreover, we have $\phi_k^{-1} = \iota_{\kpart}$, where $\iota_{\kpart}$ is given by~\eqref{eq:iotakpart:simplified}.
  \item\label{lem:kpartbasics:Phi} For $m\in\NN$, $F\in\cF_k(\Omega,\Lambda)$ and
    $\Phi_m\colon\Lambda^{([m])_k}\to(\Lambda^{S_k})^{[\floor{m/k}]^k}$ given by
    \begin{equation}\label{eq:kpartbasics:Phi}
      (\Phi_m(y)_\alpha)_\tau
      \df
      y_{\beta_\alpha\comp\tau}
      \qquad \left(\alpha\in \left[\floor{m/k}\right]^k, \tau\in S_k\right),
    \end{equation}
    where $\beta_\alpha\in([m])_k$ is given by
    \begin{equation}\label{eq:kpartbasics:betaalpha}
      \beta_\alpha(i)
      \df
      (i-1)\Floor{\frac{m}{k}} + \alpha(i)
      \qquad \left(\alpha\in \left[\Floor{\frac{m}{k}}\right]^k, i\in[k]\right),
    \end{equation}
    the diagram
    \begin{equation*}
      \begin{tikzcd}[column sep={2.5cm}]
        \cE_m(\Omega)
        \arrow[r, "F^*_m"]
        \arrow[d, "\phi_m"']
        &
        \Lambda^{([m])_k}
        \arrow[d, "\Phi_m"]
        \\
        \cE_{\floor{m/k}}(\Omega^{\kpart})
        \arrow[r, "(F^{\kpart})^*_{\floor{m/k}}"]
        &
        (\Lambda^{S_k})^{[\floor{m/k}]^k}
      \end{tikzcd}
    \end{equation*}
    commutes, where $\phi_m$ is given by~\eqref{eq:kpartbasics:phi}.
  \end{enumerate}
\end{lemma}

\begin{proposition}[$\VCN_{k,k}$-dimension invariance under partization]\label{prop:VCNkk}
  Let $\Omega$ and $\Lambda$ be non-empty Borel spaces, let $k\in\NN_+$ and let $\cH\subseteq\cF_k(\Omega,\Lambda)$ be a $k$-ary
  hypothesis class. Then $\VCN_{k,k}(\cH)=\VCN_{k,k}(\cH^{\kpart})$.
\end{proposition}

\begin{proof}
  Let us first show that $\VCN_{k,k}(\cH)\leq\VCN_{k,k}(\cH^{\kpart})$. For this, we suppose $m\in\NN$ is such
  that $\VCN_{k,k}(\cH)\geq m$ and we will show that $\VCN_{k,k}(\cH^{\kpart})\geq m$.

  Since $\VCN_{k,k}(\cH)\geq m$, we know that there exists $x\in\cE_{k\cdot m}(\Omega)$ such that
  \begin{equation*}
    \cH_x \df \{H_x \mid H\in\cH\}\subseteq (\Lambda^{S_k})^{T_{k,m}}
  \end{equation*}
  Natarajan-shatters $T_{k,m}$, where for each $H\in\cH$, the function $H_x\colon T_{k,m}\to\Lambda^{S_k}$ is given by
  \begin{equation*}
    H_x(U)_\tau \df H^*_{k\cdot m}(x)_{\iota_{U,k\cdot m}\comp\tau}
    \qquad (U\in T_{k,m}, \tau\in S_k),
  \end{equation*}
  where $\iota_{U,k\cdot m}\colon[k]\to[k\cdot m]$ is the unique increasing function with $\im(\iota_{U,k\cdot m})=U$.

  This means that there exist functions $f_0,f_1\colon T_{k,m}\to\Lambda^{S_k}$ such that for every $U\in T_{k,m}$, we have
  $f_0(U)\neq f_1(U)$ and for every $C\subseteq T_{k,m}$, there exists $H_C\in\cH$ such that for every $U\in T_{k,m}$, we have
  $(H_C)_x(U) = f_{\One[U\in C]}(U)$.

  To show that $\VCN_{k,k}(\cH^{\kpart})\geq m$, it suffices to show that for the point $\phi_{k\cdot m}(x)\in\cE_m(\Omega^{\kpart})$,
  where $\phi_{k\cdot m}$ is given by~\eqref{eq:kpartbasics:phi}, the collection
  \begin{equation*}
    \cH^{\kpart}_{\phi_{k\cdot m}(x)} \df \{(H^{\kpart})^*_m(x) \mid H\in\cH\} \subseteq (\Lambda^{S_k})^{[m]^k}
  \end{equation*}
  Natarajan-shatters $[m]^k$.

  Note that there exists a one-to-one correspondence between $T_{k,m}$ and $[m]^k$ in which $U\in T_{k,m}$ corresponds to
  $\alpha_U\in[m]^k$ given by $\alpha_U(i)\df \iota_{U,k\cdot m}(i) - (i-1)m$ (in plain English, if we order the elements of $U\in
  T_{k,m}$ in increasing manner, we know that the first element is one of $1,\ldots,m$, the second is one of $m+1,\ldots,2m$,
  the third is one of $2m+1,\ldots,3m$, and so on; each of these is one of $m$ possibilities and $\alpha_U$ simply specifies
  each of the $m$ possibilities for each element of $U$). Given $\alpha\in [m]^k$, we denote by $U_\alpha$ the unique element of
  $T_{k,m}$ corresponding to it, i.e., the unique element such that $\alpha_{U_\alpha}=\alpha$; in formulas, it is given by
  \begin{align*}
    U_\alpha
    & \df
    \{\alpha(i) + (i-1)m \mid i\in[k]\}
    \\
    & =
    \{\alpha(1), \alpha(2) + m, \alpha(3) + 2m,\ldots,\alpha(k) +(k-1)m\}
    =
    \im(\beta_\alpha),
  \end{align*}
  where $\beta_\alpha$ is given by~\eqref{eq:kpartbasics:betaalpha}. Since clearly $\beta_\alpha$ is increasing, it follows that
  $\beta_\alpha = \iota_{U_\alpha,k\cdot m}$.

  Define the functions $g_0,g_1\colon [m]^k\to\Lambda^{S_k}$ by $g_i(\alpha)\df f_i(U_\alpha)$. It is clear that
  $g_0(\alpha)\neq g_1(\alpha)$ for every $\alpha\in[m]^k$.

  We claim that for every $D\subseteq [m]^k$ and every $\alpha\in[m]^k$, we have $(H_{C_D}^{\kpart})^*_m(\phi(x))_\alpha =
  g_{\One[\alpha\in D]}(\alpha)$, where
  \begin{equation*}
    C_D = \{U_\alpha \mid \alpha\in D\}.
  \end{equation*}
  Note that once we show this, then $\cH^{\kpart}_{\phi_{k\cdot m}(x)}$ Natarajan-shatters $[m]^k$ as desired.

  But indeed, note that for every $\tau\in S_k$, by Lemma~\ref{lem:kpartbasics}\ref{lem:kpartbasics:Phi}, we have
  \begin{align*}
    \Bigl((H_{C_D}^{\kpart})^*_m\bigl(\phi_{k\cdot m}(x)\bigr)_\alpha\Bigr)_\tau
    & =
    \Bigl(\Phi_{k\cdot m}\bigl((H_{C_D})^*_{k\cdot m}(x)\bigr)_\alpha\Bigr)_\tau
    =
    (H_{C_D})^*_{k\cdot m}(x)_{\beta_\alpha\comp\tau}
    \\
    & =
    (H_{C_D})^*_{k\cdot m}(x)_{\iota_{U_\alpha,k\cdot m}\comp\tau}
    =
    (H_{C_D})_x(U_\alpha)_\tau
    \\
    & =
    f_{\One[U_\alpha\in C_D]}(U_\alpha)_\tau
    =
    g_{\One[\alpha\in D]}(U)_\tau,
  \end{align*}
  as desired. Therefore $\VCN_{k,k}(\cH)\leq\VCN_{k,k}(\cH^{\kpart})$.

  \medskip

  The proof of the other inequality is obtained essentially by reading the other proof backwards. For completeness, we make it
  explicit here: we suppose $m\in\NN$ is such that $\VCN_{k,k}(\cH^{\kpart})\geq m$ and we will show that $\VCN_{k,k}(\cH)\geq
  m$.

  Since $\VCN_{k,k}(\cH^{\kpart})\geq m$, we know that there exists $x\in\cE_m(\Omega^{\kpart})$ such that
  \begin{equation*}
    \cH^{\kpart}_x \df \{(H^{\kpart})^*_m(x) \mid H\in\cH\} \subseteq (\Lambda^{S_k})^{[m]^k}
  \end{equation*}
  Natarajan-shatters $[m]^k$. In turn, this means that there exist functions $g_0,g_1\colon [m]^k\to\Lambda^{S_k}$ such that for
  every $\alpha\in[m]^k$, we have $g_0(\alpha)\neq g_1(\alpha)$ and for every $D\subseteq [m]^k$, there exists $H_D\in\cH$ such
  that for every $\alpha\in[m]^k$, we have $((H_D^{\kpart})^*_m(x))_\alpha = g_{\One[\alpha\in D]}(\alpha)$.

  Define the functions $f_0,f_1\colon T_{k,m}\to\Lambda^{S_k}$ by $f_i(U)\df g_i(\alpha_U)$. It is clear that $g_0(U)\neq
  g_1(U)$ for every $U\in T_{k,m}$.

  Since $k\cdot m$ is divisible by $m$, Lemma~\ref{lem:kpartbasics}\ref{lem:kpartbasics:phi}, we know that $\phi_{k\cdot m}$ is a
  bijection. Our goal is to show that $\cH_{\phi_{k\cdot m}^{-1}(x)}$ Natarajan-shatters $T_{k,m}$. For this, it suffices to show that for
  every $C\subseteq T_{k,m}$ and every $U\in T_{k,m}$, we have $(H_{D_C})_{\phi_{k\cdot m}^{-1}(x)}(U) = f_{\One[U\in C]}(U)$, where
  \begin{equation*}
    D_C \df \{\alpha_U \mid U\in C\}.
  \end{equation*}
  But indeed, by Lemma~\ref{lem:kpartbasics}\ref{lem:kpartbasics:Phi}, for every $\tau\in S_k$, we have
  \begin{align*}
    (H_{D_C})_{\phi_{k\cdot m}^{-1}(x)}(U)_\tau
    & =
    (H_{D_C})^*_{k\cdot m}\bigl(\phi_{k\cdot m}^{-1}(x)\bigr)_{\iota_{U,k\cdot m}\comp\tau}
    =
    (H_{D_C})^*_{k\cdot m}\bigl(\phi_{k\cdot m}^{-1}(x)\bigr)_{\beta_{\alpha_U}\comp\tau}
    \\
    & =
    \biggl(\Phi_{k\cdot m}\Bigl((H_{D_C})^*_{k\cdot m}\bigl(\phi_{k\cdot m}^{-1}(x)\bigr)\Bigr)_{\alpha_U}\biggr)_\tau
    =
    \bigl((H_{D_C}^{\kpart})^*_m(x)_{\alpha_U}\bigr)_\tau
    \\
    & =
    g_{\One[\alpha_U\in D_C]}(\alpha_U)_\tau
    =
    f_{\One[U\in C]}(U)_\tau,
  \end{align*}
  as desired. Therefore $\VCN_{k,k}(\cH)\geq\VCN_{k,k}(\cH^{\kpart})$.
\end{proof}

\section{$\VCN_{k,k}$-dimension controls growth function}
\label{sec:growth}

In this section, we show that finite $\VCN_{k,k}$-dimension is responsible for making the number of possible patterns that a
hypothesis class generates on a point $x\in\cE_m(\Omega)$ to be much lower than expected (Lemma~\ref{lem:VCNkk->kgrowth}). This
can be seen as a $k$-ary analogue of the Sauer--Shelah--Perles Lemma (Lemma~\ref{lem:SSP} below); in fact, the proof itself will
use the classical Sauer--Shelah--Perles Lemma. As we will see in Proposition~\ref{prop:VCNkk->SHP}, the bound on the growth
function is so strong that it will trivially imply the $h$-sample Haussler packing property for every $h(m) = \omega(m^{k -
  1/(\VCN_{k,k}(\cH)+1)^{k-1}}\cdot\ln m)$.

\begin{definition}[Growth function]\label{def:growth}
  For a family $\cF\subseteq Y^X$ of functions $X\to Y$, the \emph{growth function} of $\cF$ is defined as
  \begin{equation*}
    \gamma_\cF(m) \df \sup\{\lvert\cF_V\rvert \mid V\subseteq X \land\lvert V\rvert\leq m\},
  \end{equation*}
  where
  \begin{equation*}
    \cF_V \df \{F\rest_V \mid F\in\cF\}.
  \end{equation*}
  In plain English, $\gamma_\cF(m)$ is the maximum number of functions that one can obtain by restricting all functions in $\cF$
  to the same set $V$ of size at most $m$. When $X$ is infinite, one can clearly consider only sets of size exactly $m$.
\end{definition}

We now recall the Sauer--Shelah--Perles Lemma\footnote{Appropriate naming of this lemma is apparently complicated: it
has been discovered independently by Vapnik--Chervonenkis~\cite{VC71}, Sauer~\cite{Sau72}, Shelah~\cite{She72}, who also gives
credit to Perles. The version we use here is due to Natarajan~\cite{Nat89} as we will need $Y$ finite instead of binary.}. Since
the proof of this is short, we include it in Appendix~\ref{sec:lit}.
\begin{restatable}[Vapnik--Chervonenkis~\cite{VC71}, Sauer~\cite{Sau72}, Shelah~\cite{She72}, Perles~\cite{Per72}, Natarajan~\cite{Nat89}]%
  {lemma}{SSPlemma}
  \label{lem:SSP}
  If $\cF\subseteq Y^X$ has finite Natarajan-dimension and $Y$ is finite, then
  \begin{equation*}
    \gamma_\cF(m) \leq (m+1)^{\Nat(\cF)}\cdot\binom{\lvert Y\rvert}{2}^{\Nat(\cF)}.
  \end{equation*}
\end{restatable}

\begin{definition}[$k$-growth function]\label{def:kgrowth}
  Let $k\in\NN_+$, let $\Omega=(\Omega_i)_{i=1}^k$ be a $k$-tuple of non-empty Borel spaces (a single non-empty Borel space,
  respectively), let $\Lambda$ be a non-empty Borel space and let $\cH\subseteq\cF_k(\Omega,\Lambda)$ be a $k$-partite ($k$-ary,
  respectively) hypothesis class.

  The \emph{$k$-growth function}\footnote{This notion should not be confused either with the growth function of
  Definition~\ref{def:growth} nor with the growth function $\tau^k_\cH$ in~\cite[Definition~9.4]{CM24+}, which is
  controlled by the $\VCN_k$-dimension instead.} of $\cH$ is defined as
  \begin{equation*}
    \gamma^k_\cH(m) \df \sup_{x\in\cE_m(\Omega)} \lvert\{F^*_t(x) \mid F\in\cH\}\rvert,
  \end{equation*}
  that is, it is the maximum amount of different patterns in $\Lambda^{[m]^k}$ ($\Lambda^{([m])_k}$, respectively) that one can
  get as $F^*_m(x)$ when one chooses a fixed $x\in\cE_m(\Omega)$. (Note that since the definition of $\cE_m(\Omega)$ allows for
  repetition of coordinates, we do not need to consider $\cE_t(\Omega)$ for all $t\leq m$.) When $k=1$, this concept matches the
  growth function $\gamma_\cH$ of Definition~\ref{def:growth}.
\end{definition}

To prove the high-arity analogue of Lemma~\ref{lem:SSP}, we will leverage a classical result in combinatorics on extremal
numbers of (partite or not) $k$-hypergraphs avoiding a (non-induced) complete $k$-partite hypergraph $K_{t,\ldots,t}^{(k)}$. For
this, we set up some notation.

\begin{definition}[Extremal number]\label{def:extremal}
  Let $k,t\in\NN_+$ and $n\in\NN$. The \emph{(non-partite) extremal number} $\ex(n,K_{t,\ldots,t}^{(k)})$ is the maximum number
  of edges of a $k$-hypergraph $G$ with $\lvert G\rvert = n$ and without any non-induced copies of $K_{t,\ldots,t}^{(k)}$, i.e.,
  a $k$-hypergraph $G$ in which there \emph{does not} exist a sequence $(v^i_j \mid i\in[k], j\in[t])$ of distinct vertices in
  $G$ such that for every $f\in [t]^k$, we have $\{v^1_{f(1)},\ldots,v^k_{f(k)}\}\in E(G)$.
\end{definition}

\begin{definition}[Partite extremal number]\label{def:partitekhypergraph}
  Let $k,t\in\NN_+$.
  \begin{enumdef}
  \item A \emph{$k$-partite $k$-hypergraph} (with a given $k$-partition) is a tuple $G=(V_1,\ldots,V_k,E)$, where
    $V_1,\ldots,V_k$ are pairwise disjoint sets and $E\subseteq V_1\times\cdots\times V_k$. We write
    \begin{align*}
      V_i(G) & \df V_i, &
      E(G) & \df E, &
      v_i(G) & \df \lvert V_i(G)\rvert, &
      e(G) & \df \lvert E(G)\rvert.
    \end{align*}
    We also let $V(G)\df\bigcup_{i=1}^k V_i(G)$.
  \item The \emph{complete $k$-partite hypergraph of order $t$} is the $k$-partite $k$--hypergraph $K_{t,\ldots,t}^{(k)}$ with
    each vertex set of size $t$ and all possible edges. Formally, we have
    \begin{align*}
      V_i(K_{t,\ldots,t}^{(k)}) & \df \{i\}\times[t], &
      E(K_{t,\ldots,t}^{(k)}) & \df \prod_{i=1}^k (\{i\}\times[t]).
    \end{align*}
    (The $\{i\}$ is just to ensure that the vertex sets are pairwise disjoint as per required by the formal definition.)

    For $k=2$, we use the more common notation $K_{t,t}\df K_{t,\ldots,t}^{(2)}$ and for $k=1$, we use the notation
    $K_t^{(1)}\df K_{t,\ldots,t}^{(1)}$.
  \item A \emph{(non-induced, labeled, injective) copy} of a $k$-partite $k$-hypergraph $H$ in a $k$-partite $k$-hypergraph
    $G$ is an injective function $f\colon V(H)\to V(G)$ such that
    \begin{align*}
      \forall i\in[k], f\bigl(V_i(H)\bigr)\subseteq V_i(G) & &
      f\bigl(E(H)\bigr)\subseteq E(G),
    \end{align*}
    i.e., $f$ respects the $k$-partition and maps edges to edges.
  \item\label{def:partitekhypergraph:extremal} For $n\in\NN$, the \emph{partite extremal number}
    $\ex_{\kpart}(n,K_{t,\ldots,t}^{(k)})$ is the maximum number of edges of a $k$-partite $k$-hypergraph $G$ with $v_i(G)=n$
    for every $i\in[k]$ and without any copies of $K_{t,\ldots,t}^{(k)}$.
  \end{enumdef}
\end{definition}

The following two theorems are versions of classical results in extremal combinatorics that hold for every $n\in\NN$; for a
modern proof and asymptotic versions with better coefficients (Theorems~\ref{thm:Erdos} and~\ref{thm:Erdosasymp:partite}), see
Appendix~\ref{sec:lit}.

\begin{restatable}[%
    \Kovari--\Sos--\Turan~\protect{\cite{KST54}},
    \Erdos, partite version of~\protect{\cite[Theorem~1]{Erd64}}]%
  {theorem}{KSTErdospartite}
  \label{thm:KSTErdos:partite}
  For every $n\in\NN$ and every $k,t\in\NN_+$, we have
  \begin{equation}\label{eq:Erdos}
    \ex_{\kpart}(n,K_{t,\ldots,t}^{(k)})
    \leq
    \begin{dcases*}
      2\cdot k\cdot n^{k - 1/t^{k-1}}, & if $k\geq 2$,\\
      t-1, & if $k=1$.
    \end{dcases*}
  \end{equation}
\end{restatable}

\begin{restatable}[%
    \Kovari--\Sos--\Turan, non-partite version of~\protect{\cite{KST54}},
    \Erdos, essentially~\protect{\cite[Theorem~1]{Erd64}}]%
  {theorem}{KSTErdos}
  \label{thm:KSTErdos}
  For every $n\in\NN$ and every $k,t\in\NN_+$, we have
  \begin{equation}\label{eq:KSTErdos}
    \ex(n,K_{t,\ldots,t}^{(k)})
    \leq
    \begin{dcases*}
      \frac{2\cdot n^{k - 1/t^{k-1}}}{(k-1)!}, & if $k\geq 2$,\\
      t-1, & if $k=1$.
    \end{dcases*}
  \end{equation}
\end{restatable}

\begin{lemma}[$\VCN_{k,k}$-dimension controls full growth function]\label{lem:VCNkk->kgrowth}
  Let $k\in\NN_+$, let $\Omega=(\Omega_i)_{i=1}^k$ be a $k$-tuple of non-empty Borel space (a single non-empty Borel space,
  respectively), let $\Lambda$ be a non-empty Borel space and let $\cH\subseteq\cF_k(\Omega,\Lambda)$ be a $k$-partite ($k$-ary,
  respectively) hypothesis class with finite $\VCN_{k,k}$-dimension. Let also $m\in\NN$ and in the non-partite case, let
  $\alpha$ be an order choice for $[m]$.

  For $x\in\cE_m(\Omega)$, define
  \begin{align*}
    \cH_x & \df \{H^*_m(x) \mid H\in\cH\} \subseteq \Lambda^{[m]^k}
    \intertext{in the partite case and}
    \cH_x^\alpha
    & \df
    \Bigl\{b_\alpha\bigl(H^*_m(x)\bigr) \mid H\in\cH\Bigr\}
    \subseteq
    (\Lambda^{S_k})^{\binom{[m]}{k}}
  \end{align*}
  in the non-partite case. Then
  \begin{equation}\label{eq:VCNkk->kgrowth:Nat}
    \begin{aligned}
      \Nat(\cH_x)
      & \leq
      \ex_{\kpart}(m, K_{\VCN_{k,k}(\cH)+1,\ldots,\VCN_{k,k}(\cH)+1}^{(k)}),
      \\
      \Nat(\cH_x^\alpha)
      & \leq
      \ex(m, K_{\VCN_{k,k}(\cH)+1,\ldots,\VCN_{k,k}(\cH)+1}^{(k)}).
    \end{aligned}
  \end{equation}

  In particular, we have
  \begin{equation}\label{eq:VCNkk->kgrowth:kgrowth}
    \begin{aligned}
      \gamma^k_\cH(m)
      & \leq
      \begin{dcases*}
        \begin{multlined}[b]
          (m^k+1)^{\ex_{\kpart}(m,K_{\VCN_{k,k}(\cH)+1,\ldots,\VCN_{k,k}(\cH)+1}^{(k)})}
          \\
          \cdot\binom{\lvert\Lambda\rvert}{2}^{\ex_{\kpart}(m,K_{\VCN_{k,k}(\cH)+1,\ldots,\VCN_{k,k}(\cH)+1}^{(k)})},
        \end{multlined}
        & in the partite case,
        \\
        \begin{multlined}[b]
          \left(\binom{m}{k}+1\right)^{\ex(m,K_{\VCN_{k,k}(\cH)+1,\ldots,\VCN_{k,k}(\cH)+1}^{(k)})}
          \\
          \cdot\binom{\lvert\Lambda\rvert^{k!}}{2}^{\ex(m,K_{\VCN_{k,k}(\cH)+1,\ldots,\VCN_{k,k}(\cH)+1}^{(k)})},
        \end{multlined}
        & in the non-partite case,
      \end{dcases*}
      \\
      & \leq
      \begin{dcases*}
        \begin{multlined}[b]
          \exp\Biggl(
          2\cdot k
          \cdot m^{k-1/(\VCN_{k,k}(\cH)+1)^{k-1}}
          \\
          \cdot\left(\ln(m^k+1) + \ln\binom{\lvert\Lambda\rvert}{2}\right)
          \Biggr),
        \end{multlined}
        & in the partite case if $k\geq 2$,
        \\
        \begin{multlined}[b]
          \exp\Biggl(
          \frac{2\cdot m^{k-1/(\VCN_{k,k}(\cH)+1)^{k-1}}}{(k-1)!}
          \\
          \cdot\left(\ln\left(\binom{m}{k}+1\right) + \ln\binom{\lvert\Lambda\rvert^{k!}}{2}\right)
          \Biggr),
        \end{multlined}
        & in the non-partite case if $k\geq 2$,
        \\
        (m+1)^{\VCN_{k,k}(\cH)}\cdot\binom{\lvert\Lambda\rvert}{2}^{\VCN_{k,k}(\cH)},
        & if $k=1$.
      \end{dcases*}
    \end{aligned}
  \end{equation}
\end{lemma}

\begin{proof}
  First we claim that the first inequality of~\eqref{eq:VCNkk->kgrowth:kgrowth} follows from~\eqref{eq:VCNkk->kgrowth:Nat} and
  Lemma~\ref{lem:SSP}.

  Indeed, in the partite case, we have
  \begin{align*}
    \gamma^k_\cH(m)
    & =
    \sup_{x\in\cE_m(\Omega)}
    \lvert\cH_x\rvert
    \leq
    (m^k+1)^{\Nat(\cH_x)}\cdot\binom{\lvert\Lambda\rvert}{2}^{\Nat(\cH_x)}
    \\
    & \leq
    (m^k+1)^{\ex_{\kpart}(m,K_{\VCN_{k,k}(\cH)+1,\ldots,\VCN_{k,k}(\cH)+1}^{(k)})}
    \cdot\binom{\lvert\Lambda\rvert}{2}^{\ex_{\kpart}(m,K_{\VCN_{k,k}(\cH)+1,\ldots,\VCN_{k,k}(\cH)+1}^{(k)})},
  \end{align*}
  where the first inequality follows from Lemma~\ref{lem:SSP} (as $\cH_x$ is a family of functions of the form
  $[m]^k\to\Lambda$) and the second inequality follows from~\eqref{eq:VCNkk->kgrowth:Nat}.

  In the non-partite case, we have
  \begin{align*}
    \gamma^k_\cH(m)
    & =
    \sup_{x\in\cE_m(\Omega)}
    \lvert\{F^*_m(x) \mid F\in\cH\}\rvert
    =
    \sup_{x\in\cE_m(\Omega)}
    \lvert\cH_x^\alpha\rvert
    \leq
    \left(\binom{m}{k}+1\right)^{\Nat(\cH_x^\alpha)}
    \cdot\binom{\lvert\Lambda\rvert^{k!}}{2}^{\Nat(\cH_x^\alpha)}
    \\
    & \leq
    \left(\binom{m}{k}+1\right)^{\ex(m,K_{\VCN_{k,k}(\cH)+1,\ldots,\VCN_{k,k}(\cH)+1}^{(k)})}
    \cdot\binom{\lvert\Lambda\rvert^{k!}}{2}^{\ex(m,K_{\VCN_{k,k}(\cH)+1,\ldots,\VCN_{k,k}(\cH)+1}^{(k)})},
  \end{align*}
  where the second equality follows since the function $b_\alpha$ is a bijection from $\Lambda^{([m])_k}$ to
  $(\Lambda^{S_k})^{\binom{m}{k}}$ (see~\eqref{eq:balpha}), the first inequality follows from Lemma~\ref{lem:SSP} (as $\cH_x$ is
  a family of functions of the form $\binom{[m]}{k}\to\Lambda^{S_k}$) and the second inequality follows
  from~\eqref{eq:VCNkk->kgrowth:Nat}.

  \medskip
  
  The second inequality of~\eqref{eq:VCNkk->kgrowth:kgrowth} follows from Theorems~\ref{thm:KSTErdos:partite}
  and~\ref{thm:KSTErdos}.

  \medskip

  It remains to prove the inequalities in~\eqref{eq:VCNkk->kgrowth:Nat}. Both the partite and non-partite cases have analogous
  proof ideas of constructing a $k$-hypergraph $G$ whose edges correspond to the largest shattered set and proving that the
  definition of $\VCN_{k,k}$-dimension forces $G$ to not have copies of $K_{t,\ldots,t}^{(k)}$; in turn this bounds the number
  of edges of $G$ (hence the size of the largest shattered set) in terms of the extremal numbers of
  Definitions~\ref{def:extremal} and~\ref{def:partitekhypergraph:extremal}. The main difference between the cases is that the
  definition of $\VCN_{k,k}$-dimension is easier to handle in the partite case, but we have to resort to partite equivariance
  (see~\eqref{eq:F*Vequiv:part}), which is more complicated than its non-partite counterpart (see~\eqref{eq:F*Vequiv}).

  We start with the partite case.

  Let $t\df\VCN_{k,k}(\cH)+1 < \infty$, fix $m\in\NN$ and $x\in\cE_m(\Omega)$ and consider the family of functions
  (cf.~\eqref{eq:cHx:partite})
  \begin{equation*}
    \cH_x \df \{H^*_m(x) \mid H\in\cH\} \subseteq \Lambda^{[m]^k}.
  \end{equation*}

  Let $N\subseteq [m]^k$ be the largest set that is Natarajan-shattered by $\cH_x$ and form the $k$-partite $k$-hypergraph $G$
  with $m$ vertices in each part and edge set $N$; formally, let
  \begin{align*}
    V_i(G) & \df \{i\}\times [m] \qquad (i\in[k]), &
    E(G) & \df \bigl\{\bigl((i,g(i))\bigr)_{i=1}^k \;\bigm\vert\; g\in N\bigr\}.
  \end{align*}

  We claim that $G$ has no copies of $K_{t,\ldots,t}^{(k)}$. Suppose not, that is, suppose $h\colon [k]\times[t]\to [k]\times
  [m]$ is a copy of $K_{t,\ldots,t}^{(k)}$, i.e., we have
  \begin{align}\label{eq:VCNkk->kgrowth:h}
    \forall (i,j)\in[k]\times [t], h(i,j)_1 = i, & &
    \forall \beta\in[t]^k, \bigl(h(i,\beta_i)_2 \mid i\in[k]\bigr)\in N.
  \end{align}

  For each $i\in[k]$, let $\alpha_i\colon[t]\to[m]$ be the unique function such that $h(i,j) = (i,\alpha_i(j))$ for every
  $j\in[t]$, that is, we let $\alpha_i(j) \df h(i,j)_2$ for every $j\in[t]$. Note that the second condition
  in~\eqref{eq:VCNkk->kgrowth:h} translates to
  \begin{equation*}
    \forall \beta\in[t]^k, \bigl(\alpha_i(\beta_i) \mid i\in[k]\bigr)\in N.
  \end{equation*}

  Using the functions $\alpha^\#$ of Definition~\ref{def:part:alphasharp} and the fact that the diagram~\eqref{eq:F*Vequiv:part}
  commutes, we note that for the point $w\df\alpha^\#(x)\in\cE_t(\Omega)$ and for $H\in\cH$, we have
  \begin{equation*}
    H^*_t(w)
    =
    H^*_t\bigl(\alpha^\#(x)\bigr)
    =
    \alpha^\#\bigl(H^*_m(x)\bigr).
  \end{equation*}
  In particular, we have
  \begin{equation*}
    \cH_w
    =
    \{H^*_t(w) \mid H\in\cH\}
    =
    \Bigl\{\alpha^\#\bigl(H^*_m(x)\bigr) \mid H\in\cH\Bigr\}
    \subseteq
    \Lambda^{[t]^k}.
  \end{equation*}
  We will show that $\cH_w$ Natarajan-shatters $[t]^k$, contradicting the fact that $t=\VCN_{k,k}(\cH)+1$.

  Since $N$ is Natarajan-shattered by $\cH_x$, it is clear that the set of edges in the copy of $K_{t,\ldots,t}^{(k)}$ in $G$ is
  Natarajan-shattered by $\cH_x$, that is, the set
  \begin{equation*}
    N'\df \Bigl\{\bigl(\alpha_i(\beta_i)\mid i\in[k]\bigr)\mid \beta\in[t]^k\Bigr\}
  \end{equation*}
  is Natarajan-shattered by $\cH_x$, i.e., there exist functions $f_0,f_1\colon N'\to\Lambda$ such that for every $\theta\in
  N'$, we have $f_0(\theta)\neq f_1(\theta)$ and for each $U\subseteq N'$, there exists $H_U\in\cH$ such that
  $(H_U)^*_m(x)_\theta = f_{\One[\theta\in U]}(\theta)$ for every $\theta\in N'$.

  Consider now the product of the functions $\alpha_i$, that is, the function $\alpha\colon[t]^k\to[m]^k$ given by
  \begin{equation*}
    \alpha(\beta_1,\ldots,\beta_k) \df \bigl(\alpha_1(\beta_1),\ldots,\alpha_k(\beta_k)\bigr).
  \end{equation*}
  It is clear that $\alpha$ is a bijection between $[t]^k$ and $N'$.

  Define then the functions $g_0,g_1\colon [t]^k\to\Lambda$ by $g_i\df f_i\comp\alpha$. Since $\alpha$ is a bijection, it is
  clear that $g_0(\beta)\neq g_1(\beta)$ for every $\beta\in[t]^k$. Note now that for every $V\subseteq [t]^k$ and every
  $\beta\in[t]^k$, we have
  \begin{align*}
    (H_{\alpha(V)})^*_t(w)_\beta
    & =
    \alpha^\#\bigl((H_{\alpha(V)})^*_m(x)\bigr)_\beta
    =
    (H_\alpha(V))^*_m(x)_{\alpha_1(\beta_1),\ldots,\alpha_k(\beta_k)}
    =
    (H_{\alpha(V)})^*_m(x)_{\alpha(\beta)}
    \\
    & =
    f_{\One[\alpha(\beta)\in\alpha(V)]}\bigl(\alpha(\beta)\bigr)
    =
    g_{\One[\beta\in V]}(\beta),
  \end{align*}
  so $\cH_w$ Natarajan-shatters $[t]^k$, contradicting the fact that $t=\VCN_{k,k}(\cH)+1$.

  Thus $G$ has no copies of $K_{t,\ldots,t}^{(k)}$ where $t=\VCN_{k,k}(\cH)+1$, hence
  \begin{equation*}
    \lvert N\rvert = \lvert E(G)\rvert \leq \ex_{\kpart}(m,K_{\VCN_{k,k}(\cH)+1,\ldots,\VCN_{k,k}(\cH)+1}^{(k)}),
  \end{equation*}
  concluding the proof of~\eqref{eq:VCNkk->kgrowth:Nat} in the partite case.

  \medskip

  We now prove the non-partite case.

  Let $t\df\VCN_{k,k}(\cH)+1<\infty$, fix $m\in\NN$ and $x\in\cE_m(\Omega)$. Fix also an order choice $\alpha$ for $[m]$ and
  consider the family of functions
  \begin{equation*}
    \cH_x^\alpha
    \df
    \Bigl\{b_\alpha\bigl(H^*_m(x)\bigr) \mid H\in\cH\Bigr\}
    \subseteq
    (\Lambda^{S_k})^{\binom{[m]}{k}}.
  \end{equation*}

  Let $N\subseteq\binom{[m]}{k}$ be the largest set that is Natarajan-shattered by $\cH_x^\alpha$ and form the $k$-hypergraph
  $G$ over $[m]$ whose edge set is $N$, i.e., we let $V(G)\df[m]$ and $E(G)\df N$.

  We claim that $G$ has no copies of $K_{t,\ldots,t}^{(k)}$. Suppose not, that is, suppose there exists a sequence
  $(v^i_j)_{i\in[k],j\in[t]}$ of distinct vertices of $G$ such that for every $f\in[t]^k$, we have
  $\{v^1_{f(1)},\ldots,v^k_{f(k)}\}\in E(G)$.

  Define the injection $\beta\colon[kt]\to[m]$ by
  \begin{equation*}
    \beta(\theta) \df v^{\ceil{\theta/t}}_{\theta\bmod t},
  \end{equation*}
  so that for every $U\in T_{k,t}$ (see~\eqref{eq:Tkm}), we have $\beta(U)\in E(G)$.

  Let $w\df\beta^*(x)\in\cE_{k\cdot m}(\Omega)$ and recall from~\eqref{eq:cHx} the definition of $\cH_w\df\{H_w \mid H\in\cH\}$,
  where (from~\eqref{eq:Hx}) $H_w\colon T_{k,t}\to\Lambda^{S_k}$ is given by
  \begin{equation*}
    H_w(U)_\tau \df H^*_{kt}(x)_{\iota_{U,kt}\comp\tau}
    \qquad (U\in T_{k,t}, \tau\in S_k),
  \end{equation*}
  where $\iota_{U,kt}$ is the unique increasing function $[k]\to[kt]$ with $\im(\iota_{U,kt})=U$. We will show that $\cH_w$
  Natarajan-shatters $T_{k,t}$, contradicting the fact that $t=\VCN_{k,k}(\cH)+1$.

  Since $N$ is Natarajan-shattered by $\cH_x^\alpha$, it is clear that the set of edges $\beta(T_{k,t})$ in the copy of
  $K_{t,\ldots,t}^{(k)}$ in $G$ is Natarajan-shattered by $\cH_x^\alpha$, that is, there exist functions
  $f_0,f_1\colon\beta(T_{k,t})\to\Lambda^{S_k}$ such that for every $U\in T_{k,t}$, we have $f_0(\beta(U))\neq f_1(\beta(U))$
  and for each $V\subseteq T_{k,t}$, there exists $H_V\in\cH$ such that
  \begin{equation*}
    b_\alpha\bigl((H_V)^*_m(x)\bigr)_{\beta(U)} = f_{\One[U\in V]}\bigl(\beta(U)\bigr)
  \end{equation*}
  for every $U\in T_{k,t}$.

  Define the functions $g_0,g_1\colon T_{k,t}\to\Lambda^{S_k}$ by
  \begin{equation*}
    g_i(U)_\tau
    \df
    f_i\bigl(\beta(U)\bigr)_{\alpha_{\beta(U)}^{-1}\comp\beta\comp\iota_{U,kt}\comp\tau}
  \end{equation*}
  Note that the above is well-defined since $\im(\beta\comp\iota_{U,kt})=\beta(U)=\im(\alpha_{\beta(U)})$. Note also that the
  function $S_k\ni\tau\mapsto\alpha_{\beta(U)}^{-1}\comp\beta\comp\iota_{U,kt}\comp\tau\in S_k$ is a bijection (this is because
  $\alpha_{\beta(U)}^{-1}\comp\beta\comp\iota_{U,kt}$ is itself an element of $S_k$), which in particular implies that
  $g_0(U)\neq g_1(U)$ for every $U\in T_{k,t}$. Note now that for every $V\subseteq T_{k,t}$, every $U\in T_{k,t}$ and every
  $\tau\in S_k$, we have
  \begin{align*}
    (H_V)_w(U)_\tau
    & =
    (H_V)^*_{kt}\bigl(\beta^*(x)\bigr)_{\iota_{U,kt}\comp\tau}
    =
    \beta^*\bigl((H_V)^*_m(x)\bigr)_{\iota_{U,kt}\comp\tau}
    =
    (H_V)^*_m(x)_{\beta\comp\iota_{U,kt}\comp\tau}
    \\
    & =
    (H_V)^*_m(x)_{\alpha_{\beta(U)}\comp\alpha_{\beta(U)}^{-1}\comp\beta\comp\iota_{U,kt}\comp\tau}
    =
    \Bigl(b_\alpha\bigl((H_V)^*_m(x)\bigr)_{\beta(U)}\Bigr)_{\alpha_{\beta(U)}^{-1}\comp\beta\comp\iota_{U,kt}\comp\tau}
    \\
    & =
    f_{\One[U\in V]}\bigl(\beta(U)\bigr)_{\alpha_{\beta(U)}^{-1}\comp\beta\comp\iota_{U,kt}\comp\tau}
    =
    g_{\One[U\in V]}(U)_\tau
  \end{align*}
  where the second equality follows from equivariance of $(H_V)^*$ (see~\eqref{eq:F*Vequiv}). Thus, we conclude that $(H_V)_w(U)
  = g_{\One[U\in V]}(U)$, that is, $\cH_w$ Natarajan-shatters $T_{k,t}$, contradicting the fact that $t=\VCN_{k,k}(\cH)+1$.

  Thus $G$ has no copies of $K_{t,\ldots,t}^{(k)}$ where $t=\VCN_{k,k}(\cH)+1$, hence
  \begin{equation*}
    \lvert N\rvert = \lvert E(G)\rvert \leq \ex(m,K_{\VCN_{k,k}(\cH)+1,\ldots,\VCN_{k,k}(\cH)+1}^{(k)}),
  \end{equation*}
  concluding the proof of~\eqref{eq:VCNkk->kgrowth:Nat} in the non-partite case.
\end{proof}

\begin{proposition}\label{prop:VCNkk->SHP}
  Let $k\in\NN_+$, let $\Omega=(\Omega_i)_{i=1}^k$ be a $k$-tuple of non-empty Borel spaces (a single non-empty Borel space),
  let $\Lambda$ be a finite non-empty Borel space, let $\cH\subseteq\cF_k(\Omega,\Lambda)$ be a $k$-partite ($k$-ary,
  respectively) hypothesis class and let $\ell$ be a $k$-partite ($k$-ary, respectively) loss function. Suppose that $\ell$ is
  separated and $\VCN_{k,k}(\cH) < \infty$. Then $\cH$ has the $h$-sample Haussler packing property for every $h(m) =
  \omega(m^{k - 1/(\VCN_{k,k}(\cH)+1)^{k-1}}\cdot\ln m)$.
\end{proposition}

\begin{proof}
  First note that since $\ell$ is separated, if $(H_1,\ldots,H_t)\in\cH^t$ is $\epsilon$-separated on $x\in\cE_m(\Omega)$ with
  respect to $\ell$ and an order choice $\alpha$ for $[m]$ (in the non-partite case), then we must have
  \begin{equation}\label{eq:VCNkk->SHP:growth}
    \lvert\{(H_i)^*_m(x) \mid i\in[t]\}\rvert = t.
  \end{equation}

  On the other hand, both in the partite and non-partite case, Lemma~\ref{lem:VCNkk->kgrowth} says
  \begin{equation*}
    \gamma^k_\cH(m)
    \leq
    \exp\bigl(O(m^{k - 1/(\VCN_{k,k}(\cH)+1)^{k-1}}\cdot\ln m)\bigr).
  \end{equation*}

  Since $h(m) = \omega(m^{k - 1/(\VCN_{k,k}(\cH)+1)^{k-1}}\cdot\ln m)$, there exists $m_0\in\NN$ large enough such that for every
  integer $m\geq m_0$, we have $h(m) > \log_2(\gamma^k_\cH(m))/\rho$, so if $(H_1,\ldots,H_t)\in\cH^t$ is such that $t\geq
  2^{\rho\cdot h(m)}$, then $t > \gamma^k_\cH(m)$.

  Let now $x\in\cE_m(\Omega)$ and $\alpha$ be an order choice for $[m]$ (in the non-partite case). Since the set on the
  left-hand side of~\eqref{eq:VCNkk->SHP:growth} has size at most $\gamma^k_\cH(m)$, it follows
  that~\eqref{eq:VCNkk->SHP:growth} does not hold, hence $(H_1,\ldots,H_t)\in\cH^t$ is not $\epsilon$-separated on $x$ with
  respect to $\ell$ and $\alpha$ (in the non-partite case).
\end{proof}

\section{Finite $\VCN_{k,k}$-dimension implies sample uniform convergence}

In this section, we show that finite $\VCN_{k,k}$-dimension implies sample uniform convergence.

\begin{lemma}[Partially erased empirical loss versus $k$-partite/$k$-ary growth function]\label{lem:peel}
  Let $k\in\NN_+$, let $\Omega=(\Omega_i)_{i=1}^k$ be a $k$-tuple of non-empty Borel spaces (a single non-empty Borel space,
  respectively), let $\Lambda$ be a non-empty Borel space, let $\cH\subseteq\cF_k(\Omega,\Lambda)$ be a $k$-partite ($k$-ary,
  respectively) hypothesis class, let $\ell$ be a $k$-partite ($k$-ary, respectively) agnostic loss function that is bounded and
  local, let $m\in\NN$ and let $(x,y)$ be an $[m]$-sample.

  Let also
  \begin{equation}\label{eq:peel:Mk}
    M_k \df
    \begin{dcases*}
      m^k, & in the partite case,\\
      \binom{m}{k}, & in the non-partite case.
    \end{dcases*}
  \end{equation}

  Then in the partite case, for every $\epsilon,\rho\in(0,1)$, we have
  \begin{align*}
    \MoveEqLeft
    \PP_{\rn{E}_\rho}\biggl[
      \sup_{\alpha, H\in\cH} \lvert L_{x,y,\ell}(H) - L_{x,\rn{E}_\rho(y),\ell}(H)\rvert
      \leq
      \epsilon
      \biggr]
    \\
    & \geq
    1
    - 2\cdot\exp\left(-\frac{\epsilon^2\cdot M_k}{12\cdot\lVert\ell\rVert_\infty^2}\right)
    - 2\cdot\gamma^k_\cH(m)\cdot\exp\left(-\frac{\epsilon^2\cdot\rho^2\cdot M_k}{2\cdot\lVert\ell\rVert_\infty^2}\right).
  \end{align*}
  And in the non-partite case, the same holds for every order choice $\alpha$ for $[m]$ with both $L$ replaced by $L^\alpha$.
\end{lemma}

\begin{proof}
  We prove first the partite case. The result is trivial if $\lVert\ell\rVert_\infty=0$ (where the exponentials should be
  interpreted as $\exp(-\infty)=0$, so the probability bound is $1$), so we assume $\lVert\ell\rVert_\infty > 0$.

  Since $\ell$ is local, we can decompose it in terms of a non-agnostic loss function $\ell_r$ and a regularization term $r$ as
  in~\eqref{eq:localellr:part}, and we can further ensure that $\lVert\ell_r\rVert_\infty\leq\lVert\ell\rVert_\infty$ (see
  Remark~\ref{rmk:localbounded}).

  We are interested in showing that with high probability, the following quantity is small:
  \begin{align*}
    \MoveEqLeft
    \sup_{H\in\cH} \lvert L_{x,y,\ell}(H) - L_{x,\rn{E}_\rho(y),\ell}(H)\rvert
    \\
    & =
    \sup_{H\in\cH}
    \left\lvert
    \frac{1}{m^k}\sum_{\beta\in[m]^k}
    \ell\bigl(H, \beta^*(x), y_\beta\bigr)
    - \frac{1}{\lvert\cU_{\rn{E}_\rho(y)}\rvert}\sum_{\beta\in\cU_{\rn{E}_\rho(y)}}
    \ell\bigl(H,\beta^*(x),\rn{E}_\rho(y)_\beta\bigr)
    \right\rvert
    \\
    & =
    \sup_{H\in\cH}
    \left\lvert
    \frac{1}{m^k}\sum_{\beta\in[m]^k}
    \ell_r\Bigl(\beta^*(x), H\bigl(\beta^*(x)\bigr), y_\beta\Bigr)
    - \frac{1}{\lvert\cU_{\rn{E}_\rho(y)}\rvert}\sum_{\beta\in\cU_{\rn{E}_\rho(y)}}
    \ell_r\Bigl(\beta^*(x), H\bigl(\beta^*(x)\bigr),y_\beta\Bigr)
    \right\rvert,
  \end{align*}
  where the last equality follows since $\rn{E}_\rho(y)_\beta = y_\beta$ for every $\beta\in\cU_{\rn{E}_\rho(y)}$ and since the
  regularization terms cancel out.

  To do this, first note that
  \begin{equation*}
    \lvert\cU_{\rn{E}_\rho(y)}\rvert
    =
    \lvert\{\beta\in[m]^k \mid \rn{E}_\rho(y)\neq\unk\}\rvert
  \end{equation*}
  has binomial distribution $\Bi(m^k,\rho)$, so by the multiplicative version of Chernoff's bound, we have
  \begin{equation}\label{eq:Chernoff}
    \PP_{\rn{E}_\rho}\left[\bigl\lvert\lvert\cU_{\rn{E}_\rho(y)}\rvert - \rho\cdot m^k\bigr\rvert
      >
      \frac{\epsilon\cdot\rho\cdot m^k}{2\cdot\lVert\ell\rVert_\infty}
      \right]
    \leq
    2\cdot\exp\left(-\frac{\epsilon^2\cdot m^k}{12\cdot\lVert\ell\rVert_\infty^2}\right),
  \end{equation}
  i.e., with high probability $\cU_{\rn{E}_\rho(y)}$ has size close to its expected value $\rho\cdot m^k$.

  Thus, it will suffice to prove instead that with high probability, the following quantity is small:
  \begin{equation}\label{eq:compare}
    \begin{aligned}
      \MoveEqLeft
      \sup_{H\in\cH}
      \left\lvert
      \frac{1}{m^k}\sum_{\beta\in[m]^k}
      \ell_r\Bigl(\beta^*(x), H\bigl(\beta^*(x)\bigr), y_\beta\Bigr)
      - \frac{1}{\rho\cdot m^k}\sum_{\beta\in\cU_{\rn{E}_\rho(y)}}
      \ell_r\Bigl(\beta^*(x), H\bigl(\beta^*(x)\bigr), y_\beta\Bigr)
      \right\rvert
      \\
      & =
      \sup_{H\in\cH} \frac{1}{m^k}\cdot
      \left\lvert
      \sum_{\beta\in[m]^k}
      \left(1 - \frac{\One[\beta\in\cU_{\rn{E}_\rho(y)}]}{\rho}\right)\cdot
      \ell_r\Bigl(\beta^*(x), H\bigl(\beta^*(x)\bigr), y_\beta\Bigr)
      \right\rvert.
    \end{aligned}
  \end{equation}

  Let us define a collection of i.i.d.\ random variables $\rn{Z}_\beta$ ($\beta\in[m]^k$), each of which takes value $1$ with
  probability $1-\rho$ and value $1-1/\rho$ with probability $\rho$. Since in $\rn{E}_\rho(y)$, each entry of $y$ is
  independently erased with probability $1-\rho$, the last expression in~\eqref{eq:compare} has the same distribution as
  \begin{equation*}
    \sup_{H\in\cH} \frac{1}{m^k}\cdot
    \left\lvert
    \sum_{\beta\in[m]^k}
    \rn{Z}_\beta\cdot\ell_r\Bigl(\beta^*(x), H\bigl(\beta^*(x)\bigr), y_\beta\Bigr)
    \right\rvert.
  \end{equation*}

  Now note that the expression inside the supremum above only depends on $H$ through the values $H^*_m(x)$, which means that it
  is equal to
  \begin{equation}\label{eq:compare2}
    \sup_{G\in\cH(x)} \frac{1}{m^k}\cdot
    \left\lvert
    \sum_{\beta\in[m]^k}
    \rn{Z}_\beta\cdot\ell_r\bigl(\beta^*(x), G(\beta), y_\beta\bigr)
    \right\rvert,
  \end{equation}
  where
  \begin{equation*}
    \cH(x) \df \{H^*_m(x) \mid H\in\cH\},
  \end{equation*}
  whose size upper bounded by the $k$-partite growth function $\gamma^k_\cH(m)$.

  Fix one $G\in\cH(x)$ and note that since $\lVert\ell_r\rVert_\infty\leq\lVert\ell\rVert_\infty$ and since
  $\EE_{\rn{Z}}[\rn{Z}_\beta] = 0$ and $1-1/\rho\leq\rn{Z}_\beta\leq 1$, Hoeffding's Inequality gives
  \begin{align*}
    \PP_{\rn{Z}}\left[
      \frac{1}{m^k}\cdot
      \left\lvert
      \sum_{\beta\in[m]^k}
      \rn{Z}_\beta\cdot\ell_r\bigl(\beta^*(x), G(\beta), y_\beta\bigr)
      \right\rvert
      >
      \frac{\epsilon}{2}
      \right]
    & \leq
    2\cdot\exp\left(
    -\frac{2\cdot\bigl(m^k\cdot\epsilon/(2\cdot\lVert\ell\rVert_\infty)\bigr)^2}{m^k\cdot\rho^{-2}}
    \right)
    \\
    & =
    2\cdot\exp\left(-\frac{\epsilon^2\cdot\rho^2\cdot m^k}{2\cdot\lVert\ell\rVert_\infty^2}\right),
  \end{align*}
  so by the union bound and recalling that the expression in~\eqref{eq:compare} has the same distribution as the one
  in~\eqref{eq:compare2}, we conclude that
  \begin{equation}\label{eq:Hoeffding}
    \begin{multlined}
      \PP_{\rn{E}_\rho}\left[
        \sup_{H\in\cH}
        \left\lvert
        \frac{1}{m^k}\;\;\sum_{\mathclap{\beta\in[m]^k}}
        \ell_r\Bigl(\beta^*(x), H\bigl(\beta^*(x)\bigr), y_\beta\Bigr)
        - \frac{1}{\rho\cdot m^k}\;\;\sum_{\mathclap{\beta\in\cU_{\rn{E}_\rho(y)}}}
        \ell_r\Bigl(\beta^*(x), H\bigl(\beta^*(x)\bigr), \rn{E}_\rho(y)_\beta\Bigr)
        \right\rvert
        >
        \frac{\epsilon}{2}
        \right]
      \\
      \leq
      2\cdot\gamma^k_\cH(m)\cdot\exp\left(-\frac{\epsilon^2\cdot\rho^2\cdot m^k}{2\cdot\lVert\ell\rVert_\infty^2}\right).
    \end{multlined}
  \end{equation}

  Let $E$ be the event that is the intersection of the complements of the events in~\eqref{eq:Chernoff} and~\eqref{eq:Hoeffding}
  so that the union bound guarantees that
  \begin{equation*}
    \PP_{\rn{E}_\rho}[E]
    \geq
    1
    - 2\cdot\exp\left(-\frac{\epsilon^2\cdot m^k}{12\cdot\lVert\ell\rVert_\infty^2}\right)
    - 2\cdot\gamma^k_\cH(m)\cdot\exp\left(-\frac{\epsilon^2\cdot\rho^2\cdot m^k}{2\cdot\lVert\ell\rVert_\infty^2}\right).
  \end{equation*}

  Consider an outcome $\rn{w}$ of $\rn{E}_\rho(y)$ within the event $E$ and note that
  \begin{align*}
    \MoveEqLeft
    \sup_{H\in\cH} \lvert L_{x,y,\ell}(H) - L_{x,\rn{w},\ell}(H)\rvert
    \\
    & \leq
    \begin{multlined}[t]
      \sup_{H\in\cH}
      \left\lvert
      \frac{1}{m^k}\sum_{\beta\in[m]^k}
      \ell_r\Bigl(\beta^*(x), H\bigl(\beta^*(x)\bigr), y_\beta\Bigr)
      - \frac{1}{\rho\cdot m^k}\sum_{\beta\in\cU_{\rn{E}_\rho(y)}}
      \ell_r\Bigl(\beta^*(x), H\bigl(\beta^*(x)\bigr), y_\beta\Bigr)
      \right\rvert
      \\
      +
      L_{x,\rn{w},\ell}(H)\cdot
      \left\lvert
      1 - \frac{\lvert\cU_{\rn{w}}\rvert}{\rho\cdot m^k}
      \right\rvert
    \end{multlined}
    \\
    & \leq
    \frac{\epsilon}{2}
    + \frac{\lVert\ell\rVert_\infty}{\rho\cdot m^k}
    \cdot\frac{\epsilon\cdot\rho\cdot m^k}{2\cdot\lVert\ell\rVert_\infty}
    \\
    & =
    \epsilon,
  \end{align*}
  concluding the proof of the partite case.

  \medskip

  We now prove the non-partite case. The proof is completely analogous to the partite case, except for the following changes:
  \begin{itemize}
  \item In the non-partite case, we have an order choice $\alpha$ for $[m]$ that determines the orientation of how empirical
    losses are computed; this has no effect on the proof (other than notational) as both empirical and partially erased
    empirical losses are computed with respect to the same order choice.
  \item In the non-partite case, (symmetric) erasure happens on a $k$-set basis rather than $k$-tuple basis, so our random
    variables $\rn{Z}$ that re-encode the difference between the two losses will be indexed by $\binom{[m]}{k}$ instead of
    $[m]^k$.
  \item In the non-partite case, empirical losses are a (normalized) sum of $\binom{m}{k}$ terms (corresponding to $k$-subsets
    of $[m]$) instead of $m^k$ terms (corresponding to $k$-tuples in $[m]$), this change is reflected in the final bound (this
    calculation change is precisely captured by the definition of $M_k$ in~\eqref{eq:peel:Mk}).
  \end{itemize}
  For completeness, we spell out the argument below (omitting some of the intermediate calculation steps):

  Similarly to the partite case, the case $\lVert\ell\rVert_\infty=0$ is trivial (once we interpret the exponentials as
  $\exp(-\infty)=0$), so we assume $\lVert\ell\rVert_\infty > 0$.

  Since $\ell$ is local, we decompose it in terms of a non-agnostic loss function $\ell_r$ and a regularization term $r$ with
  $\lVert\ell_r\rVert_\infty\leq\lVert\ell\rVert_\infty$ (see~\eqref{eq:localellr} and Remark~\ref{rmk:localbounded}).

  We want to show that with high probability, the following quantity is small:
  \begin{multline*}
    \sup_{H\in\cH} \lvert L_{x,y,\ell}^\alpha(H) - L_{x,\rn{E }^{\sym}_\rho(y),\ell}(H)\rvert
    =
    \sup_{H\in\cH}
    \Biggl\lvert
    \frac{1}{\binom{m}{k}}\sum_{U\in\binom{[m]}{k}}
    \ell_r\Bigl(\alpha_U^*(x), b_\alpha\bigl(H^*_m(x)\bigr)_U, b_\alpha(y)_U\Bigr)
    \\
    - \frac{1}{\lvert\cU_{\rn{E}^{\sym}_\rho(y)}\rvert}\sum_{U\in\cU_{\rn{E}^{\sym}_\rho(y)}}
    \ell_r\Bigl(\alpha_U^*(x), b_\alpha\bigl(H^*_m(x)\bigr)_U, b_\alpha(y)_U\Bigr)
    \Biggr\rvert.
  \end{multline*}

  We then note that
  \begin{align*}
    \lvert\cU_{\rn{E}^{\sym}_\rho(y)}\rvert
    & =
    \left\lvert\left\{U\in\binom{[m]}{k} \;\middle\vert\;
    \forall\beta\in([m])_k, (\im(\beta)=U\to y_\beta\neq\unk)\right\}\right\rvert
    \\
    & =
    \left\lvert\left\{U\in\binom{[m]}{k} \;\middle\vert\;
    \unk\notin\im\bigl(b_\alpha(y)_U\bigr)\right\}\right\rvert
  \end{align*}
  has binomial distribution $\Bi(\binom{m}{k},\rho)$, so by multiplicative Chernoff's bound, we have
  \begin{equation}\label{eq:Chernoff:nonpartite}
    \PP_{\rn{E}^{\sym}_\rho}\left[
      \left\lvert\lvert\cU_{\rn{E}^{\sym}_\rho(y)}\rvert - \rho\cdot\binom{m}{k}\right\rvert
      >
      \frac{\epsilon\cdot\rho\cdot\binom{m}{k}}{2\cdot\lVert\ell\rVert_\infty}
      \right]
    \leq
    2\cdot\exp\left(-\frac{\epsilon^2\cdot\binom{m}{k}}{12\cdot\lVert\ell\rVert_\infty^2}\right),
  \end{equation}
  that is, with high probability, the size of $\cU_{\rn{E}^{\sym}_\rho(y)}$ is close to $\rho\cdot\binom{m}{k}$.

  Thus, it will suffice to show that with high probability the following quantity is small:
  \begin{align*}
    \MoveEqLeft
    \begin{multlined}[t]
      \sup_{H\in\cH}
      \biggl\lvert
      \frac{1}{\binom{m}{k}}\sum_{U\in\binom{[m]}{k}}
      \ell_r\Bigl(\alpha_U^*(x), b_\alpha\bigl(H^*_m(x)\bigr)_U, b_\alpha(y)_U\Bigr)
      \\
      - \frac{1}{\rho\cdot\binom{m}{k}}\sum_{U\in\cU_{\rn{E}^{\sym}_\rho(y)}}
      \ell_r\Bigl(\alpha_U^*(x), b_\alpha\bigl(H^*_m(x)\bigr)_U, b_\alpha(y)_U\Bigr)
      \biggr\rvert
    \end{multlined}
    \\
    & =
    \sup_{H\in\cH}
    \frac{1}{\binom{m}{k}}\cdot
    \left\lvert
    \sum_{U\in\binom{[m]}{k}}
    \left(1 - \frac{\One[U\in\cU_{\rn{E}^{\sym}_\rho(y)}]}{\rho}\right)\cdot
    \ell_r\Bigl(\alpha_U^*(x), b_\alpha\bigl(H^*_m(x)\bigr)_U, b_\alpha(y)_U\Bigr)
    \right\rvert.
  \end{align*}

  We then define a collection of i.i.d.\ random variables $\rn{Z}_U$ ($U\in\binom{[m]}{k}$), each of which takes value $1$ with
  probability $1-\rho$ and value $1-1/\rho$ with probability $\rho$ so that the last expression above has the same distribution as
  \begin{align*}
    \MoveEqLeft
    \sup_{H\in\cH}
    \frac{1}{\binom{m}{k}}\cdot
    \left\lvert
    \sum_{U\in\binom{[m]}{k}}
    \rn{Z}_U\cdot\ell_r\Bigl(\alpha_U^*(x), b_\alpha\bigl(H^*_m(x)\bigr)_U, b_\alpha(y)_U\Bigr)
    \right\rvert.
    \\
    & =
    \sup_{G\in\cH(x)}
    \frac{1}{\binom{m}{k}}\cdot
    \left\lvert
    \sum_{U\in\binom{[m]}{k}}
    \rn{Z}_U\cdot\ell_r\Bigl(\alpha_U^*(x), b_\alpha(G)_U, b_\alpha(y)_U\Bigr)
    \right\rvert,
  \end{align*}
  where
  \begin{equation*}
    \cH(x) \df \{H^*_m(x) \mid H\in\cH\},
  \end{equation*}
  whose size is upper bounded by the $k$-ary growth function $\gamma^k_\cH(m)$.

  For a fixed $G\in\cH(x)$, since $\lVert\ell_r\rVert_\infty\leq\lVert\ell\rVert_\infty$, $\EE_{\rn{Z}}[\rn{Z}_U]=0$ and
  $1-1/\rho\leq\rn{Z}_U\leq 1$, Hoeffding's Inequality gives
  \begin{equation*}
    \PP_{\rn{Z}}\left[
    \frac{1}{\binom{m}{k}}\cdot
    \left\lvert
    \sum_{U\in\binom{[m]}{k}}
    \rn{Z}_U\cdot\ell_r\Bigl(\alpha_U^*(x), b_\alpha(G)_U, b_\alpha(y)_U\Bigr)
    \right\rvert
    >
    \frac{\epsilon}{2}
    \right]
    >
    2\cdot\exp\left(-\frac{\epsilon^2\cdot\rho^2\cdot\binom{m}{k}}{2\cdot\lVert\ell\rVert_\infty^2}\right)
  \end{equation*}
  so by the union bound, we conclude that
  \begin{equation}\label{eq:Hoeffding:nonpartite}
    \begin{aligned}
      \begin{multlined}[t]
        \PP_{\rn{E}^{\sym}_\rho}\Biggl[
          \sup_{H\in\cH}
          \biggl\lvert
          \frac{1}{\binom{m}{k}}\sum_{U\in\binom{[m]}{k}}
          \ell_r\Bigl(\alpha_U^*(x), b_\alpha\bigl(H^*_m(x)\bigr)_U, b_\alpha(y)_U\Bigr)
          \\
          - \frac{1}{\rho\cdot\binom{m}{k}}\sum_{U\in\cU_{\rn{E}^{\sym}_\rho(y)}}
          \ell_r\Bigl(\alpha_U^*(x), b_\alpha\bigl(H^*_m(x)\bigr)_U, b_\alpha(y)_U\Bigr)
          \biggr\rvert
          >
          \frac{\epsilon}{2}
          \Biggr]
      \end{multlined}
      \\
      \leq
      2\cdot\gamma_\cH(m)\cdot\exp\left(-\frac{\epsilon^2\cdot\rho^2\cdot\binom{m}{k}}{2\cdot\lVert\ell\rVert_\infty^2}\right).
    \end{aligned}
  \end{equation}

  Letting $E$ be the event that is the intersection of the complements of the events in~\eqref{eq:Hoeffding:nonpartite}
  and~\eqref{eq:Chernoff:nonpartite}, we get
  \begin{equation*}
    \PP_{\rn{E}^{\sym}_\rho}[E]
    \geq
    1-
    - 2\cdot\exp\left(-\frac{\epsilon^2\cdot\binom{m}{k}}{12\cdot\lVert\ell\rVert_\infty^2}\right)
    - 2\cdot\gamma_\cH(m)\cdot\exp\left(-\frac{\epsilon^2\cdot\rho^2\cdot\binom{m}{k}}{2\cdot\lVert\ell\rVert_\infty^2}\right)
  \end{equation*}
  and for every outcome $\rn{w}$ of $\rn{E}^{\sym}_\rho(y)$ within the event $E$, we have
  \begin{align*}
    \MoveEqLeft
    \sup_{H\in\cH}\lvert L_{x,y,\ell}^\alpha(H) - L_{x,\rn{w},\ell}^\alpha(H)\rvert
    \\
    & \leq
    \begin{multlined}[t]
      \sup_{H\in\cH}
      \biggl\lvert
      \frac{1}{\binom{m}{k}}\sum_{U\in\binom{[m]}{k}}
      \ell_r\Bigl(\alpha_U^*(x), b_\alpha\bigl(H^*_m(x)\bigr)_U, b_\alpha(y)_U\Bigr)
      \\
      - \frac{1}{\rho\cdot\binom{m}{k}}\sum_{U\in\cU_{\rn{E}^{\sym}_\rho(y)}}
      \ell_r\Bigl(\alpha_U^*(x), b_\alpha\bigl(H^*_m(x)\bigr)_U, b_\alpha(y)_U\Bigr)
      \biggr\rvert
      \\
      +
      L_{x,\rn{w},\ell}^\alpha(H)\cdot
      \left\lvert
      1 - \frac{\lvert\cU_{\rn{w}}\rvert}{\rho\cdot\binom{m}{k}}
      \right\rvert
    \end{multlined}
    \\
    & \leq
    \frac{\epsilon}{2}
    + \frac{\lVert\ell\rVert_\infty}{\rho\cdot\binom{m}{k}}
    \cdot\frac{\epsilon\cdot\rho\cdot\binom{m}{k}}{2\cdot\lVert\ell\rVert_\infty}
    \\
    & =
    \epsilon,
  \end{align*}
  concluding the proof.
\end{proof}

\begin{restatable}[Finite $\VCN_{k,k}$-dimension implies sample uniform convergence]{proposition}{propVCNkktoSUC}
  \label{prop:VCNkk->SUC}
  Let $k\in\NN_+$, let $\Omega=(\Omega_i)_{i=1}^k$ be a $k$-tuple of non-empty Borel spaces (a single Borel space,
  respectively), let $\Lambda$ be a finite non-empty Borel space, let $\cH\subseteq\cF_k(\Omega,\Lambda)$ be a $k$-partite ($k$-ary,
  respectively) hypothesis class with $\VCN_{k,k}(\cH) < \infty$ and let $\ell$ be a $k$-partite ($k$-ary, respectively)
  agnostic loss function that is bounded and local. In the non-partite case, we further suppose that $\ell$ is symmetric.

  Finally, let
  \begin{equation*}
    B_\ell
    \df
    \begin{dcases*}
       \max\left\{\frac{1}{2}, \lVert\ell\rVert_\infty\right\}, & if $k=1$,\\
       \max\left\{\frac{1}{4\cdot k}, \lVert\ell\rVert_\infty\right\}, & if $k\geq 2$.
    \end{dcases*}
  \end{equation*}

  Then $\cH$ has the sample uniform convergence property with respect to $\ell$.

  The corresponding associated function is as follows:
  \begin{itemize}
  \item When $\lvert\Lambda\rvert=1$, we have $m^{\SUC}_{\cH,\ell}\equiv 1$.
  \item When $\lvert\Lambda\rvert\geq 2$ and $k=1$, we have
    \begin{align*}
      \MoveEqLeft
      m^{\SUC}_{\cH,\ell}(\epsilon,\delta,\rho)
      \\
      & \df
      \max\bigggl\{\frac{12\cdot\lVert\ell\rVert_\infty^2}{\epsilon^2}\cdot\ln\frac{4}{\delta},
      \\
      & \quad
      \frac{2e}{e-1}
      \cdot\frac{2\cdot B_\ell^2\cdot\VCN_{k,k}(\cH)}{\epsilon^2\cdot\rho^2}
      \cdot
      \ln\frac{4\cdot B_\ell^2\cdot\VCN_{k,k}(\cH)}{\epsilon^2\cdot\rho^2}
      \\
      & \qquad
      + \frac{4\cdot B_\ell^2}{\epsilon^2\cdot\rho^2}
      \cdot\left(
      \VCN_{k,k}(\cH)\cdot\ln\binom{\lvert\Lambda\rvert}{2} + \ln\frac{4}{\delta}
      \right)
      + 1
      \bigggr\}
      \\
      & =
      O\left(
      \frac{\lVert\ell\rVert_\infty^2}{\epsilon^2\cdot\rho^2}
      \cdot\left(
      \VCN_{k,k}(\cH)\cdot\ln\frac{\lVert\ell\rVert_\infty\cdot\VCN_{k,k}(\cH)}{\epsilon^2\cdot\rho^2}
      + \VCN_{k,k}(\cH)\cdot\ln\lvert\Lambda\rvert
      + \ln\frac{1}{\delta}
      \right)
      \right).
    \end{align*}
  \item When $\lvert\Lambda\rvert\geq 2$ and $k\geq 2$, in the partite case, we have
    \begin{align*}
      \MoveEqLeft
      m^{\SUC}_{\cH,\ell}(\epsilon,\delta,\rho)
      \\
      & \df
      \max\bigggl\{\left(\frac{12\cdot\lVert\ell\rVert_\infty^2}{\epsilon^2}\cdot\ln\frac{4}{\delta}\right)^{1/k},
      \\
      & \quad
      \Biggl(
      \frac{e}{e-1}
      \cdot\frac{16\cdot k^2\cdot B_\ell^2\cdot (\VCN_{k,k}(\cH)+1)^{k-1}}{\epsilon^2\cdot\rho^2}
      \cdot\ln
      \frac{16\cdot k^2\cdot B_\ell^2\cdot (\VCN_{k,k}(\cH)+1)^{k-1}}{\epsilon^2\cdot\rho^2}
      \\
      & \qquad
      + \frac{16\cdot k\cdot B_\ell^2}{\epsilon^2\cdot\rho^2}\cdot\ln\binom{\lvert\Lambda\rvert}{2}
      + 1
      \Biggr)^{(\VCN_{k,k}(\cH)+1)^{k-1}}
      + \left(\frac{4\cdot B_\ell^2}{\epsilon^2\cdot\rho^2}\cdot\ln\frac{4}{\delta}\right)^{1/k}
      \bigggr\}
      \\
      & =
      O\Bigggl(
      \bigggl(
      \frac{k\cdot\lVert\ell\rVert_\infty^2}{\epsilon^2\cdot\rho^2}
      \cdot\Biggl(
      k\cdot\VCN_{k,k}(\cH)^{k-1}\cdot\ln\frac{k\cdot\lVert\ell\rVert_\infty\cdot\VCN_{k,k}(\cH)^{k-1}}{\epsilon\cdot\rho}
      \\
      & \qquad
      + \ln\lvert\Lambda\rvert
      \Biggr)
      \bigggr)\Bigggr)^{(\VCN_{k,k}(\cH)+1)^{k-1}}
      +
      O\left(
      \left(\frac{\lVert\ell\rVert_\infty^2}{\epsilon^2\cdot\rho^2}\cdot\ln\frac{1}{\delta}\right)^{1/k}
      \right).
    \end{align*}
  \item When $\lvert\Lambda\rvert\geq 2$ and $k\geq 2$, in the non-partite case, we have
    \begin{align*}
      \MoveEqLeft
      m^{\SUC}_{\cH,\ell}(\epsilon,\delta,\rho)
      \\
      & \df
      \max\bigggl\{k\cdot\left(\frac{12\cdot\lVert\ell\rVert_\infty^2}{\epsilon^2}\cdot\ln\frac{4}{\delta}\right)^{1/k},
      \\
      & \quad
      \Biggl(
      \frac{e}{e-1}
      \cdot\frac{16\cdot B_\ell^2\cdot k^{k+1}\cdot (\VCN_{k,k}(\cH)+1)^{k-1}}{(k-1)!\cdot\epsilon^2\cdot\rho^2}
      \cdot\ln\frac{16\cdot B_\ell^2\cdot k^{k+1}\cdot (\VCN_{k,k}(\cH)+1)^{k-1}}{(k-1)!\cdot\epsilon^2\cdot\rho^2}
      \\
      & \qquad
      + \frac{16\cdot B_\ell^2\cdot k^k}{(k-1)!\cdot\epsilon^2\cdot\rho^2}
      \cdot\left(\ln\binom{\lvert\Lambda\rvert^{k!}}{2} - \ln k!\right)
      + k!
      \Biggr)^{(\VCN_{k,k}(\cH)+1)^{k-1}}
      \\
      & \qquad
      + \left(\frac{4\cdot B_\ell^2\cdot k^k}{\epsilon^2\cdot\rho^2}\cdot\ln\frac{4}{\delta}\right)^{1/k}
      \bigggr\}
      \\
      & =
      O\Bigggl(
      \bigggl(
      \frac{k^k\cdot\lVert\ell\rVert_\infty^2}{(k-1)!\cdot\epsilon^2\cdot\rho^2}
      \cdot\Biggl(
      k\cdot\VCN_{k,k}(\cH)^{k-1}
      \cdot\ln\frac{k^{k+1}\cdot\lVert\ell\rVert_\infty\cdot\VCN_{k,k}(\cH)^{k-1}}{(k-1)!\cdot\epsilon^2\cdot\rho^2}
      \\
      & \qquad
      + k!\cdot\ln\lvert\Lambda\rvert\Biggr)
      + k!
      \bigggr)\Bigggr)^{(\VCN_{k,k}(\cH)+1)^{k-1}}
      +
      O\left(
      \left(\frac{\lVert\ell\rVert_\infty^2\cdot k^k}{\epsilon^2\cdot\rho^2}\cdot\ln\frac{1}{\delta}\right)^{1/k}
      \right).
    \end{align*}
  \end{itemize}
\end{restatable}

\begin{proof}[Proof (sketch).]
  Here, we will only show that some $m^{\SUC}_{\cH,\ell}$ exists and defer precise computations to Appendix~\ref{sec:calc}.

  By Lemma~\ref{lem:peel}, we know that for every $m\in\NN_+$ and every $[m]$-sample $(x,y)$ (and in the non-partite case every
  order choice $\alpha$ for $[m]$), we have
  \begin{equation*}
    \sup_{H\in\cH} \lvert L_{x,y,\ell}(H) - L_{x,\rn{E}_\rho(y),\ell}(H)\rvert
    \leq
    \epsilon
  \end{equation*}
  with probability at least
  \begin{equation*}
    1
    - 2\cdot\exp\left(-\frac{\epsilon^2\cdot M_k}{12\cdot\lVert\ell\rVert_\infty^2}\right)
    - 2\cdot\gamma_\cH(m)\cdot\exp\left(-\frac{\epsilon^2\cdot\rho^2\cdot M_k}{2\cdot\lVert\ell\rVert_\infty^2}\right),
  \end{equation*}
  (and in the non-partite case, we replace both instances of $L$ by $L^\alpha$), so it suffices to show that when $m\geq
  m^{\SUC}_{\cH,\ell}(\epsilon,\delta,\rho)$, the quantity above is at least $1-\delta$. In turn, it suffices to show that each
  of the negative terms above is at most $\delta/2$ in absolute value, or, equivalently, show that
  \begin{gather*}
    M_k \geq \frac{12\cdot\lVert\ell\rVert_\infty^2}{\epsilon^2}\cdot\ln\frac{4}{\delta},
    \\
    \ln\bigl(\gamma_\cH(m)\bigr)
    -\frac{\epsilon^2\cdot\rho^2\cdot M_k}{2\cdot\lVert\ell\rVert_\infty^2}
    \leq
    \ln\frac{\delta}{4}.
  \end{gather*}

  Recalling that $M_k = \Theta(m^k)$, the former one clearly holds if $m$ is large.

  For the latter one, using Lemma~\ref{lem:VCNkk->kgrowth}, it suffices to show that
  \begin{align*}
    \MoveEqLeft
    \ln\frac{\delta}{4} + \frac{\epsilon^2\cdot\rho^2\cdot M_k}{2\cdot\lVert\ell\rVert_\infty^2}
    \\
    & \geq
    \begin{dcases*}
      2\cdot k
      \cdot m^{k-1/(\VCN_{k,k}(\cH)+1)^{k-1}}
      \cdot\left(\ln(m^k+1) + \ln\binom{\lvert\Lambda\rvert}{2}\right),
      & in the partite if $k\geq 2$,
      \\
      \frac{2\cdot m^{k-1/(\VCN_{k,k}(\cH)+1)^{k-1}}}{(k-1)!}
      \!\cdot\!\left(\ln\left(\binom{m}{k}+1\right) + \ln\binom{\lvert\Lambda\rvert^{k!}}{2}\right),
      & in the non-partite if $k\geq 2$,
      \\
      \VCN_{k,k}(\cH)\cdot\left(\ln(m+1) + \ln\binom{\lvert\Lambda\rvert}{2}\right),
      & if $k=1$.
    \end{dcases*}
  \end{align*}
  Again, since $M_k = \Theta(m^k)$, by analyzing the exponents of $m$, we see that the above holds when $m$ is sufficiently
  large.
\end{proof}

\section{Sample uniform convergence implies adversarial sample completion learnability}

In this section, we show that sample uniform convergence implies adversarial sample completion learnability. Our notation was
carefully set up so that we can prove both the partite and non-partite versions essentially simultaneously. Lemma~\ref{lem:repr}
below says that the sample completion version of representativeness captures the notion we expect it to capture towards showing
that sample uniform convergence implies adversarial sample completion learnability in Proposition~\ref{prop:SUC->advSC}.
Both proofs are straightforward adaptations of their classical PAC counterparts.

\begin{lemma}[Representativeness]\label{lem:repr}
  Let $k\in\NN_+$, let $\Omega=(\Omega_i)_{i=1}^k$ be a $k$-tuple of non-empty Borel spaces (a single non-empty Borel space,
  respectively), let $\Lambda$ be a non-empty Borel space, let $\cH\subseteq\cF_k(\Omega,\Lambda)$ be a $k$-partite ($k$-ary,
  respectively) hypothesis class, let $\ell$ be a $k$-partite ($k$-ary, respectively) agnostic loss function, let $m\in\NN$, let
  $(x,y)$ be a partially erased $[m]$-sample and let $y'$ extend $y$. In the non-partite case, we also let $\alpha$ be an order
  choice for $[m]$.

  If $\cA$ is an empirical risk minimizer for $\ell$ and $(x,y)$ is $\epsilon/2$-representative with respect to $\cH$, $y'$ and
  $\ell$, then
  \begin{align*}
    L_{x,y',\ell}\bigl(\cA(x,y)\bigr) & \leq \inf_{H\in\cH} L_{x,y',\ell}(H) + \epsilon
    \intertext{in the partite case and}
    L_{x,y',\ell}^\alpha\bigl(\cA(x,y)\bigr) & \leq \inf_{H\in\cH} L_{x,y',\ell}^\alpha(H) + \epsilon
  \end{align*}
  in the non-partite case.
\end{lemma}

\begin{proof}
  In the partite case, this follows from
  \begin{equation*}
    L_{x,y',\ell}\bigl(\cA(x,y)\bigr)
    \leq
    L_{x,y,\ell}\bigl(\cA(x,y)\bigr) + \frac{\epsilon}{2}
    =
    \inf_{H\in\cH} L_{x,y,\ell}(H) + \frac{\epsilon}{2}
    \leq
    \inf_{H\in\cH} L_{x,y',\ell}(H) + \epsilon,
  \end{equation*}
  where the inequalities are due to $\epsilon$-representativeness and the equality is due to $\cA$ being an empirical risk
  minimizer. The non-partite case has the same proof by adding a superscript $\alpha$ to all instances of $L$.
\end{proof}

\begin{proposition}[Sample uniform convergence implies adversarial sample completion learnability]\label{prop:SUC->advSC}
  Let $k\in\NN_+$, let $\Omega=(\Omega_i)_{i=1}^k$ be a $k$-tuple of non-empty Borel spaces (a single non-empty Borel space,
  respectively), let $\Lambda$ be a non-empty Borel space, let $\cH\subseteq\cF_k(\Omega,\Lambda)$ be a $k$-partite ($k$-ary,
  respectively) hypothesis class and let $\ell$ be a $k$-partite ($k$-ary, respectively) agnostic loss function. Suppose
  completion (almost) empirical risk minimizers exist (see Remark~\ref{rmk:ERM}).

  If $\cH$ has the adversarial sample uniform convergence with respect to $\ell$, then $\cH$ is adversarial sample completion
  $k$-PAC learnable (and in the non-partite case, both symmetric and non-symmetric sample completion learnability hold). More
  precisely, any completion empirical risk minimizer $\cA$ for $\ell$ is an adversarial sample completion $k$-PAC learner for
  $\cH$ with
  \begin{align*}
    m^{\advSC}_{\cH,\ell,\cA}(\epsilon,\delta,\rho)
    & =
    \begin{dcases*}
      m^{\SUC}_{\cH,\ell}\left(\frac{\epsilon}{2},\delta,\rho\right), & in the partite case,\\
      m^{\SUC}_{\cH,\ell}\left(\frac{\epsilon}{2},\delta,\rho^{k!}\right), & in the non-partite case.
    \end{dcases*}
    \\
    m^{\advsSC}_{\cH,\ell,\cA}(\epsilon,\delta,\rho)
    & =
    m^{\SUC}_{\cH,\ell}\left(\frac{\epsilon}{2},\delta,\rho\right).
  \end{align*}
\end{proposition}

\begin{proof}
  The non-partite non-symmetric case follows from the non-partite symmetric case by Remark~\ref{rmk:symm->nonsymm}.

  Let us prove the partite and non-partite symmetric cases simultaneously.

  Let $m\geq m^{\SUC}_{\cH,\ell}(\epsilon/2,\delta,\rho)$ be an integer and suppose that $\cA$ is an empirical risk minimizer
  for $\ell$. If $(x,y)$ is an $[m]$-sample, then with probability at least $1-\delta$, we have that
  \begin{equation*}
    (x,\rn{w})
    \df
    \begin{dcases*}
      (x,\rn{E}_\rho(y)), & in the partite case,\\
      (x,\rn{E}^{\sym}_\rho(y)), & in the non-partite symmetric case
    \end{dcases*}
  \end{equation*}
  is $\epsilon/2$-representative with respect to $\cH$, $y$ and $\ell$. By Lemma~\ref{lem:repr}, for all such outcomes of
  $\rn{w}$, we have
  \begin{equation*}
    L_{x,y,\ell}^\alpha\Bigl(\cA\bigl(x,\rn{w}\bigr)\Bigr) \leq \inf_{H\in\cH} L_{x,y,\ell}^\alpha(H) + \epsilon
  \end{equation*}
  for every order choice $\alpha$ for $[m]$ (where $\alpha$ is dropped in the partite case), so we conclude that
  \begin{equation*}
    \PP_{\rn{w}}\biggl[
      L_{x,y,\ell}^\alpha\Bigl(\cA\bigl(x,\rn{w}\bigr)\Bigr)
      \leq
      \inf_{H\in\cH} L_{x,y,\ell}^\alpha(H) + \epsilon
      \biggr]
    \geq 1 - \delta,
  \end{equation*}
  as desired.
\end{proof}

\section{Probabilistic Haussler packing property}

In this section, we prove the final two implications that involve the $m^k$-probabilistic Haussler packing property, namely,
that it is implied by sample completion $k$-PAC learnability (Proposition~\ref{prop:SC->PHP}) and that it implies finite
$\VCN_{k,k}$-dimension (Proposition~\ref{prop:PHP->VCNkk}).

\begin{restatable}[Sample completion $k$-PAC learnability implies $m^k$-probabilistic Haussler packing property]{proposition}{propSCtoPHP}
  \label{prop:SC->PHP}
  Let $k\in\NN_+$, let $\Omega=(\Omega_i)_{i=1}^k$ be a $k$-tuple of non-empty Borel spaces (a single non-empty Borel space),
  let $\Lambda$ be a finite non-empty Borel space, let $\cH\subseteq\cF_k(\Omega,\Lambda)$ be a $k$-partite ($k$-ary,
  respectively) hypothesis class and let $\ell$ be a $k$-partite ($k$-ary, respectively) loss function. Suppose that either
  $\ell$ is metric or $\ell$ is separated and bounded and let
  \begin{align*}
    c_\ell & \df
    \begin{dcases*}
      1, & if $\ell$ is metric,\\
      \frac{s(\ell)}{\lVert\ell\rVert_\infty}, & otherwise,
    \end{dcases*}
    &
    K & \df
    \begin{dcases*}
      0, & in the partite case,\\
      k - 1, & in the non-partite case.
    \end{dcases*}
  \end{align*}

  If $\cH$ is sample completion $k$-PAC learnable with respect to $\ell$ with a sample completion $k$-PAC learner $\cA$, then
  $\cH$ has the $m^k$-probabilistic Haussler packing property with respect to $\ell$ with associated function
  \begin{equation}\label{eq:SC->PHP:m}
    m^{\hPHP[m^k]}_{\cH,\ell}(\epsilon,\delta,\rho)
    \df
    \min_{\widetilde{\rho},\widetilde{\delta}}
    \bigggl\lceil
    \max\Biggl\{
    m^{\SC}_{\cH,\ell,\cA}\left(
    \frac{c_\ell\cdot\epsilon}{2}, \widetilde{\delta}, \widetilde{\rho}
    \right),
    \left(
    \frac{
      \ln(\delta) - \ln(\delta-\widetilde{\delta})
    }{
      \rho\cdot\ln(2) - \ln\bigl(\widetilde{\rho}\cdot(\lvert\Lambda\rvert-1) + 1\bigr)
    }
    \right)^{1/k}
    + K
    \Biggr\}
    \bigggr\rceil
  \end{equation}
  when $\lvert\Lambda\rvert\geq 2$, where the minimum is over all
  \begin{align*}
    \widetilde{\delta} & \in (0,\delta), &
    \widetilde{\rho} & \in \left(0,\frac{2^\rho-1}{\lvert\Lambda\rvert-1}\right).
  \end{align*}
  and $m^{\hPHP[m^k]}_{\cH,\ell}\equiv 1$ when $\lvert\Lambda\rvert=1$.
\end{restatable}

\begin{proof}[Proof (sketch).]
  Here, we will cover only the partite case when $\ell$ is metric and we will only show that some $m^{\hPHP[m^k]}_{\cH,\ell}$
  exists; we defer the general case and precise computations to Appendix~\ref{sec:calc}.

  Suppose $m$ is a sufficiently large integer, $\mu\in\Pr(\Omega)$, let $H_1,\ldots,H_t\in\cH$ be such that $t\geq 2^{\rho\cdot
    m^k}$ and let
  \begin{equation*}
    S_\epsilon
    \df
    \{x\in\cE_m(\Omega) \mid (H_1,\ldots,H_t)\text{ is $\epsilon$-separated on $x$ w.r.t.\ $\ell$}\}.
  \end{equation*}
  Our goal is to show that $\mu(S_\epsilon)\leq\delta$.

  For this, we will use sample completion $k$-PAC learnability with parameters $\epsilon/2, \delta/2, \widetilde{\rho}$, where
  the last one is going to be sufficiently small, but depending only on $\rho$ and $\lvert\Lambda\rvert$. We assume that $m$ is
  larger than $m^{\SC}_{\cH,\ell,\cA}(\epsilon/2,\delta/2,\widetilde{\rho})$.

  Let us encode the erasure operation $\rn{E}_{\widetilde{\rho}}$ in a different manner. Given $y\in\Lambda^{[m]^{(k)}}$ and
  $w\in\{0,1\}^{[m]^{(k)}}$, let $E(y,w)\in(\Lambda\cup\{\unk\})^{[m]^{(k)}}$ be given by
  \begin{equation*}
    E(y,w)_\beta \df
    \begin{dcases*}
      y_\beta, & if $w_\beta=1$,\\
      \unk, & if $w_\beta=0$
    \end{dcases*}
  \end{equation*}
  and note that if $\nu_m\in\Pr(\{0,1\}^{[m]^{(k)}})$ is the distribution in which each entry is $1$ independently with
  probability $\widetilde{\rho}$, then for $\rn{w}\sim\nu_m$, we have $\rn{E}_{\widetilde{\rho}}(y)\sim E(y,\rn{w})$.

  For each $i\in[t]$, let
  \begin{equation*}
    G_i
    \df
    \left\{(x,w)\in\cE_m(\Omega)\times\{0,1\}^{[m]^{(k)}}
    \;\middle\vert\;
    L_{x,(H_i)^*_m(x),\ell}\biggl(\cA\Bigl(x, E\bigl((H_i)^*_m(x), w\bigr)\Bigr)\biggr)
    \leq
    \frac{\epsilon}{2}
    \right\}.
  \end{equation*}

  Note that since $m\geq m^{\SC}_{\cH,\ell,\cA}(\epsilon/2,\delta/2,\widetilde{\rho})$, sample completion $k$-PAC learnability
  guarantees that
  \begin{equation}\label{eq:SC->PHP:SC:simplified}
    \PP_{\rn{x}\sim\mu^m}\Bigl[\PP_{\rn{w}\sim\nu_m}\bigl[(\rn{x},\rn{w})\in G_i\bigr]\Bigr]
    \geq
    1 - \frac{\delta}{2}.
  \end{equation}

  We claim that every fixed $(x,w)\in S_\epsilon\times\{0,1\}^{[m]^{(k)}}$ is in at most $\lvert\Lambda\rvert^{\lvert
    w^{-1}(1)\rvert}$ many $G_i$. To see this, first note that for all $i\in[t]$, exactly the same entries of
  $E((H_i)^*_m(x),w)$ are $\unk$; namely, these are exactly the entries of $w$ that are $0$. If $(x,w)\in
  S_\epsilon\times\{0,1\}^{[m]^{(k)}}$ is in more than $\lvert\Lambda\rvert^{\lvert w^{-1}(1)\rvert}$ many $G_i$, then by
  Pigeonhole Principle, there must exist $i,j\in[t]$ with $i < j$ such that $(x,w)\in G_i\cap G_j$ and $E((H_i)^*_m(x),w) =
  E((H_j)^*_m(x),w)$, which in particular implies $\cA(x, E((H_i)^*_m(x),w)) = \cA(x, E((H_j)^*_m(x),w))$, hence we get
  \begin{align*}
    \epsilon
    & \geq
    L_{x,(H_i)^*_m(x),\ell}\biggl(\cA\Bigl(x, E\bigl((H_i)^*_m(x), w\bigr)\Bigr)\biggr)
    + L_{x,(H_j)^*_m(x),\ell}\biggl(\cA\Bigl(x, E\bigl((H_j)^*_m(x), w\bigr)\Bigr)\biggr)
    \\
    & \geq
    L_{x,(H_i)^*(x),\ell}(H_j),
  \end{align*}
  where the second inequality follows since $\ell$ is metric. However, this contradicts the fact that $(H_1,\ldots,H_t)$ is
  $\epsilon$-separated on $x$ with respect to $\ell$ as $x\in S_\epsilon$. Thus, we conclude that for every $(x,w)\in
  S_\epsilon\times\{0,1\}^{[m]^{(k)}}$, we have
  \begin{equation*}
    \sum_{i\in[t]} \One_{G_i}(x,w) \leq \lvert\Lambda\rvert^{\lvert w^{-1}(1)\rvert}.
  \end{equation*}

  Putting this together with~\eqref{eq:SC->PHP:SC:simplified}, we get
  \begin{align*}
    \left(1 - \frac{\delta}{2}\right)\cdot t
    & \leq
    \EE_{\rn{x}\sim\mu^m}\left[\EE_{\rn{w}\sim\nu_m}\left[\sum_{i\in[t]}\One_{G_i}(\rn{x},\rn{w})\right]\right]
    \\
    & =
    \begin{multlined}[t]
      \mu(S_\epsilon)\cdot\EE_{\rn{x}\sim\mu^m}\left[\EE_{\rn{w}\sim\nu_m}\left[\sum_{i\in[t]}\One_{G_i}(\rn{x},\rn{w})\right]
        \;\middle\vert\;\rn{x}\in S_\epsilon\right]
      \\
      + \bigl(1-\mu(S_\epsilon)\bigr)\cdot\EE_{\rn{x}\sim\mu^m}\left[\EE_{\rn{w}\sim\nu_m}\left[\sum_{i\in[t]}\One_{G_i}(\rn{x},\rn{w})\right]
        \;\middle\vert\;\rn{x}\notin S_\epsilon\right]
    \end{multlined}
    \\
    & \leq
    \mu(S_\epsilon)\cdot\EE_{\rn{w}\sim\nu_m}[\lvert\Lambda\rvert^{\lvert\rn{w}^{-1}(1)\rvert}]
    + \bigl(1-\mu(S_\epsilon)\bigr)\cdot t
    \\
    & =
    \mu(S_\epsilon)\cdot\bigl(\widetilde{\rho}\cdot\lvert\Lambda\rvert + (1-\widetilde{\rho})\bigr)^{m^k}
    + \bigl(1-\mu(S_\epsilon)\bigr)\cdot t
    \\
    & =
    t
    + \mu(S_\epsilon)\cdot\Bigl(\bigl(\widetilde{\rho}\cdot(\lvert\Lambda\rvert - 1) + 1\bigr)^{m^k} - t\Bigr),
  \end{align*}
  where the second equality follows since the entries of $\rn{w}$ are independent Bernoulli variables with parameter
  $\widetilde{\rho}$. Thus, we get
  \begin{equation*}
    \mu(S_\epsilon)\cdot\Bigl(t - 2^{C_{\widetilde{\rho},\lvert\Lambda\rvert}\cdot m^k}\Bigr)
    \leq
    \frac{\delta}{2}\cdot t,
  \end{equation*}
  where
  \begin{equation*}
    C_{\widetilde{\rho},\lvert\Lambda\rvert}
    \df
    \log_2\bigl(\widetilde{\rho}\cdot(\lvert\Lambda\rvert - 1) + 1\bigr).
  \end{equation*}

  Since $t\geq 2^{\rho\cdot m^k}$, if we assume that $\widetilde{\rho}$ is sufficiently small in terms of $\rho$ and
  $\lvert\Lambda\rvert$ so that $C_{\widetilde{\rho},\lvert\Lambda\rvert} < \rho$, then
  \begin{equation*}
    \mu(S_\epsilon^\alpha)
    \leq
    \frac{\delta}{2}
    \cdot
    \frac{t}{
      t - 2^{C_{\widetilde{\rho},\lvert\Lambda\rvert}\cdot m^{(k)}}
    }
    \leq
    \frac{\delta}{2}
    \cdot
    \frac{2^{\rho\cdot m^{(k)}}}{
      2^{\rho\cdot m^k} - 2^{C_{\widetilde{\rho},\lvert\Lambda\rvert}\cdot m^k}
    },
  \end{equation*}
  where the second inequality follows from $t\geq 2^{\rho\cdot m^k}$ and the fact that for
  $c\df(\widetilde{\rho}\cdot(\lvert\Lambda\rvert-1) + 1)^{m^{(k)}} > 0$, the function $(c,\infty)\ni x\mapsto
  x/(x-c)\in\RR_{\geq 0}$ is decreasing. Now since $C_{\widetilde{\rho},\lvert\Lambda\rvert} < \rho$, if $m$ is large enough
  then the last fraction in the above is at most $2$, hence the whole expression is at most $\delta$, as desired.
\end{proof}

\begin{restatable}[$m^k$-probabilistic Haussler packing property implies finite $\VCN_{k,k}$-dimension]{proposition}{propPHPtoVCNkk}
  \label{prop:PHP->VCNkk}
  Let $k\in\NN_+$, let $\Omega=(\Omega_i)_{i=1}^k$ be a $k$-tuple of non-empty Borel spaces (a single non-empty Borel space,
  respectively), let $\Lambda$ be a non-empty Borel space, let $\cH\subseteq\cF_k(\Omega,\Lambda)$ be a $k$-partite ($k$-ary,
  respectively) hypothesis class and let $\ell$ be a $k$-partite ($k$-ary, respectively) loss function that is separated. Let
  also
  \begin{equation*}
    h_2(t) \df t\cdot\log_2\frac{1}{t} + (1-t)\cdot\log_2\frac{1}{1-t}
  \end{equation*}
  denote the binary entropy.

  Suppose $\cH$ has the $m^k$-probabilistic Haussler packing property with respect to $\ell$.

  Then in the partite case, we have
  \begin{equation}\label{eq:PHP->VCNkk:VCNkk:partite}
    \begin{aligned}
      \VCN_{k,k}(\cH)
      & \leq
      \min_{\epsilon,\delta,\rho}
      \max\Bigggl\{m^2,
      \\ & \qquad
      \Floor{
        \left(
        d
        - \log_2\left(
        1 - \left(\frac{1 - (1-1/m)^{k\cdot m}\cdot(1 - 2^{(h_2(\epsilon/s(\ell)) - 1)\cdot m^k + d})}{1-\delta}\right)^{1/d}
        \right)
        \right)^{1/k}
      }
      \Bigggr\},
    \end{aligned}
  \end{equation}
  where the minimum is over
  \begin{align*}
    \epsilon & \in \left(0,\frac{s(\ell)}{2}\right), &
    \delta & \in(0,4^{-k}), &
    \rho & \in \left(0,1 - h_2\left(\frac{\epsilon}{s(\ell)}\right)\right),
  \end{align*}
  and
  \begin{gather}
    m
    \df
    \Ceil{
      \max\left\{
      2,
      m^{\hPHP[m^k]}_{\cH,\ell}(\epsilon,\delta,\rho),
      \left(
      \frac{1 - \log_2(1-4^k\cdot\delta)}{1-h_2(\epsilon/s(\ell))-\rho}
      \right)^{1/k}
      \right\}
    },
    \label{eq:PHP->VCNkk:m:partite}
    \\
    d \df \ceil{\rho\cdot m^k}.
    \label{eq:PHP->VCNkk:d:partite}
  \end{gather}

  And in the non-partite case we have
  \begin{equation}\label{eq:PHP->VCNkk:VCNkk}
    \begin{multlined}
      \VCN_{k,k}(\cH)
      \leq
      \min_{\epsilon,\delta,\rho}
      \max\Biggggl\{\frac{m^2}{k},
      \biggggl\lfloor
      \Bigggl(
      d
      - \log_2\bigggl(
      1
      \\
      - \left(\frac{
        1 - ((1-1/m)^m - k\cdot e^{-m/(8\cdot k)})\cdot(1 - 2^{(h_2(\epsilon\cdot(2\cdot k)^k/(k!\cdot s(\ell))) - 1)(m/(2\cdot k))^k + d})
      }{
        1-\delta
      }\right)^{1/d}
      \bigggr)
      \Bigggr)^{1/k}
      \biggggr\rfloor
      \Biggggr\},
    \end{multlined}
  \end{equation}
  where the minimum is over
  \begin{align*}
    \epsilon & \left(0, \frac{k!\cdot s(\ell)}{2\cdot(2\cdot k)^k}\right), &
    \delta & \in \left(0, \frac{1}{12}\right), &
    \rho & \in \left(0, \frac{1-h_2(\epsilon\cdot(2\cdot k)^k/(k!\cdot s(\ell)))}{(2\cdot k)^k}\right) , &
  \end{align*}
  and
  \begin{gather}
    \begin{multlined}
      m
      \df
      \Biggl\lceil
      \max\Biggl\{
      8\cdot k\cdot\ln(4\cdot k),
      \;
      m^{\hPHP[m^k]}_{\cH,\ell}(\epsilon,\delta,\rho),
      \\
      2\cdot k\cdot
      \left(
      \frac{1 - \log_2(1 - 12\cdot\delta)}{1 - h_2(\epsilon\cdot(2\cdot k)^k/(k!\cdot s(\ell))) - \rho\cdot(2\cdot k)^k}
      \right)^{1/k}
      \Biggr\}
      \Biggr\rceil,
    \end{multlined}
    \label{eq:PHP->VCNkk:m}
    \\
    d\df\ceil{\rho\cdot m^k}.
    \label{eq:PHP->VCNkk:d}
  \end{gather}
\end{restatable}

\begin{proof}[Proof (sketch).]
  Here, we will cover only the partite case and we will only show that $\VCN_{k,k}(\cH)$ is finite; we defer precise
  computations and the non-partite case to Appendix~\ref{sec:calc}.

  Let $\epsilon$, $\delta$ and $\rho$ be small enough to be chosen later and $m$ be large enough also to be chosen later, but
  let us already ensure that $m\geq m^{\hPHP[m^k]}(\epsilon,\delta,\rho)$.

  Suppose $n\in\NN$ is such that $\VCN_{k,k}(\cH)\geq n$ and let us show that if $m$ is large enough, then $n$ being large
  enough in terms of this $m$ leads to a contradiction with $m\geq m^{\hPHP[m^k]}(\epsilon,\delta,\rho)$.

  As per definition of $\VCN_{k,k}(\cH)$, we know that there exists $z\in\cE_n(\Omega)$ such that
  \begin{equation*}
    \cH_z \df \{H^*_n(z) \mid H\in\cH\} \subseteq \Lambda^{[n]^k}
  \end{equation*}
  Natarajan-shatters $[n]^k$. It will be convenient to index our witnesses to the shattering by $\FF_2^{[n]^k}$. Namely, we know
  that there exist $f_0,f_1\colon [n]^k\to\Lambda$ with $f_0(\beta)\neq f_1(\beta)$ for every $\beta\in[n]^k$ and $H_w\in\cH$
  ($w\in\FF_2^{[n]^k}$) such that for every $w\in\FF_2^{[n]^k}$ and every $\beta\in[n]^k$, we have $(H_w)^*_n(z)_\beta =
  f_{w_\beta}(\beta)$.

  We will prove that there exists $C\subseteq\FF_2^{[n]^k}$ of size at least $2^{\rho\cdot m^k}$ and $\mu\in\Pr(\Omega)$ such
  that if $C = \{w_1,\ldots,w_{\lvert C\rvert}\}$ and $\rn{x}\sim\mu^m$, then $(H_{w_1},\ldots,H_{w_{\lvert C\rvert}})$ is
  $\epsilon$-separated on $\rn{x}$ with probability larger than $\delta$, hence leading to a contradiction with $m\geq
  m^{\hPHP[m^k]}(\epsilon,\delta,\rho)$.

  We want to view $C\subseteq\FF_2^{[n]^k}$ as a binary code (more specifically, we will choose a linear code) so that we can
  invoke some (basic) techniques of coding theory. Recall that a linear code (over $\FF_2$) with base set $X$ is an
  $\FF_2$-linear subspace $C$ of $\FF_2^X$ and the \emph{distance} of $C$ is defined as
  \begin{equation*}
    \dist(C)
    \df
    \inf_{\substack{w_1,w_2\in C\\w_1\neq w_2}} \lvert\{j\in X \mid (w_1)_j \neq (w_2)_j\}\rvert
    =
    \inf_{w\in C\setminus\{0\}} \lvert w^{-1}(1)\rvert,
  \end{equation*}
  where $w^{-1}(1) = \{j\in X \mid w_j = 1\}$ is the support of $w$ and the equality follows from the fact that $C$ is an
  $\FF_2$-linear subspace (so $w_1 - w_2\in C$ whenever $w_1,w_2\in C$ and $0\in C$).

  We will be particularly interested in the cases when $X$ is either $[n]^k$ or $[m]^k$ and in the distance induced by a
  ``structured projection'' operation that relates the two cases, which is not typical of coding theory. Namely, for
  $\gamma=(\gamma_i)_{i\in[k]}$ with $\gamma_i\colon [m]\to[n]$, we define a function $\gamma^*\colon\FF_2^{[n]^k}\to
  \FF_2^{[m]^k}$ by
  \begin{equation*}
    \gamma^*(w)_\beta \df w_{\gamma_{\#}(\beta)},
  \end{equation*}
  where $\gamma_{\#}\colon [m]^k\to[n]^k$ is the ``product'' function given by
  $\gamma_{\#}(\beta)_i\df\gamma_i(\beta_i)$. Clearly, $\gamma^*$ is a linear map. For a linear code $C\subseteq\FF_2^{[n]^k}$,
  define
  \begin{align*}
    \dist_\gamma(C)
    & \df
    \inf_{\substack{w_1,w_2\in C\\w_1\neq w_2}} \lvert\{j\in [m]^k \mid \gamma^*(w_1)_j\neq \gamma^*(w_2)_j\}\rvert
    \\
    & =
    \inf_{w\in C\setminus\{0\}} \lvert\gamma^*(w)^{-1}(1)\rvert
    \\
    & =
    \begin{dcases*}
      \dist(\gamma^*(C)), & if $\gamma^*$ is injective on $C$,\\
      0, & otherwise.
    \end{dcases*}
  \end{align*}
  Again, the first equality follows since $C$ is linear; the second equality follows since $\gamma^*$ is a linear transformation
  (which also means that saying $\gamma^*$ is injective on $C$ is equivalent to saying that its kernel has trivial intersection
  with $C$).

  Our goal is to find a linear code $C\subseteq\FF_2^{[n]^k}$ of dimension $d\df\ceil{\rho\cdot m^k}$ such that for most
  $\gamma$, we have $\dist_\gamma(C) > \epsilon\cdot m^k/s(\ell)$. In fact, we will prove that a uniformly random linear code of
  dimension $d$ satisfies this property with positive probability:

  \begin{claim}\label{clm:linearcode:simplified}
    There exists a linear code $C\subseteq\FF_2^{[n]^k}$ of dimension $d\df\ceil{\rho\cdot m^k}$ such that if
    $\rn{\gamma}_1,\ldots,\rn{\gamma}_k$ are i.i.d.\ with each $\rn{\gamma}_i$ uniformly distributed in $[n]^m$, then
    \begin{equation*}
      \PP_{\rn{\gamma}}\left[
        \dist_{\rn{\gamma}}(C) > \epsilon\cdot\frac{m^k}{s(\ell)}
        \right]
      >
      \delta.
    \end{equation*}
  \end{claim}

  Before we find such a linear code, let us see why its existence yields the result. First note that since $\cH_z$
  Natarajan-shatters $[n]^k$, there cannot be repetitions among the variables of $z$ corresponding to the same part, that is,
  recalling that $z\in\cE_n(\Omega) = \prod_{i=1}^k \Omega_i^n$, if $z=(z_1,\ldots,z_k)$, then each $z_i$ has all of its
  coordinates distinct (we claim nothing about how coordinates of some $z_i$ relate to coordinates of a $z_j$ with $i\neq j$).

  Define $\mu\in\Pr(\Omega)$ by letting $\mu_i$ be the uniform measure on the (exactly $n$) points of $\Omega_i$ that are the
  coordinates of $z_i$ and let $C\subseteq\FF_2^{[n]^k}$ be given by Claim~\ref{clm:linearcode:simplified} and enumerate its
  elements as $C=\{w_1,\ldots,w_t\}$, where $t\df\lvert C\rvert = 2^d\geq 2^{\rho\cdot m^k}$.

  Note that if we show that
  \begin{equation*}
    \PP_{\rn{x}\sim\mu^m}[(H_{w_1},\ldots,H_{w_t})\text{ is $\epsilon$-separated on $\rn{x}$ w.r.t.\ $\ell$}]
    >
    \delta,
  \end{equation*}
  then the proof is concluded as this is a contradiction with the probabilistic Haussler packing property guarantee as $m\geq
  m^{\hPHP[m^k]}_{\cH,\ell}(\epsilon,\delta,\rho)$.

  But indeed, for each $i\in[k]$ define the random element $\rn{\gamma}_i$ of $[n]^m$ by letting $\rn{\gamma}_i$ be the unique
  function $[m]\to[n]$ such that
  \begin{equation*}
    (\rn{x}_i)_j = (z_i)_{\rn{\gamma}_i(j)}
  \end{equation*}
  and note that since $\mu_i$ is the uniform distribution on the coordinates of $z_i$, it follows that $\rn{\gamma}_i$ is
  uniformly distributed on $[n]^m$. It is also clear that the $\rn{\gamma}_i$ are mutually independent.

  Claim~\ref{clm:linearcode:simplified} then says that with probability greater than $\delta$, we have
  \begin{equation}\label{eq:PHP->VCNkk:simplified:dist:partite}
    \dist_{\rn{\gamma}}(C) > \epsilon\cdot\frac{m^k}{s(\ell)}.
  \end{equation}

  But note that
  \begin{align*}
    \dist_{\rn{\gamma}}(C)
    & =
    \inf_{\substack{w,w'\in C\\w\neq w'}}
    \lvert\{j\in[m]^k \mid \rn{\gamma}^*(w)_j\neq \rn{\gamma}^*(w')_j\}\rvert
    \\
    & =
    \inf_{\substack{w,w'\in C\\w\neq w'}}
    \lvert\{\beta\in[m]^k \mid (H_w)^*_m(\rn{x})_\beta \neq (H_{w'})^*_m(\rn{x})_\beta\}\rvert
    \\
    & \leq
    \frac{m^k}{s(\ell)}
    \inf_{1\leq i < j\leq t}
    L_{\rn{x},(H_{w_i})^*_m(\rn{x}),\ell}(H_{w_j})\cdot m^k,
  \end{align*}
  so~\eqref{eq:PHP->VCNkk:simplified:dist:partite} implies that $(H_{w_1},\ldots,H_{w_t})$ is $\epsilon$-separated on $\rn{x}$
  w.r.t.\ $\ell$.

  It remains then to prove Claim~\ref{clm:linearcode:simplified}:

  \begin{proofof}{Claim~\ref{clm:linearcode:simplified} (sketch)}
    Again, here, we only prove Claim~\ref{clm:linearcode:simplified} for $\epsilon$, $\delta$ and $\rho$ small enough, $m$ large
    enough and $n$ large enough and we defer the proof of the claim with the precise bounds to Appendix~\ref{sec:calc} (see
    Claim~\ref{clm:linearcode:partite}).

    Let $\rn{A}$ be a random $[n]^k\times[d]$-matrix with entries in $\FF_2$, picked uniformly at random (i.e., a uniformly at
    random element of $\FF_2^{[n]^k\times[d]}$) and let $\rn{C}\df\im(\rn{A})$ be the image of $\rn{A}$, which is clearly a
    (random) linear subspace of $\FF_2^{[n]^k}$ of dimension at most $d$.

    In fact, we can compute exactly the probability that the dimension of $\rn{C}$ is $d$ by simply counting in how many ways we
    can generate each row of $\rn{A}$ to not be in the span of the previous rows:
    \begin{equation*}
      \PP_{\rn{C}}[\dim_{\FF_2}(\rn{C}) = d]
      =
      2^{-d\cdot n^k}\prod_{j=0}^{d-1} (2^{n^k} - 2^j)
      =
      \prod_{j=0}^{d-1} (1 - 2^{j-n^k})
      \geq
      (1-2^{d-n^k})^d,
    \end{equation*}
    where the inequality follows assuming $d\df\ceil{\rho\cdot m^k}\leq n^k$. Note that if $n$ is large enough in terms of $m$,
    then the above probability can be made as close to $1$ as needed, i.e., $\rn{C}$ has dimension exactly $d$ asymptotically
    almost surely.

    To prove the existence of the desired linear code, it then suffices to show that the probability
    \begin{equation*}
      \PP_{\rn{C}}\left[
        \PP_{\rn{\gamma}}\left[
          \dist_{\rn{\gamma}}(\rn{C}) > \frac{\epsilon\cdot m^k}{s(\ell)}
          \right]
        >
        \delta
        \right].
    \end{equation*}
    is bounded away from $0$ as $n\to\infty$, that is, it is at least a constant $K > 0$ to be picked later that does not depend
    on any of $\epsilon$, $\delta$, $\rho$, $m$ or $n$. Since the inner probability is at most $1$, by (reverse) Markov's
    Inequality, it suffices to show the following bound on expectation:
    \begin{equation}\label{eq:linearcodepartite:goal}
      \EE_{\rn{C}}\left[
        \PP_{\rn{\gamma}}\left[
          \dist_{\rn{\gamma}}(\rn{C}) > \frac{\epsilon\cdot m^k}{s(\ell)}
          \right]
        >
        \delta
        \right]
      >
      1 - (1-\delta)\cdot(1-K).
    \end{equation}

    For each $i\in[k]$, let $E_i(\rn{\gamma}_i)$ be the event that $\rn{\gamma}_i$ has no repeated values (i.e., $\rn{\gamma}_i$
    is injective) and let $E(\rn{\gamma})$ be the conjunction of the $E_i(\rn{\gamma}_i)$. Note that
    \begin{align*}
      \PP_{\rn{\gamma}}\bigl[E(\rn{\gamma})\bigr]
      & =
      \prod_{i=1}^k \PP_{\rn{\gamma}_i}\bigl[E_i(\rn{\gamma}_i)\bigr]
      =
      \left(\frac{(n)_m}{n^m}\right)^k
      \\
      & \geq
      \left(1 - \frac{m}{n}\right)^{k\cdot m}
      >
      \left(1 - \frac{1}{m}\right)^{k\cdot m},
    \end{align*}
    where the last inequality follows assuming $n > m^2 > 0$. Note that as $m\to\infty$, the above converges to $e^{-k}$, so we
    can assume that $m$ is large enough so that the above is at least $e^{-k}/2$.

    Then the left-hand side of our goal in~\eqref{eq:linearcodepartite:goal} can be bounded as:
    \begin{align*}
      \EE_{\rn{C}}\left[
        \PP_{\rn{\gamma}}\left[
          \dist_{\rn{\gamma}}(\rn{C}) > \frac{\epsilon\cdot m^k}{s(\ell)}
          \right]
        \right]
      & =
      \EE_{\rn{\gamma}}\left[
        \EE_{\rn{C}}\left[
          \One\left[
            \dist_{\rn{\gamma}}(\rn{C}) > \frac{\epsilon\cdot m^k}{s(\ell)}
            \right]
          \right]
        \right]
      \\
      & >
      \frac{e^{-k}}{2}\cdot
      \EE_{\rn{\gamma}}\left[
        \EE_{\rn{C}}\left[
          \One\left[
            \dist_{\rn{\gamma}}(\rn{C}) > \frac{\epsilon\cdot m^k}{s(\ell)}
            \right]
          \right]
        \Given
        E(\rn{\gamma})
        \right].
    \end{align*}

    Thus, it suffices to show that for every fixed $\gamma$ in the event $E(\gamma)$, we have
    \begin{equation*}
      \PP_{\rn{C}}\left[
        \dist_\gamma(\rn{C}) > \frac{\epsilon\cdot m^k}{s(\ell)}
        \right]
      \geq
      \frac{2}{e^{-k}}\cdot\bigl(1 - (1 - \delta)\cdot(1-K)\bigr)
    \end{equation*}
    which in turn is equivalent to
    \begin{equation*}
      \PP_{\rn{C}}\left[
        \dist_\gamma(\rn{C}) \leq \frac{\epsilon\cdot m^k}{s(\ell)}
        \right]
      \leq
      1 - \frac{2}{e^{-k}}\cdot\bigl(1 - (1 - \delta)\cdot(1-K)\bigr).
    \end{equation*}

    From the definition of $\rn{C}$, we know that the set $\rn{C}\setminus\{0\}$ is a subset\footnote{The only reason we say
    subset instead of equality is because we are \emph{not} restricting to the event in which $\rn{A}$ is full rank, so the set
    above might potentially have $0$.} of
    \begin{equation*}
      \{\rn{A}(z) \mid z\in\FF_2^{[d]}\setminus\{0\}\}.
    \end{equation*}

    By the union bound, it then suffices to show that for every $z\in\FF_2^{[d]}\setminus\{0\}$, we have\footnote{It would have
    been fine to put $2^d-1$ instead of $2^d$ in the denominator, but this leads to a slightly cleaner expression.}
    \begin{equation*}
      \PP_{\rn{A}}\left[
        \Bigl\lvert\gamma^*\bigl(\rn{A}(z)\bigr)^{-1}(1)\Bigr\rvert
        \leq
        \frac{\epsilon\cdot m^k}{s(\ell)}
        \right]
      \leq
      \frac{1}{2^d}\cdot\left(1 - \frac{2}{e^{-k}}\cdot\bigl(1 - (1 - \delta)\cdot(1-K)\bigr)\right).
    \end{equation*}

    Since $\rn{A}$ is picked uniformly at random in $\FF_2^{[n]^k\times[d]}$, for each fixed $z\in\FF_2^{[d]}\setminus\{0\}$, we
    know that $\rn{A}(z)$ is uniformly distributed on $\FF_2^{[n]^k}$, so the above is equivalent to
    \begin{equation}\label{eq:linearcodepartite:goalw}
      \PP_{\rn{w}}\left[
        \lvert\gamma^*(\rn{w})^{-1}(1)\rvert
        \leq
        \frac{\epsilon\cdot m^k}{s(\ell)}
        \right]
      \leq
      \frac{1}{2^d}\cdot\left(1 - \frac{2}{e^{-k}}\cdot\bigl(1 - (1 - \delta)\cdot(1-K)\bigr)\right),
    \end{equation}
    where $\rn{w}$ is picked uniformly at random in $\FF_2^{[n]^k}$.

    Since $\gamma$ is in the event $E(\gamma)$, it follows that the projection $\gamma^*$ is full rank; this means that the
    probability above is straightforward to compute: by counting how many ways $\rn{w}$ can project into a ball of radius
    $\epsilon\cdot m^k/s(\ell)$ around the origin (in $\FF_2^{[m]^k}$) and measuring the size of the kernel of $\gamma^*$; in
    formulas:
    \begin{align*}
      \PP_{\rn{w}}\left[
        \bigl\lvert\gamma^*(\rn{w})^{-1}(1)\bigr\rvert
        \leq
        \frac{\epsilon\cdot m^k}{s(\ell)}
        \right]
      & =
      \frac{1}{2^{n^k}}\cdot\left(\sum_{j=0}^{\floor{\epsilon\cdot m^k/s(\ell)}}\binom{m^k}{j}\right)\cdot 2^{n^k-m^k}
      \\
      & \leq
      2^{(h_2(\epsilon/s(\ell))-1)\cdot m^k},
    \end{align*}
    where the inequality is the standard upper bound on the size of the Hamming ball in terms of the binary entropy (see
    e.g.~\cite[Lemma~4.7.2]{Ash65}), by assuming that $\epsilon/s(\ell)\in(0,1/2)$ as $\epsilon$ is small.

    Thus, to get~\eqref{eq:linearcodepartite:goalw}, we need that
    \begin{equation*}
      2^{(h_2(\epsilon/s(\ell))-1)\cdot m^k}
      \leq
      \frac{1}{2^d}\cdot\left(1 - \frac{2}{e^{-k}}\cdot\bigl(1 - (1 - \delta)\cdot(1-K)\bigr)\right).
    \end{equation*}
    Recalling that $d\df\ceil{\rho\cdot m^k}\leq\rho\cdot m^k + 1$, it suffices to show
    \begin{equation*}
      2^{(h_2(\epsilon/s(\ell))-1 + \rho)\cdot m^k + 1}
      \leq
      \left(1 - \frac{2}{e^{-k}}\cdot\bigl(1 - (1 - \delta)\cdot(1-K)\bigr)\right).
    \end{equation*}

    If we assume that $\rho$ is smaller than $1 - h_2(\epsilon/s(\ell))$, then the coefficient of $m^k$ in the above is
    negative, so as $m\to\infty$ (with $n$ large enough in terms of $m$ as specified before), the above bound converges to $0$,
    in particular, some $m$ is large enough to yield the existence of the desired linear code $C$.
  \end{proofof}
  This concludes the proof of the proposition.
\end{proof}

\section{Proof of the main theorems}

In this section, we put together the results of the previous sections to prove our main theorems of Section~\ref{sec:main}
(which are restated below for convenience).

\thmSCpart*

\begin{proof}
  It is clear that $\ell^{\ag}$ is local and bounded.

  The implication~\ref{thm:SCpart:VCNkk}$\implies$\ref{thm:SCpart:SUC} is Proposition~\ref{prop:VCNkk->SUC}.

  The implication~\ref{thm:SCpart:SUC}$\implies$\ref{thm:SCpart:advSC} is Proposition~\ref{prop:SUC->advSC}.

  The implication~\ref{thm:SCpart:advSC}$\implies$\ref{thm:SCpart:SC} follows by conditioning on the outcome of
  $\rn{x}\sim\mu^m$, see Remark~\ref{rmk:advSC->agSC->SC}.

  The implication~\ref{thm:SCpart:SC}$\implies$\ref{thm:SCpart:PHP} is Proposition~\ref{prop:SC->PHP}.

  The implication~\ref{thm:SCpart:PHP}$\implies$\ref{thm:SCpart:VCNkk} is Proposition~\ref{prop:PHP->VCNkk}.

  The implication~\ref{thm:SCpart:VCNkk}$\implies$\ref{thm:SCpart:SHPbootstrap} is Proposition~\ref{prop:VCNkk->SHP}.

  Finally, the implications~\ref{thm:SCpart:SHPbootstrap}$\implies$\ref{thm:SCpart:SHP}
  and~\ref{thm:SCpart:SHP}$\implies$\ref{thm:SCpart:PHP} are trivial (see Remark~\ref{rmk:SHP->PHP}).
\end{proof}

\thmSC*

\begin{proof}
  It is clear that $\ell^{\kpart}$ is separated and bounded and that $\ell^{\ag}$ is symmetric, local, separated and bounded.

  The equivalence between items~\ref{thm:SC:VCNkk} and~\ref{thm:SC:VCNkkpart} is Proposition~\ref{prop:VCNkk}.

  The equivalence between all items involving the partization $\cH^{\kpart}$ (i.e., items~\ref{thm:SC:VCNkkpart},
  \ref{thm:SC:SUCpart}, \ref{thm:SC:advSCpart}, \ref{thm:SC:SCpart}, \ref{thm:SC:SHPpart}, \ref{thm:SC:SHPbootstrappart}
  and~\ref{thm:SC:PHPpart}) is Theorem~\ref{thm:SCpart}.

  The implication~\ref{thm:SC:VCNkk}$\implies$\ref{thm:SC:SUC} is Proposition~\ref{prop:VCNkk->SUC}.

  The implication~\ref{thm:SC:SUC}$\implies$\ref{thm:SC:advSCsymm} is Proposition~\ref{prop:SUC->advSC}.

  The implications~\ref{thm:SC:advSCsymm}$\implies$\ref{thm:SC:SCsymm}
  and~\ref{thm:SC:advSC}$\implies$\ref{thm:SC:SC} follow by conditioning on the outcome of $\rn{x}\sim\mu^m$, see
  Remark~\ref{rmk:advSC->agSC->SC}.

  The implications~\ref{thm:SC:advSCsymm}$\implies$\ref{thm:SC:advSC}
  and~\ref{thm:SC:SCsymm}$\implies$\ref{thm:SC:SC} follow from Remark~\ref{rmk:symm->nonsymm}.

  The implication~\ref{thm:SC:SC}$\implies$\ref{thm:SC:PHP} is Proposition~\ref{prop:SC->PHP}.

  The implication~\ref{thm:SC:PHP}$\implies$\ref{thm:SC:VCNkk} is Proposition~\ref{prop:PHP->VCNkk}.

  Finally, the implications~\ref{thm:SC:SHPbootstrap}$\implies$\ref{thm:SC:SHP} and~\ref{thm:SC:SHP}$\implies$\ref{thm:SC:PHP}
  are trivial (see Remark~\ref{rmk:SHP->PHP}).
\end{proof}

\printbibliography

\appendix

\section{Included proofs from the literature}
\label{sec:lit}

In this section we collect some short proofs of classic results of the literature that are relevant to the current paper. Some
results here are marginal improvements over their literature counterparts.

\SSPlemma*

\begin{proof}
  We want to show that if $V$ is a set of size $m$, then
  \begin{equation*}
    \lvert\cF_V\rvert \leq (m+1)^{\Nat(\cF)}\cdot\binom{\lvert Y\rvert}{2}^{\Nat(\cF)}.
  \end{equation*}
  We prove this by induction in $m$. The result clearly holds when $m\leq 1$ and when $\Nat(\cF)=0$ (as it forces
  $\lvert\cF_V\rvert\leq 1$), so we suppose $m\geq 2$ and $\Nat(\cF_V)\geq 1$.

  Let $v\in V$, let $U\df V\setminus\{v\}$ and for every $y\in Y$ and every $F\colon U\to Y$, let $F_y\colon V\to Y$ be the
  unique extension of $F$ that maps $v$ to $y$. For every $\{y_0,y_1\}\subseteq\binom{Y}{2}$, we also let
  \begin{equation*}
    \cF^{\{y_0,y_1\}} \df \{F\colon U\to Y \mid F_{y_0},F_{y_1}\in \cF_V\}
  \end{equation*}
  and we note that
  \begin{equation}\label{eq:cFV}
    \lvert\cF_V\rvert \leq \lvert\cF_U\rvert + \sum_{\{y_0,y_1\}\in\binom{Y}{2}}\lvert\cF^{\{y_0,y_1\}}\rvert.
  \end{equation}

  Clearly $\Nat(\cF_U)\leq\Nat(\cF_V)\leq\Nat(\cF)$. For the other families, we claim that
  $\Nat(\cF^{\{y_0,y_1\}})\leq\Nat(\cF_V)-1\leq\Nat(\cF)-1$ for every $\{y_0,y_1\}\in\binom{Y}{2}$. Indeed, if
  $\cF^{\{y_0,y_1\}}$ shatters $A\subseteq U$, then it is clear that $\cF_V$ shatters $A\cup\{v\}$ as each function $U\to Y$ of
  $\cF^{\{y_0,y_1\}}$ can be extended to a function $V\to Y$ in the two ways that map $v$ to $y_0$ and to $y_1$. Using this and
  inductive hypothesis on~\eqref{eq:cFV}, we get
  \begin{align*}
    \lvert\cF_V\rvert
    & \leq
    m^{\Nat(\cF)}\cdot\binom{\lvert Y\rvert}{2}^{\Nat(\cF)}
    + \binom{\lvert Y\rvert}{2}\cdot m^{\Nat(\cF)-1}\cdot\binom{\lvert Y\rvert}{2}^{\Nat(\cF)-1}
    \\
    & =
    (m+1)\cdot m^{\Nat(\cF)-1}\cdot\binom{\lvert Y\rvert}{2}^{\Nat(\cF)}
    \leq
    (m+1)^{\Nat(\cF)}\cdot\binom{\lvert Y\rvert}{2}^{\Nat(\cF)},
  \end{align*}
  as desired.
\end{proof}

\begin{definition}
  For a $k$-partite $k$-hypergraph $G$ and $v\in V(G)$, the \emph{neighborhood} of $v$ in $G$ is the set $N_G(v)$ of
  $(k-1)$-tuples which along with $v$ are in $E(G)$. In a formula, this is a bit awkward to define as $v$ could be in any
  $V_i(G)$:
  \begin{equation*}
    N_G(v)
    \df
    \left\{(w_1,\ldots,w_{k-1})\in \prod_{j\in[k]\setminus\{i_v\}} V_j(G)
    \;\middle\vert\;
    (w_1,\ldots,w_{i_v-1},v,w_{i_v+1},\ldots,w_{k-1})\in E(G)
    \right\},
  \end{equation*}
  where $i_v$ is the unique element of $[k]$ such that $v\in V_{i_v}(G)$. The \emph{degree} of $v$ is $d_G(v)\df\lvert
  N_G(v)\rvert\leq\prod_{j\in[k]\setminus\{i_v\}} v_j(G)$.
\end{definition}

\begin{theorem}[\Kovari--\Sos--\Turan~\protect{\cite{KST54}}]\label{thm:KST}
  For every $n\in\NN$ and $t\in\NN_+$, we have
  \begin{equation}
    \ex_{\kpart[2]}(n,K_{t,t}) \leq (t-1)^{1/t}\cdot n^{2-1/t} + (t-1)\cdot n.
    \label{eq:KST}
  \end{equation}
\end{theorem}

\begin{proof}
  If $t=1$, then it is clear that $\ex_{\kpart[2]}(n,K_{1,1})=0$ and that the bound in~\eqref{eq:KST} amounts to $0$, so suppose
  $t\geq 2$. If $n\leq t$, then the bound in~\eqref{eq:KST} is trivial as it is at least $n(n-1)$, so we may suppose that $n\geq
  t+1$.

  Consider the function $g\colon\RR_{\geq 0}\to\RR$ given by
  \begin{equation*}
    g(x) \df
    \begin{dcases*}
      \binom{x}{t}, & if $x\geq t$,\\
      0, & otherwise.
    \end{dcases*}
  \end{equation*}
  (Here the binomial is defined in terms of the falling factorial: $\binom{x}{t}\df (x)_t/t! = x(x-1)\cdots(x-t+1)/t!$). It is
  straightforward to check that $g$ is a convex function that matches the binomial $\binom{x}{t}$ whenever $x$ is an integer.

  Suppose $G$ is a $2$-partite graph without any copies of $K_{t,t}$ and $n$ vertices on each side. Since $n\geq t\geq 2$, the
  bound in~\eqref{eq:KST} is at least $tn$, so we may suppose that $e(G)\geq tn$. Since $G$ has no copies of $K_{t,t}$, we know
  for every $U\in\binom{V_2(G)}{t}$, we must have $\lvert\bigcap_{v\in U} N_G(v)\rvert\leq t-1$, so we get
  \begin{align*}
    (t-1)\binom{n}{t}
    & \geq
    \sum_{U\in\binom{V_2(G)}{t}} \left\lvert\bigcap_{v\in U} N_G(v)\right\rvert
    =
    \sum_{v\in V_1(G)} \binom{d_G(v)}{t}
    \\
    & \geq
    n\cdot\binom{n^{-1}\sum_{v\in V_1(G)} d_G(v)}{t}
    =
    n\cdot\binom{e(G)/n}{t},
  \end{align*}
  where the second inequality is Jensen's Inequality for the function $g$ (and uses the fact that $e(G)/n\geq t$).

  Thus, we conclude that $(t-1)\cdot (n)_t \geq n\cdot (e(G)/n)_t$, which in particular implies that
  \begin{equation*}
    n\cdot \left(\frac{e(G)}{n} - t + 1\right)^t \leq (t-1) n^t,
  \end{equation*}
  from which~\eqref{eq:KST} follows.
\end{proof}

\begin{lemma}[\Erdos~\protect{\cite{Erd64}}]\label{lem:Erdos}
  If $W$ is a finite set, $A_1,\ldots,A_n\subseteq W$ and $t\in[n]$, then there exists $I\in\binom{[n]}{t}$ such that
  \begin{equation*}
    \left\lvert\bigcap_{i\in I} A_i\right\rvert
    \geq
    \frac{1}{(n)_t\cdot\lvert W\rvert^{t-1}}\cdot\left(\sum_{i=1}^n \lvert A_i\rvert\right)^t
    - \left(\frac{n^t}{(n)_t} - 1\right)\cdot\max_{i\in[n]}\lvert A_i\rvert.
  \end{equation*}
\end{lemma}

\begin{proof}
  By Jensen's Inequality, we have
  \begin{align*}
    \frac{1}{\lvert W\rvert^{t-1}}\cdot\left(\sum_{i=1}^n\lvert A_i\rvert\right)^t
    & =
    \lvert W\rvert\cdot\left(
    \frac{1}{\lvert W\rvert}
    \sum_{w\in W} \sum_{i\in[n]}\One_{A_i}(w)
    \right)^t
    \leq
    \sum_{w\in W}\left(\sum_{i\in[n]}\One_{A_i}(w)\right)^t
    \\
    & =
    \sum_{w\in W} \sum_{i\in[n]^t} \prod_{j=1}^t \One_{A_{i(j)}}(w)
    =
    \sum_{i\in[n]^t} \left\lvert\bigcap_{j=1}^t A_{i(j)}\right\rvert
    \\
    & \leq
    \sum_{i\in([n])_t} \left\lvert\bigcap_{j=1}^t A_{i(j)}\right\rvert
    + (n^t - (n)_t)\max_{i\in[n]}\lvert A_i\rvert,
  \end{align*}
  where the last inequality follows by noting that there are $n^t - (n)_t$ terms corresponding to non-injective $i\colon[t]\to
  [n]$ and each term can be bounded by $\max_{i\in[n]}\lvert A_i\rvert$.

  By grouping terms according to $\im(i)$, we conclude that
  \begin{equation*}
    \sum_{I\in\binom{[n]}{t}} \left\lvert\bigcap_{i\in I} A_i\right\rvert
    \geq
    \frac{1}{t!\cdot\lvert W\rvert^{t-1}}\cdot\left(\sum_{i=1}^n\lvert A_i\rvert\right)^t
    - \frac{n^t - (n)_t}{t!}\max_{i\in[n]}\lvert A_i\rvert,
  \end{equation*}
  so there must exist $I\in\binom{[n]}{t}$ such that $\bigcap_{i\in I} A_i$ has size that is at least a $\binom{n}{t}$ fraction
  of the value above, that is, we get
  \begin{equation*}
    \left\lvert\bigcap_{i\in I} A_i\right\rvert
    \geq
    \frac{1}{(n)_t\cdot\lvert W\rvert^{t-1}}\cdot\left(\sum_{i=1}^n\lvert A_i\rvert\right)^t
    - \left(\frac{n^t}{(n)_t} - 1\right)\cdot\max_{i\in[n]}\lvert A_i\rvert,
  \end{equation*}
  as desired.
\end{proof}

\begin{theorem}[\Erdos, partite version of~\protect{\cite[Theorem~1]{Erd64}}]\label{thm:Erdos}
  For every $n,k,t\in\NN_+$ with $k\geq 2$ and $t\leq n$, we have
  \begin{equation}\label{eq:Erdosrec}
    \begin{aligned}
      \ex_{\kpart}(n,K_{t,\ldots,t}^{(k)})
      & \leq
      \left(\left(
      \ex_{\kpart[(k-1)]}(n,K_{t,\ldots,t}^{(k-1)}) + \left(\frac{n^t}{(n)_t} - 1\right) n^{k-1}
      \right)
      \cdot(n)_t\cdot n^{(k-1)(t-1)}
      \right)^{1/t}
      \\
      & \leq
      \left(\ex_{\kpart[(k-1)]}(n,K_{t,\ldots,t}^{(k-1)}) + t\cdot (t-1)\cdot n^{k-2}\right)^{1/t}\cdot n^{k-(k-1)/t},
    \end{aligned}
  \end{equation}
\end{theorem}

\begin{proof}
  First note that since $1\leq t\leq n$, we have
  \begin{equation*}
    (n-t+1)^t \leq (n)_t \leq n^t,
  \end{equation*}
  so we get
  \begin{equation*}
    \frac{n^t}{(n)_t}
    \geq
    \left(1 + \frac{t-1}{n-t+1}\right)^t
    \geq
    1 + \frac{t(t-1)}{n-t+1}
    \geq
    1 + \frac{t(t-1)}{n}.
  \end{equation*}
  These derivations along with a straightforward computation explain the second inequality in~\eqref{eq:Erdosrec}.

  \medskip

  Let us prove the first inequality in~\eqref{eq:Erdosrec}. Let $G$ be a $k$-partite $k$-hypergraph of size $n$ without any
  copies of $K_{t,\ldots,t}^{(k)}$ and consider the sequence of neighborhoods $(N_G(v))_{v\in V_k(G)}$. These are $n$ subsets of
  $\prod_{i=1}^{k-1} V_i(G)$, hence of size at most $n^{k-1}$ each, so by Lemma~\ref{lem:Erdos}, there exists
  $U\in\binom{V_k(G)}{t}$ such that
  \begin{align*}
    \left\lvert\bigcap_{v\in U} N_G(v)\right\rvert
    & \geq
    \frac{1}{(n)_t\cdot n^{(k-1)(t-1)}}\cdot\left(\sum_{v\in V_k(G)} d_G(v)\right)^t
    - \left(\frac{n^t}{(n)_t} - 1\right)\cdot n^{k-1}
    \\
    & =
    \frac{e(G)^t}{(n)_t\cdot n^{(k-1)(t-1)}}
    - \left(\frac{n^t}{(n)_t} - 1\right)\cdot n^{k-1},
  \end{align*}
  from which we conclude that
  \begin{equation*}
    e(G)
    \leq
    \left(\left(
    \left\lvert\bigcap_{v\in U} N_G(v)\right\rvert
    + \left(\frac{n^t}{(n)_t} - 1\right)\cdot n^{k-1}
    \right)\cdot (n)_t\cdot n^{(k-1)(t-1)}
    \right)^{1/t}.
  \end{equation*}

  Let $H$ be the $(k-1)$-partite $(k-1)$-hypergraph with vertex sets $V_i(H)\df V_i(G)$ ($i\in[k-1]$) and edge set
  $E(H)\df\bigcap_{v\in U} N_G(v)$. Since $G$ has no copies of $K_{t,\ldots,t}^{(k)}$, it follows that $H$ has no copies of
  $K_{t,\ldots,t}^{(k-1)}$ (as any such copy along with $U$ would form a copy of $K_{t,\ldots,t}^{(k)}$ in $G$), so we must have
  \begin{equation*}
    \left\lvert\bigcap_{v\in U} N_G(v)\right\rvert
    =
    e(H)
    \leq
    \ex_{\kpart[(k-1)]}(n,K_{t,\ldots,t}^{(k-1)})
  \end{equation*}
  and the first inequality in~\eqref{eq:Erdosrec} follows.
\end{proof}

\begin{theorem}[\Kovari--\Sos--\Turan~\protect{\cite{KST54}}, \Erdos, partite version of~\protect{\cite[Theorem~1]{Erd64}}]
  \label{thm:Erdosasymp}
  For every $n\in\NN$ and every $k,t\in\NN_+$ with $k\geq 2$, we have
  \begin{equation}\label{eq:Erdosasymp}
    \begin{aligned}
      \ex_{\kpart}(n,K_{t,\ldots,t}^{(k)})
      & \leq
      \begin{multlined}[t]
        (t-1)^{1/t^{k-1}}\cdot n^{k-1/t^{k-1}}
        + (t-1)^{1/t^{k-2}}\cdot n^{k-1/t^{k-2}}
        \\
        + \sum_{j=3}^k (t\cdot (t-1))^{1/t^{k-j+1}}\cdot n^{k-1/t^{k-j+1}}
      \end{multlined}
      \\
      & \leq
      \bigl(c_{t,k} + o(1)\bigr)\cdot n^{k-1/t^{k-1}},
    \end{aligned}
  \end{equation}
  where
  \begin{equation*}
    c_{t,k} \df (t-1)^{1/t^{k-1}} < 1.5.
  \end{equation*}
\end{theorem}

\begin{proof}
  We prove~\eqref{eq:Erdosasymp} by induction in $k$. For $k=2$, this reduces to
  \begin{equation*}
    \ex_{\kpart[2]}(n,K_{t,t})
    \leq
    (t-1)^{1/t}\cdot n^{2-1/t}
    + (t-1)\cdot n,
  \end{equation*}
  which is precisely~\eqref{eq:KST} in Theorem~\ref{thm:KST}.

  For $k\geq 3$, first note that if $n\leq t-1$, then we clearly have
  \begin{equation*}
    \ex_{\kpart}(n,K_{t,\ldots,t}^{(k)})
    =
    n^k
  \end{equation*}
  and the right-hand side of~\eqref{eq:Erdosasymp} is clearly at least $n^k$; in fact each of the $k$ terms that are added
  together on the right-hand side is at least $n^k$ when $n\leq t-1$.

  Suppose then that $n\geq t$ so that we that by~\eqref{eq:Erdosrec} in Theorem~\ref{thm:Erdos}, we get
  \begin{align*}
    \ex_{\kpart}(n,K_{t,\ldots,t}^{(k)})
    & \leq
    \left(\ex_{\kpart[(k-1)]}(n,K_{t,\ldots,t}^{(k-1)}) + t\cdot (t-1)\cdot n^{k-2}\right)^{1/t}\cdot n^{k-(k-1)/t}
    \\
    & \leq
    \begin{multlined}[t]
      \Biggl(
      (t-1)^{1/t^{k-2}}\cdot n^{k-1-1/t^{k-2}}
      + (t-1)^{1/t^{k-3}}\cdot n^{k-1-1/t^{k-3}}
      \\
      + \sum_{j=3}^{k-1} (t\cdot (t-1))^{1/t^{k-j}}\cdot n^{k-1-1/t^{k-j}}
      + t\cdot (t-1)\cdot n^{k-2}
      \Biggr)^{1/t}
      \cdot n^{k-(k-1)/t}
    \end{multlined}
    \\
    & \leq
    \begin{multlined}[t]
      (t-1)^{1/t^{k-1}}\cdot n^{k-1/t^{k-1}}
      + (t-1)^{1/t^{k-2}}\cdot n^{k-1-1/t^{k-2}}
      \\
      + \sum_{j=3}^{k-1} (t\cdot (t-1))^{1/t^{k-j+1}}\cdot n^{k-1/t^{k-j+1}}
      + (t\cdot (t-1))^{1/t}\cdot n^{k-1/t}
    \end{multlined}
    \\
    & =
    \begin{multlined}[t]
      (t-1)^{1/t^{k-1}}\cdot n^{k-1/t^{k-1}}
      + (t-1)^{1/t^{k-2}}\cdot n^{k-1/t^{k-2}}
      \\
      + \sum_{j=3}^k (t\cdot (t-1))^{1/t^{k-j+1}}\cdot n^{k-1/t^{k-j+1}},
    \end{multlined}
  \end{align*}
  where the third inequality follows from
  \begin{equation*}
    \left(\sum_{i=1}^u a_i\right)^{1/t} \leq \sum_{i=1}^u a_i^{1/t}
  \end{equation*}
  whenever $a_i\geq 0$.
\end{proof}

\KSTErdospartite*

\begin{proof}
  The case $k=1$ is trivial as
  \begin{equation}\label{eq:exkpart1}
    \ex_{\kpart[1]}(n,K_t^{(1)}) = \min\{n,t-1\}\leq t-1.
  \end{equation}

  If $k\geq 3$, then~\eqref{eq:KSTErdos} follows from~\eqref{eq:Erdosasymp} in Theorem~\ref{thm:Erdosasymp} by noting that each
  of the $k$ terms has a coefficient that is at most $2$.

  For the case $k=2$, the bound in~\eqref{eq:Erdosasymp} (which is the same as that in~\eqref{eq:KST} of Theorem~\ref{thm:KST})
  is not very good when $n$ is small, so instead we use~\eqref{eq:Erdosrec} of Theorem~\ref{thm:Erdos} to get
  \begin{equation*}
    \ex_{\kpart[2]}(n,K_{t,t})
    \leq
    \bigl(\ex_{\kpart[1]}(n,K_t^{(1)}) + t\cdot (t-1)\bigr)^{1/t}\cdot n^{2-1/t}
    \leq
    (t^2-1)^{1/t}\cdot n^{2-1/t}
    \leq
    3\cdot n^{2-1/t}
    \leq
    2\cdot k\cdot n^{2-1/t},
  \end{equation*}
  where the second inequality follows from~\eqref{eq:exkpart1}.
\end{proof}

\begin{lemma}\label{lem:KSTErdospartization}
  For every $n\in\NN$ and every $k,t\in\NN_+$, we have
  \begin{equation*}
    \ex(n,K_{t,\ldots,t}^{(k)}) \leq \frac{\ex_{\kpart}(n,K_{t,\ldots,t}^{(k)})}{k!}.
  \end{equation*}
\end{lemma}

\begin{proof}
  Given a $k$-hypergraph $G$ on $n$ vertices and without any copies of $K_{t,\ldots,t}^{(k)}$, we consider its $k$-partite
  version $G^{\kpart}$ as in Definition~\ref{def:kpart}: in combinatorial notation, we set $V_i(G^{\kpart})\df V(G)$ and
  \begin{equation*}
    E(G^{\kpart}) \df \{(v_1,\ldots,v_k)\in V(G)^k \mid \{v_1,\ldots,v_k\}\in E(G)\}.
  \end{equation*}

  It is straightforward to check that $G^{\kpart}$ has exactly $e(G^{\kpart}) = k!\cdot e(G)$ edges and has no copies of
  $K_{t,\ldots,t}^{(k)}$, so the result follows.
\end{proof}

\begin{theorem}[%
    \Kovari--\Sos--\Turan, non-partite version of~\protect{\cite{KST54}},
    \Erdos, essentially~\protect{\cite[Theorem~1]{Erd64}}]%
  \label{thm:Erdosasymp:partite}
  For every $n\in\NN$ and every $k,t\in\NN_+$ with $k\geq 2$, we have
  \begin{equation}\label{eq:Erdosasymp:partite}
    \ex(n,K_{t,\ldots,t}^{(k)})
    \leq
    \bigl(c_{t,k} + o(1)\bigr)\cdot n^{k-1/t^{k-1}},
  \end{equation}
  where
  \begin{equation*}
    c_{t,k}
    \df
    \frac{(t-1)^{1/t^{k-1}}}{k!}
    <
    \frac{1.5}{k!}.
  \end{equation*}
\end{theorem}

\begin{proof}
  Follows from~\eqref{eq:Erdosasymp} in Theorem~\ref{thm:Erdos} and Lemma~\ref{lem:KSTErdospartization}.
\end{proof}

\KSTErdos*

\begin{proof}
  Follows from~\eqref{eq:Erdos} in Theorem~\ref{thm:KSTErdos:partite} and Lemma~\ref{lem:KSTErdospartization}.
\end{proof}

\section{Calculations}
\label{sec:calc}

In this section, we make precise the calculations omitted in Propositions~\ref{prop:VCNkk->SUC} and~\ref{prop:SC->PHP} (whose
statements are repeated here for the reader's convenience). We will first need some calculation lemmas.

\begin{lemma}[\protect{\cite[Lemma~9.8]{CM24+}}]\label{lem:logcalcs}
  For every $x\geq x_0 > 1$, we have
  \begin{align*}
    \min\left\{\frac{\ln\ln x_0}{\ln x_0},0\right\}\cdot\ln x
    & \leq
    \ln\ln x
    \leq
    \frac{\ln x}{e}.
  \end{align*}
\end{lemma}

\begin{proof}
  For the first inequality, note that if $x_0\geq e$, then the left-hand side is $0$ and
  $\ln\ln x\geq 0$, so we may suppose that $x_0 < e$. In this case, it suffices to show that the
  function
  \begin{align*}
    f(x) & \df \frac{\ln\ln x}{\ln x}
  \end{align*}
  defined for $x\geq x_0$ attains its minimum at $x_0$. For this, we compute its derivative:
  \begin{align*}
    f'(x) & = \frac{1 - \ln\ln x}{x(\ln x)^2}
  \end{align*}
  and note that the only critical point of $f$ is at $x = e^e$.

  Since
  \begin{align*}
    f(e^e) & = \frac{1}{e} > 0, &
    \lim_{x\to\infty} f(x) & = 0, &
    f(x_0) & = \frac{\ln\ln x_0}{\ln x_0} < 0,
  \end{align*}
  the inequality follows.

  \medskip

  The second inequality is equivalent to $\ln x \leq x^{1/e}$, so it suffices to show that the function
  \begin{align*}
    g(x) & \df x^{1/e} - \ln x
  \end{align*}
  defined for $x\geq x_0$ is non-negative. We will show that $g$ is non-negative even extending its
  definition for $x\geq 1$. For this, we compute its derivative:
  \begin{align*}
    g'(x) & = \frac{x^{1/e - 1}}{e} - \frac{1}{x} = \frac{x^{1/e}/e - 1}{x}
  \end{align*}
  and note that the only critical point of $g$ is at $x = e^e$.

  Since
  \begin{align*}
    g(e^e) & = 0, &
    g(1) & = 1, &
    \lim_{x\to\infty} g(x) & = \infty,
  \end{align*}
  the inequality follows.
\end{proof}

\begin{lemma}[\protect{\cite[Lemma~9.11]{CM24+}}, a slight improvement of~\protect{\cite[Lemma~A.2]{SB14}}]\label{lem:xgeqlnx}
  If $a\geq 1/2$, $b\geq 0$ and
  \begin{equation*}
    x \geq \frac{2e}{e-1}\cdot a\ln(2a) + 2b
    \qquad ({}\leq 3.164\cdot a\ln(2a) + 2b),
  \end{equation*}
  then $x\geq a\ln x + b$.
\end{lemma}

\begin{proof}
  It suffices to show $x\geq 2a\ln x$ and $x\geq 2b$. Since $a\geq 1/2$, we have $\ln(2a)\geq 0$, hence
  $x\geq 2b$. To show $x\geq 2a\ln x$, it suffices to show that the function
  \begin{align*}
    f(x) & \df x - 2a\ln x
  \end{align*}
  defined for
  \begin{align*}
    x & \geq \frac{2e}{e-1}\cdot a\ln(2a) + 2b,
  \end{align*}
  is non-negative. We will show that $f$ is non-negative even extending its definition for $x\geq
  2e/(e-1)\cdot a\ln(2a)$. For this, we compute its derivative:
  \begin{align*}
    f'(x) & = 1 - \frac{2a}{x}
  \end{align*}
  and note that the only critical point of $f$ is potentially at $x = 2a$, if $2a$ belongs to the
  domain, that is, if $2a\geq 2e/(e-1)\cdot a\ln(2a)$. But if this is the case, we get
  \begin{align*}
    f(2a) & \geq \left(\frac{e}{e-1} - 1\right)\cdot 2a\ln(2a) \geq 0.
  \end{align*}

  On the other hand, since $\lim_{x\to\infty} f(x) = \infty$, it suffices to show that
  $f(2e/(e-1)\cdot a\ln(2a))$ is non-negative. But indeed, note that
  \begin{align*}
    f\left(\frac{2e}{e-1}\cdot a\ln(2a)\right)
    & =
    \frac{2e}{e-1}\cdot a\ln(2a) - 2a\ln\left(\frac{2e}{e-1}\cdot a\ln(2a)\right)
    \\
    & =
    \left(\frac{e}{e-1} - 1\right)\cdot 2a\ln(2a) - 2a\ln(\ln(2a)^{e/(e-1)})
    \\
    & \geq
    \left(\frac{e}{e-1} - 1 - \frac{1}{e-1}\right)\cdot 2a\ln(2a)
    =
    0,
  \end{align*}
  where the inequality follows from Lemma~\ref{lem:logcalcs}.
\end{proof}

\begin{lemma}\label{lem:xgeqlnx:super}
  Let $a,b,c,d,t\in\RR_{\geq 0}$, $k\in\NN_+$ and $x\in\RR$ be such that
  \begin{gather*}
    \begin{aligned}
      0 < t & \leq k, &
      \frac{a\cdot k}{t} & \geq \frac{1}{4}, &
      b & \geq 1,
    \end{aligned}
    \\
    \begin{aligned}
      x
      & \geq
      \left(
      \frac{e}{e-1}\cdot\frac{4\cdot a\cdot k}{t}\cdot\ln\left(\frac{4\cdot a\cdot k}{t}\right)
      + 4\cdot a\cdot c + b
      \right)^{1/t}
      + (2\cdot d)^{1/k}
      \\
      \Biggl( & \leq
      \left(
      6.328
      \cdot\frac{a\cdot k}{t}\cdot\ln\left(\frac{4\cdot a\cdot k}{t}\right)
      + 4\cdot a\cdot c + b
      \right)^{1/t}
      + (2\cdot d)^{1/k}
      \Biggr).
    \end{aligned}
  \end{gather*}

  Then
  \begin{equation*}
    x^k
    \geq
    a\cdot x^{k-t}\cdot\bigl(\ln(x^k+b) + c\bigr) + d.
  \end{equation*}
\end{lemma}

\begin{proof}
  It suffices to show $x^k \geq 2\cdot d$ and
  \begin{equation*}
    x^k \geq 2\cdot a\cdot x^{k-t}\cdot(\ln(x^k+b) + c).
  \end{equation*}

  The former follows simply because $x\geq (2\cdot d)^{1/k}$ and $a\cdot k/t\geq 1/4$, so the logarithm in the lower bound for
  $x$ is non-negative. For the latter, since $x^k+b\leq (x^t+b)^{k/t}$ (as $t\leq k$ and $b\geq 1$), it suffices to show
  \begin{align*}
    x^t & \geq \frac{2\cdot a\cdot k}{t}\cdot(\ln(x^t+b)+c),
    \intertext{which is equivalent to}
    x^t + b & \geq \frac{2\cdot a\cdot k}{t}\cdot\ln(x^t+b) + 2\cdot a\cdot c + b.
  \end{align*}

  But this follows from Lemma~\ref{lem:xgeqlnx} as $a\cdot k/t\geq 1/4$ and
  \begin{equation*}
    x^t + b
    \geq
    \frac{e}{e-1}\cdot\frac{4\cdot a\cdot k}{t}\cdot\ln\left(\frac{4\cdot a\cdot k}{t}\right)
    + 2\cdot (2\cdot a\cdot c + b).
    \qedhere
  \end{equation*}
\end{proof}

\propVCNkktoSUC*

\begin{proof}
  (The beginning of the proof is the same as in the proof sketch until we split into cases.)

  By Lemma~\ref{lem:peel}, we know that for every $m\in\NN_+$ and every $[m]$-sample $(x,y)$ (and in the non-partite case every
  order choice $\alpha$ for $[m]$), we have
  \begin{equation*}
    \sup_{H\in\cH} \lvert L_{x,y,\ell}(H) - L_{x,\rn{E}_\rho(y),\ell}(H)\rvert
    \leq
    \epsilon
  \end{equation*}
  with probability at least
  \begin{equation*}
    1
    - 2\cdot\exp\left(-\frac{\epsilon^2\cdot M_k}{12\cdot\lVert\ell\rVert_\infty^2}\right)
    - 2\cdot\gamma_\cH(m)\cdot\exp\left(-\frac{\epsilon^2\cdot\rho^2\cdot M_k}{2\cdot\lVert\ell\rVert_\infty^2}\right),
  \end{equation*}
  (and in the non-partite case, we replace both instances of $L$ by $L^\alpha$), so it suffices to show that when $m\geq
  m^{\SUC}_{\cH,\ell}(\epsilon,\delta,\rho)$, the quantity above is at least $1-\delta$. In turn, it suffices to show that each
  of the negative terms above is at most $\delta/2$ in absolute value, or, equivalently, show that
  \begin{gather}
    M_k \geq \frac{12\cdot\lVert\ell\rVert_\infty^2}{\epsilon^2}\cdot\ln\frac{4}{\delta},
    \label{eq:VCNkk->SUC:first}
    \\
    \ln\bigl(\gamma_\cH(m)\bigr)
    -\frac{\epsilon^2\cdot\rho^2\cdot M_k}{2\cdot\lVert\ell\rVert_\infty^2}
    \leq
    \ln\frac{\delta}{4}.
    \label{eq:VCNkk->SUC:second}
  \end{gather}

  It is clear from the first term in the maxima in the definition of $m^{\SUC}_{\cH,\ell}(\epsilon,\delta,\rho)$
  that~\eqref{eq:VCNkk->SUC:first} holds (in the non-partite case, we also use the bound $\binom{m}{k}\geq (m/k)^k$).

  Using Lemma~\ref{lem:VCNkk->kgrowth} (and the fact that $B_\ell\geq\lVert\ell\rVert_\infty$), \eqref{eq:VCNkk->SUC:second} is
  equivalent to
  \begin{equation}\label{eq:VCNkk->SUC:goal}
    \begin{aligned}
      \MoveEqLeft
      \ln\frac{\delta}{4} + \frac{\epsilon^2\cdot\rho^2\cdot M_k}{2\cdot B_\ell^2}
      \\
      & \geq
      \begin{dcases*}
        2\cdot k
        \cdot m^{k-1/(\VCN_{k,k}(\cH)+1)^{k-1}}
        \cdot\left(\ln(m^k+1) + \ln\binom{\lvert\Lambda\rvert}{2}\right),
        & in the partite if $k\geq 2$,
        \\
        \frac{2\cdot m^{k-1/(\VCN_{k,k}(\cH)+1)^{k-1}}}{(k-1)!}
        \!\cdot\!
        \left(\ln\left(\binom{m}{k}+1\right) + \ln\binom{\lvert\Lambda\rvert^{k!}}{2}\right),
        & in the non-partite if $k\geq 2$,
        \\
        \VCN_{k,k}(\cH)\cdot\left(\ln(m+1) + \ln\binom{\lvert\Lambda\rvert}{2}\right),
        & if $k=1$.
      \end{dcases*}
    \end{aligned}
  \end{equation}

  In the case $k=1$, \eqref{eq:VCNkk->SUC:goal} amounts to
  \begin{equation*}
    \ln\frac{\delta}{4} + \frac{\epsilon^2\cdot\rho^2\cdot m}{2\cdot B_\ell^2}
    \geq
    \VCN_{k,k}(\cH)\cdot\left(\ln(m+1) + \ln\binom{\lvert\Lambda\rvert}{2}\right),
  \end{equation*}
  which is equivalent to $x\geq a\ln x + b$, where
  \begin{align*}
    x
    & \df
    m+1,
    \\
    a
    & \df
    \frac{2\cdot B_\ell^2\cdot\VCN_{k,k}(\cH)}{\epsilon^2\cdot\rho^2}
    \geq
    2\cdot B_\ell^2,
    \\
    b
    & \df
    \frac{2\cdot B_\ell^2}{\epsilon^2\cdot\rho^2}
    \cdot\left(
    \VCN_{k,k}(\cH)\cdot\ln\binom{\lvert\Lambda\rvert}{2} + \ln\frac{4}{\delta}
    \right) + 1,
  \end{align*}
  so the result follows from Lemma~\ref{lem:xgeqlnx} as our choice of $m^{\SUC}_{\cH,\ell}(\epsilon,\delta,\rho)$ ensures that
  $a\geq 1/2$, $b\geq 0$ and $x\geq 2e/(e-1)\cdot a\ln(2a) + 2b$.

  \medskip

  We now consider the partite case with $k\geq 2$. Recalling that in this case $M_k = m^k$, \eqref{eq:VCNkk->SUC:goal} amounts to
  \begin{equation*}
    m^k
    \geq
    \frac{4\cdot k\cdot B_\ell^2}{\epsilon^2\cdot\rho^2}
    \cdot m^{k-1/(\VCN_{k,k}(\cH)+1)^{k-1}}
    \cdot\left(\ln(m^k+1) + \ln\binom{\lvert\Lambda\rvert}{2}\right)
    +
    \frac{2\cdot B_\ell^2}{\epsilon^2\cdot\rho^2}\cdot
    \ln\frac{4}{\delta},
  \end{equation*}
  that is, we want $x^k\geq a\cdot x^{k-t}\cdot(\ln(x^k+b) + c) + d$, where
  \begin{align*}
    x & \df m,
    &
    a & \df \frac{4\cdot k\cdot B_\ell^2}{\epsilon^2\cdot\rho^2},
    &
    b & \df 1,
    \\
    c & \df \ln\binom{\lvert\Lambda\rvert}{2},
    &
    d & \df \frac{2\cdot B_\ell^2}{\epsilon^2\cdot\rho^2}\cdot\ln\frac{4}{\delta},
    &
    t & \df \frac{1}{(\VCN_{k,k}(\cH)+1)^{k-1}},
  \end{align*}
  so the result follows from Lemma~\ref{lem:xgeqlnx:super} as our choice of $m^{\SUC}_{\cH,\ell}(\epsilon,\delta,\rho)$ ensures
  that $a,b,c,d,t\geq 0$, $0 < t\leq k$,
  \begin{equation*}
    \frac{a\cdot k}{t}
    =
    \frac{4\cdot k^2\cdot B_\ell^2\cdot (\VCN_{k,k}(\cH)+1)^{k-1}}{\epsilon^2\cdot\rho^2}
    \geq
    4\cdot k^2\cdot B_\ell^2
    \geq
    \frac{1}{4}
  \end{equation*}
  and
  \begin{equation*}
    x
    \geq
    \left(
    \frac{e}{e-1}\cdot\frac{4\cdot a\cdot k}{t}\cdot\ln\left(\frac{4\cdot a\cdot k}{t}\right)
    + 4\cdot a\cdot c + b
    \right)^{1/t}
    + (2\cdot d)^{1/k}.
  \end{equation*}

  \medskip

  Finally, we consider the non-partite case with $k\geq 2$. Recalling that in this case $M_k = \binom{m}{k}$,
  \eqref{eq:VCNkk->SUC:goal} amounts to
  \begin{equation*}
    \ln\frac{\delta}{4} + \frac{\epsilon^2\cdot\rho^2\cdot\binom{m}{k}}{2\cdot B_\ell^2}
    \geq
    \frac{2\cdot m^{k-1/(\VCN_{k,k}(\cH)+1)^{k-1}}}{(k-1)!}
    \cdot\left(\ln\left(\binom{m}{k}+1\right) + \ln\binom{\lvert\Lambda\rvert^{k!}}{2}\right).
  \end{equation*}

  Using the bounds $(m/k)^k\leq \binom{m}{k}\leq m^k/k!$, it suffices to show
  \begin{equation*}
    m^k
    \geq
    \frac{4\cdot B_\ell^2\cdot k^k}{(k-1)!\cdot\epsilon^2\cdot\rho^2}
    \cdot m^{k-1/(\VCN_{k,k}(\cH)+1)^{k-1}}
    \cdot\left(\ln\left(m^k + k!\right) - \ln k! + \ln\binom{\lvert\Lambda\rvert^{k!}}{2}\right)
    + \frac{2\cdot B_\ell^2\cdot k^k}{\epsilon^2\cdot\rho^2}\cdot\ln\frac{4}{\delta},
  \end{equation*}
  that is, we want $x^k\geq a\cdot x^{k-t}\cdot(\ln(x^k+b) + c) + d$, where
  \begin{align*}
    x & \df m,
    &
    a & \df \frac{4\cdot B_\ell^2\cdot k^k}{(k-1)!\cdot\epsilon^2\cdot\rho^2},
    &
    b & \df k! \geq 1
    \\
    c & \df \ln\binom{\lvert\Lambda\rvert^{k!}}{2} - \ln k!,
    &
    d & \df \frac{2\cdot B_\ell^2\cdot k^k}{\epsilon^2\cdot\rho^2}\cdot\ln\frac{4}{\delta},
    &
    t & \df \frac{1}{(\VCN_{k,k}(\cH)+1)^{k-1}},
  \end{align*}
  so the result follows from Lemma~\ref{lem:xgeqlnx:super} as our choice of $m^{\SUC}_{\cH,\ell}(\epsilon,\delta,\rho)$ ensures
  that $a,b,c,d,t\geq 0$, $0 < t\leq k$,
  \begin{equation*}
    \frac{a\cdot k}{t}
    =
    \frac{4\cdot B_\ell^2\cdot k^{k+1}\cdot (\VCN_{k,k}(\cH)+1)^{k-1}}{(k-1)!\cdot\epsilon^2\cdot\rho^2}
    \geq
    4\cdot k^2\cdot B_\ell^2
    \geq
    \frac{1}{4}
  \end{equation*}
  and
  \begin{equation*}
    x
    \geq
    \left(
    \frac{e}{e-1}\cdot\frac{4\cdot a\cdot k}{t}\cdot\ln\left(\frac{4\cdot a\cdot k}{t}\right)
    + 4\cdot a\cdot c + b
    \right)^{1/t}
    + (2\cdot d)^{1/k}.
    \qedhere
  \end{equation*}
\end{proof}

The next lemma says that a loss function $\ell$ that is separated and bounded satisfies a weak version of triangle inequality
for the empirical loss.

\begin{lemma}\label{lem:almostmetric}
  Let $k\in\NN_+$, let $\Omega=(\Omega_i)_{i=1}^k$ be a $k$-tuple of non-empty Borel spaces (a single non-empty Borel space,
  respectively), let $\Lambda$ be a non-empty Borel space, let $\cH\subseteq\cF_k(\Omega,\Lambda)$ be a $k$-partite ($k$-ary,
  respectively) hypothesis class, let $\ell$ be a $k$-partite ($k$-ary, respectively) loss function that is separated and
  bounded, let $m\in\NN_+$ and let $x\in\cE_m(\Omega)$.

  Then for every $F,F',H\in\cH$, we have
  \begin{equation*}
    s(\ell)\cdot L_{x,F^*_m(x),\ell}^\alpha(F')
    \leq
    \lVert\ell\rVert_\infty\cdot
    \bigl(L_{x,F^*_m(x),\ell}^\alpha(H) + L_{x,(F')^*_m(x),\ell}^\alpha(H)\bigr)
  \end{equation*}
  for every order choices $\alpha$ for $[m]$ in the non-partite case and in the partite case, the same holds dropping the order
  choices.
\end{lemma}

\begin{proof}
  We prove only the non-partite case as the partite case has an analogous proof. Let
  \begin{equation*}
    D(F,F')
    \df
    \left\{U\in\binom{[m]}{k} \;\middle\vert\;
    \exists\beta\in([m])_k, (\im(\beta) = U\land F^*_m(x)_\beta\neq (F')^*_m(x)_\beta)
    \right\}
  \end{equation*}
  and define $D(F,H)$ and $D(F',H)$ analogously. Note that an alternative formula for the above is
  \begin{equation*}
    D(F,F')
    =
    \left\{U\in\binom{[m]}{k} \;\middle\vert\;
    b_\alpha\bigl(F^*_m(x)\bigr) \neq b_\alpha\bigl((F')^*_m(x)\bigr)
    \right\}.
  \end{equation*}

  Clearly, we have $\lvert D(F,F')\rvert\leq\lvert D(F,H)\rvert + \lvert D(F',H)\rvert$. On the other hand, we have
  \begin{align*}
    s(\ell)\cdot L_{x,F^*_m(x),\ell}^\alpha(F')
    & =
    \frac{s(\ell)}{\binom{m}{k}}\sum_{U\in\binom{[m]}{k}}
    \ell\Bigl(\alpha_U^*(x), b_\alpha\bigl((F')^*_m(x)\bigr)_U, b_\alpha\bigl(F^*_m(x)\bigr)_U\Bigr)
    \\
    & \leq
    \frac{s(\ell)\cdot\lVert\ell\rVert_\infty}{\binom{m}{k}}\cdot\lvert D(F,F')\rvert
    \\
    & \leq
    \frac{s(\ell)\cdot\lVert\ell\rVert_\infty}{\binom{m}{k}}
    \cdot\bigl(\lvert D(F,H)\rvert + D(F',H)\rvert\bigr)
    \\
    & \leq
    \begin{multlined}[t]
      \frac{\lVert\ell\rVert_\infty}{\binom{m}{k}}\sum_{U\in\binom{[m]}{k}}
      \biggl(
      \ell\Bigl(\alpha_U^*(x), b_\alpha\bigl(H^*_m(x)\bigr)_U, b_\alpha\bigl(F^*_m(x)\bigr)_U\Bigr)
      \\
      + \ell\Bigl(\alpha_U^*(x), b_\alpha\bigl(H^*_m(x)\bigr)_U, b_\alpha\bigl((F')^*_m(x)\bigr)_U\Bigr)
      \biggr)
    \end{multlined}
    \\
    & =
    \lVert\ell\rVert_\infty\cdot
    \bigl(L_{x,F^*_m(x),\ell}^\alpha(H) + L_{x,(F')^*_m(x),\ell}^\alpha(H)\bigr).
    \qedhere
  \end{align*}
\end{proof}

\propSCtoPHP*

\begin{proof}
  Taking into account partite vs.\ non-partite and $\ell$ metric vs.\ separated and bounded, there are a total of four cases to
  prove. We will prove them essentially simultaneously by making use of the already defined
  \begin{align*}
    c_\ell & \df
    \begin{dcases*}
      1, & if $\ell$ is metric,\\
      \frac{s(\ell)}{\lVert\ell\rVert_\infty}, & otherwise,
    \end{dcases*}
    &
    K & \df
    \begin{dcases*}
      0, & in the partite case,\\
      k - 1, & in the non-partite case,
    \end{dcases*}
  \end{align*}
  along with the notation
  \begin{align*}
    [m]^{(k)} & \df
    \begin{dcases*}
      [m]^k, & in the partite case,\\
      ([m])_k & in the non-partite case.
    \end{dcases*}
    &
    m^{(k)} & \df
    \begin{dcases*}
      m^k, & in the partite case,\\
      (m)_k & in the non-partite case.
    \end{dcases*}
  \end{align*}
  Note that regardless of the case, we have
  \begin{equation*}
    (m-K)^k\leq\lvert [m]^{(k)}\rvert = m^{(k)}\leq m^k.
  \end{equation*}

  Furthermore, since empirical losses in the non-partite case require an order choice, throughout the argument, whenever we have
  an order choice $\alpha$ for $[m]$, it only applies to the non-partite case and should simply be dropped in the partite case.

  First, note that the result is trivial when $\lvert\Lambda\rvert=1$, so we suppose $\lvert\Lambda\rvert\geq 2$.

  Given $\epsilon,\delta,\rho\in(0,1)$, first note that the conditions $0 < \widetilde{\delta} < \delta$ and $0 <
  \widetilde{\rho} < (2^\rho-1)/(\lvert\Lambda\rvert-1)$ ensure that both the numerator and denominator on the second term of
  the right-hand side of~\eqref{eq:SC->PHP:m} are well-defined and positive. Also note that the minimum in the same equation is
  indeed attained as the ceiling ensures that the expression takes values in $\NN$. Let then
  $(\widetilde{\delta},\widetilde{\rho})$ attain the minimum in~\eqref{eq:SC->PHP:m}.

  Let $m\geq m^{\hPHP[m^k]}_{\cH,\ell}(\epsilon,\delta,\rho)$ be an integer, let $\alpha$ be an order choice for $[m]$ (in the
  non-partite case), let $\mu\in\Pr(\Omega)$, let $H_1,\ldots,H_t\in\cH$ be such that $t\geq 2^{\rho\cdot m^k}$ and let
  \begin{equation*}
    S_\epsilon^\alpha
    \df
    \{x\in\cE_m(\Omega) \mid (H_1,\ldots,H_t)\text{ is $\epsilon$-separated on $x$ w.r.t.\ $\ell$ and $\alpha$}\}.
  \end{equation*}
  Our goal is to show that $\mu(S_\epsilon^\alpha)\leq\delta$.

  Let us encode the erasure operation $\rn{E}_{\widetilde{\rho}}$ in a different manner. Given $y\in\Lambda^{[m]^{(k)}}$ and
  $w\in\{0,1\}^{[m]^{(k)}}$, let $E(y,w)\in(\Lambda\cup\{\unk\})^{[m]^{(k)}}$ be given by
  \begin{equation}\label{eq:erasureencode}
    E(y,w)_\beta \df
    \begin{dcases*}
      y_\beta, & if $w_\beta=1$,\\
      \unk, & if $w_\beta=0$
    \end{dcases*}
  \end{equation}
  and note that if $\nu_m\in\Pr(\{0,1\}^{[m]^{(k)}})$ is the distribution in which each entry is $1$ independently with
  probability $\widetilde{\rho}$, then for $\rn{w}\sim\nu_m$, we have $\rn{E}_{\widetilde{\rho}}(y)\sim E(y,\rn{w})$.

  For each $i\in[t]$, let
  \begin{equation*}
    G_i^\alpha
    \df
    \left\{(x,w)\in\cE_m(\Omega)\times\{0,1\}^{[m]^{(k)}}
    \;\middle\vert\;
    L_{x,(H_i)^*_m(x),\ell}^\alpha\biggl(\cA\Bigl(x, E\bigl((H_i)^*_m(x), w\bigr)\Bigr)\biggr)
    \leq
    \frac{c_\ell\cdot\epsilon}{2}
    \right\}.
  \end{equation*}

  Note that since $m\geq m^{\SC}_{\cH,\ell,\cA}(c_\ell\cdot\epsilon/2,\widetilde{\delta},\widetilde{\rho})$, sample completion
  $k$-PAC learnability guarantees that
  \begin{equation}\label{eq:SC->PHP:SC}
    \PP_{\rn{x}\sim\mu^m}\Bigl[\PP_{\rn{w}\sim\nu_m}\bigl[(\rn{x},\rn{w})\in G_i^\alpha\bigr]\Bigr]
    \geq
    1 - \widetilde{\delta}.
  \end{equation}

  We claim that every fixed $(x,w)\in S_\epsilon^\alpha\times\{0,1\}^{[m]^{(k)}}$ is in at most $\lvert\Lambda\rvert^{\lvert
    w^{-1}(1)\rvert}$ many $G_i^\alpha$. To see this, first note that for all $i\in[t]$, exactly the same entries of
  $E((H_i)^*_m(x),w)$ are $\unk$; namely, these are exactly the entries of $w$ that are $0$. If $(x,w)\in
  S_\epsilon^\alpha\times\{0,1\}^{[m]^{(k)}}$ is in more than $\lvert\Lambda\rvert^{\lvert w^{-1}(1)\rvert}$ many $G_i^\alpha$,
  then by Pigeonhole Principle, there must exist $i,j\in[t]$ with $i < j$ such that $(x,w)\in G_i^\alpha\cap G_j^\alpha$ and
  $E((H_i)^*_m(x),w) = E((H_j)^*_m(x),w)$, which in particular implies $\cA(x, E((H_i)^*_m(x),w)) = \cA(x, E((H_j)^*_m(x),w))$,
  hence we get
  \begin{align*}
    c_\ell\cdot\epsilon
    & \geq
    L_{x,(H_i)^*_m(x),\ell}^\alpha\biggl(\cA\Bigl(x, E\bigl((H_i)^*_m(x), w\bigr)\Bigr)\biggr)
    + L_{x,(H_j)^*_m(x),\ell}^\alpha\biggl(\cA\Bigl(x, E\bigl((H_j)^*_m(x), w\bigr)\Bigr)\biggr)
    \\
    & \geq
    c_\ell\cdot L_{x,(H_i)^*(x),\ell}^\alpha(H_j),
  \end{align*}
  where the second inequality follows from triangle inequality when $\ell$ is metric and from Lemma~\ref{lem:almostmetric} when
  $\ell$ is separated and bounded, so $L_{x,(H_i)^*_m(x),\ell}^\alpha(H_j)\leq\epsilon$, contradicting the fact that
  $(H_1,\ldots,H_t)$ is $\epsilon$-separated on $x$ with respect to $\ell$ and $\alpha$ as $x\in S_\epsilon^\alpha$. Thus, we
  conclude that for every $(x,w)\in S_\epsilon^\alpha\times\{0,1\}^{[m]^{(k)}}$, we have
  \begin{equation*}
    \sum_{i\in[t]} \One_{G_i^\alpha}(x,w) \leq \lvert\Lambda\rvert^{\lvert w^{-1}(1)\rvert}.
  \end{equation*}

  Putting this together with~\eqref{eq:SC->PHP:SC}, we get
  \begin{align*}
    (1 - \widetilde{\delta})\cdot t
    & \leq
    \EE_{\rn{x}\sim\mu^m}\left[\EE_{\rn{w}\sim\nu_m}\left[\sum_{i\in[t]}\One_{G_i^\alpha}(\rn{x},\rn{w})\right]\right]
    \\
    & =
    \begin{multlined}[t]
      \mu(S_\epsilon^\alpha)\cdot\EE_{\rn{x}\sim\mu^m}\left[\EE_{\rn{w}\sim\nu_m}\left[\sum_{i\in[t]}\One_{G_i^\alpha}(\rn{x},\rn{w})\right]
        \;\middle\vert\;\rn{x}\in S_\epsilon^\alpha\right]
      \\
      + \bigl(1-\mu(S_\epsilon^\alpha)\bigr)\cdot\EE_{\rn{x}\sim\mu^m}\left[
        \EE_{\rn{w}\sim\nu_m}\left[\sum_{i\in[t]}\One_{G_i^\alpha}(\rn{x},\rn{w})\right]
        \;\middle\vert\;\rn{x}\notin S_\epsilon^\alpha
        \right]
    \end{multlined}
    \\
    & \leq
    \mu(S_\epsilon^\alpha)\cdot\EE_{\rn{w}\sim\nu_m}[\lvert\Lambda\rvert^{\lvert\rn{w}^{-1}(1)\rvert}]
    + \bigl(1-\mu(S_\epsilon^\alpha)\bigr)\cdot t
    \\
    & =
    \mu(S_\epsilon^\alpha)\cdot\bigl(\widetilde{\rho}\cdot\lvert\Lambda\rvert + (1-\widetilde{\rho})\bigr)^{m^{(k)}}
    + \bigl(1-\mu(S_\epsilon^\alpha)\bigr)\cdot t
    \\
    & =
    t
    + \mu(S_\epsilon^\alpha)\cdot\Bigl(\bigl(\widetilde{\rho}\cdot(\lvert\Lambda\rvert - 1) + 1\bigr)^{m^{(k)}} - t\Bigr),
  \end{align*}
  where the second equality follows since the entries of $\rn{w}$ are independent Bernoulli variables with parameter
  $\widetilde{\rho}$. Thus, we get
  \begin{equation*}
    \mu(S_\epsilon^\alpha)\cdot\Bigl(t - \bigl(\widetilde{\rho}\cdot(\lvert\Lambda\rvert - 1) + 1\bigr)^{m^{(k)}}\Bigr)
    \leq
    \widetilde{\delta}\cdot t.
  \end{equation*}

  We now note that since $t\geq 2^{\rho\cdot m^k}\geq 2^{\rho\cdot m^{(k)}}$ and since
  $\widetilde{\rho}<(2^\rho-1)/(\lvert\Lambda\rvert-1)$, the expression in the parentheses on the left-hand side of the above is
  positive, so we conclude that
  \begin{equation*}
    \mu(S_\epsilon^\alpha)
    \leq
    \frac{\widetilde{\delta}\cdot t}{
      t - \bigl(\widetilde{\rho}\cdot(\lvert\Lambda\rvert-1) + 1\bigr)^{m^{(k)}}
    }
    \leq
    \frac{\widetilde{\delta}\cdot 2^{\rho\cdot m^{(k)}}}{
      2^{\rho\cdot m^{(k)}} - \bigl(\widetilde{\rho}\cdot(\lvert\Lambda\rvert-1) + 1\bigr)^{m^{(k)}}
    },
  \end{equation*}
  where the second inequality follows from $t\geq 2^{\rho\cdot m^k}\geq 2^{\rho\cdot m^{(k)}}$ and the fact that for
  $c\df(\widetilde{\rho}\cdot(\lvert\Lambda\rvert-1) + 1)^{m^{(k)}} > 0$, the function $(c,\infty)\ni x\mapsto
  x/(x-c)\in\RR_{\geq 0}$ is decreasing.

  Our goal is to show that the quantity above is at most $\delta$, or, equivalently, to show that
  \begin{equation*}
    \widetilde{\delta}\cdot 2^{\rho\cdot m^{(k)}}
    \leq
    \delta\cdot\Bigl(2^{\rho\cdot m^{(k)}} - \bigl(\widetilde{\rho}\cdot(\lvert\Lambda\rvert-1) + 1\bigr)^{m^{(k)}}\Bigr),
  \end{equation*}
  which itself is equivalent to
  \begin{equation*}
    \delta\cdot\bigl(\widetilde{\rho}\cdot(\lvert\Lambda\rvert-1) + 1\bigr)^{m^{(k)}}
    \leq
    (\delta-\widetilde{\delta})\cdot 2^{\rho\cdot m^{(k)}}
  \end{equation*}
  and is in turn equivalent to
  \begin{equation*}
    \ln(\delta) - \ln(\delta-\widetilde{\delta})
    \leq
    m^{(k)}\cdot\Bigl(\rho\cdot\ln(2) - \ln\bigl(\widetilde{\rho}\cdot(\lvert\Lambda\rvert-1) + 1\bigr)\Bigr).
  \end{equation*}

  But this follows
  \begin{equation*}
    m
    \geq
    m^{\hPHP[m^k]}_{\cH,\ell}(\epsilon,\delta,\rho)
    \geq
    \left(
    \frac{
      \ln(\delta) - \ln(\delta-\widetilde{\delta})
    }{
      \rho\cdot\ln(2) - \ln\bigl(\widetilde{\rho}\cdot(\lvert\Lambda\rvert-1) + 1\bigr)
    }
    \right)^{1/k}
    + K,
  \end{equation*}
  using $m^{(k)}\geq (m-K)^k$.
\end{proof}

\propPHPtoVCNkk*

\begin{proof}
  We prove first the partite case.

  First, let us show that all calculations in~\eqref{eq:PHP->VCNkk:VCNkk:partite} and~\eqref{eq:PHP->VCNkk:m:partite} are valid.

  The condition $\delta\in(0,4^{-k})$ ensures that the logarithm in~\eqref{eq:PHP->VCNkk:m:partite} is well-defined and the
  condition $\rho\in(0,1-h_2(\epsilon/s(\ell)))$ ensures that the denominator in~\eqref{eq:PHP->VCNkk:m:partite} is positive,
  hence the $(1/k)$th power in~\eqref{eq:PHP->VCNkk:m:partite} is also well-defined.

  Since $m\geq 2$ and the function $(1-1/x)^x$ is increasing (for $x > 1$), it follows that
  \begin{equation}\label{eq:between4ande}
    4^{-k}\leq \left(1 - \frac{1}{m}\right)^{k\cdot m}\leq e^{-k},
  \end{equation}
  this together with $d\geq\rho\cdot m^k$ means that in~\eqref{eq:PHP->VCNkk:VCNkk:partite}, the expression under the $(1/d)$th
  power is at least
  \begin{equation*}
    \frac{1 - e^{-k}\cdot(1 - 2^{(\rho + h_2(\epsilon/s(\ell)) - 1)\cdot m^k})}{1-\delta},
  \end{equation*}
  which is non-negative since $\rho\in(0,1-h_2(\epsilon/s(\ell)))$, so the $(1/d)$th power is well-defined.

  Using the other inequality of~\eqref{eq:between4ande} and $d\leq\rho\cdot m^k+1$, we also deduce that the expression under the
  logarithm in~\eqref{eq:PHP->VCNkk:VCNkk:partite} is at least
  \begin{equation*}
    1 - \left(\frac{1 - 4^{-k}\cdot(1 - 2^{(\rho + h_2(\epsilon/s(\ell)) - 1)\cdot m^k + 1})}{1-\delta}\right)^{1/d},
  \end{equation*}
  which is non-negative since
  \begin{equation*}
    m
    \geq
    \left(
    \frac{1 - \log_2(1-4^k\cdot\delta)}{1-h_2(\epsilon/s(\ell))-\rho}
    \right)^{1/k}.
  \end{equation*}

  Thus, all expressions in~\eqref{eq:PHP->VCNkk:VCNkk:partite} and~\eqref{eq:PHP->VCNkk:m:partite} are well-defined.

  Furthermore, note that the minimum on the right-hand side of~\eqref{eq:PHP->VCNkk:VCNkk:partite} is indeed attained due to the
  floor. Let then $(\epsilon,\delta,\rho)$ attain the minimum in~\eqref{eq:PHP->VCNkk:VCNkk:partite} (and let $m$ and $d$ be
  defined as in~\eqref{eq:PHP->VCNkk:m:partite} and~\eqref{eq:PHP->VCNkk:d:partite}, respectively). Let $n\df\VCN_{k,k}(\cH)$
  and suppose for a contradiction that
  \begin{equation*}
    n
    >
    \max\left\{
    m^2,
    \left(
    d
    - \log_2\left(
    1 - \left(\frac{1 - (1-1/m)^{k\cdot m}\cdot(1 - 2^{(h_2(\epsilon/s(\ell)) - 1)\cdot m^k + d})}{1-\delta}\right)^{1/d}
    \right)
    \right)^{1/k}
    \right\}.
  \end{equation*}
  (Note that we removed the floor as $n$ is an integer.)

  As per definition of $\VCN_{k,k}(\cH)$ in Definition~\ref{def:SCpart:VCNkk}, let $z\in\cE_n(\Omega)$ be such that
  \begin{equation*}
    \cH_z \df \{H^*_n(z) \mid H\in\cH\} \subseteq \Lambda^{[n]^k}
  \end{equation*}
  Natarajan-shatters $[n]^k$. It will be convenient to index our witnesses to the shattering by $\FF_2^{[n]^k}$. Namely, we know
  that there exist $f_0,f_1\colon [n]^k\to\Lambda$ with $f_0(\beta)\neq f_1(\beta)$ for every $\beta\in[n]^k$ and $H_w\in\cH$
  ($w\in\FF_2^{[n]^k}$) such that for every $w\in\FF_2^{[n]^k}$ and every $\beta\in[n]^k$, we have $(H_w)^*_n(z)_\beta =
  f_{w_\beta}(\beta)$.

  We will prove that there exists $C\subseteq\FF_2^{[n]^k}$ of size at least $2^{\rho\cdot m^k}$ and a probability $k$-partite
  template $\mu\in\Pr(\Omega)$ such that if $C = \{w_1,\ldots,w_{\lvert C\rvert}\}$ and $\rn{x}\sim\mu^m$, then
  $(H_{w_1},\ldots,H_{w_{\lvert C\rvert}})$ is $\epsilon$-separated on $\rn{x}$ with probability larger than $\delta$.

  For $\gamma=(\gamma_i)_{i\in[k]}$ with $\gamma_i\colon [m]\to[n]$, we define a function $\gamma^*\colon\FF_2^{[n]^k}\to
  \FF_2^{[m]^k}$ by
  \begin{equation*}
    \gamma^*(w)_\beta \df w_{\gamma_{\#}(\beta)},
  \end{equation*}
  where $\gamma_{\#}\colon [m]^k\to[n]^k$ is the ``product'' function given by
  $\gamma_{\#}(\beta)_i\df\gamma_i(\beta_i)$. Clearly, $\gamma^*$ is a linear map. For a linear code $C\subseteq\FF_2^{[n]^k}$
  (i.e., an $\FF_2$-linear subspace), define
  \begin{align*}
    \dist_\gamma(C)
    & \df
    \inf_{\substack{w_1,w_2\in C\\w_1\neq w_2}} \lvert\{j\in [m]^k \mid \gamma^*(w_1)_j\neq \gamma^*(w_2)_j\}\rvert
    \\
    & =
    \inf_{w\in C\setminus\{0\}} \lvert\gamma^*(w)^{-1}(1)\rvert,
  \end{align*}
  where the equality follows since $C$ is linear.

  Our goal is to find a linear code $C\subseteq\FF_2^{[n]^k}$ of dimension $d\df\ceil{\rho\cdot m^k}$ such that for most
  $\gamma$, we have $\dist_\gamma(C) > \epsilon\cdot m^k/s(\ell)$. In fact, we will prove that a uniformly random linear code of
  dimension $d$ satisfies this property with positive probability:

  \begin{claim}\label{clm:linearcode:partite}
    There exists a linear code $C\subseteq\FF_2^{[n]^k}$ of dimension $d\df\ceil{\rho\cdot m^k}$ such that if
    $\rn{\gamma}_1,\ldots,\rn{\gamma}_k$ are i.i.d.\ with each $\rn{\gamma}_i$ uniformly distributed in $[n]^m$, then
    \begin{equation*}
      \PP_{\rn{\gamma}}\left[
        \dist_{\rn{\gamma}}(C) > \frac{\epsilon\cdot m^k}{s(\ell)}
        \right]
      >
      \delta.
    \end{equation*}
  \end{claim}

  Before we prove the claim, let us see why it yields the result.

  First note that since $\cH_z$ Natarajan-shatters $[n]^k$, there cannot be repetitions among the variables of $z$ corresponding
  to the same part, that is, if $z = (z_1,\ldots,z_k)$ (with $z_i\in\Omega_i^n$), then the coordinates of each $z_i$ are
  distinct. We can then define $\mu\in\Pr(\Omega)$ by letting $\mu_i\in\Pr(\Omega_i)$ ($i\in[k]$) be the uniform measure on the
  set
  \begin{equation*}
    \{(z_i)_1,\ldots,(z_i)_n\}
  \end{equation*}
  (which has exactly $n$ points).

  Let $C\subseteq\FF_2^{[n]^k}$ be as in Claim~\ref{clm:linearcode:partite} and enumerate its elements as $C = \{w_1,\ldots,w_t\}$,
  where $t\df\lvert C\rvert = 2^d\geq 2^{\rho\cdot m^k}$.

  Note that if we show that
  \begin{equation*}
    \PP_{\rn{x}\sim\mu^m}[(H_{w_1},\ldots,H_{w_t})\text{ is $\epsilon$-separated on $\rn{x}$ w.r.t.\ $\ell$}]
    >
    \delta,
  \end{equation*}
  then the proof is concluded as this is a contradiction with the probabilistic Haussler packing property guarantee as $m\geq
  m^{\hPHP[m^k]}_{\cH,\ell}(\epsilon,\delta,\rho)$.

  But indeed, for each $i\in[k]$ define the random element $\rn{\gamma}_i$ of $[n]^m$ by letting $\rn{\gamma}_i$ be the unique
  function $[m]\to[n]$ such that
  \begin{equation*}
    (\rn{x}_i)_j = (z_i)_{\rn{\gamma}_i(j)}
  \end{equation*}
  and note that since $\mu_i$ is the uniform distribution on $\{(z_i)_1,\ldots,(z_i)_n\}$, it follows that $\rn{\gamma}_i$ is
  uniformly distributed on $[n]^m$. It is also clear that the $\rn{\gamma}_i$ are mutually independent.

  Claim~\ref{clm:linearcode:partite} then says that with probability greater than $\delta$, we have
  \begin{equation}\label{eq:PHP->VCNkk:dist:partite}
    \dist_{\rn{\gamma}}(C) > \frac{\epsilon\cdot m^k}{s(\ell)}.
  \end{equation}

  But note that
  \begin{align*}
    \dist_{\rn{\gamma}}(C)
    & =
    \inf_{\substack{w,w'\in C\\w\neq w'}}
    \lvert\{j\in[m]^k \mid \rn{\gamma}^*(w)_j\neq \rn{\gamma}^*(w')_j\}\rvert
    \\
    & =
    \inf_{\substack{w,w'\in C\\w\neq w'}}
    \lvert\{\beta\in[m]^k \mid (H_w)^*_m(\rn{x})_\beta \neq (H_{w'})^*_m(\rn{x})_\beta\}\rvert
    \\
    & \leq
    \inf_{1\leq i < j\leq t}
    \frac{L_{\rn{x},(H_{w_i})^*_m(\rn{x}),\ell}(H_{w_j})\cdot m^k}{s(\ell)},
  \end{align*}
  so~\eqref{eq:PHP->VCNkk:dist:partite} implies that $(H_{w_1},\ldots,H_{w_t})$ is $\epsilon$-separated on $\rn{x}$
  w.r.t.\ $\ell$.

  It remains then to prove Claim~\ref{clm:linearcode:partite}.

  \begin{proofof}{Claim~\ref{clm:linearcode:partite}}
    Let $\rn{A}$ be a random $[n]^k\times[d]$-matrix with entries in $\FF_2$, picked uniformly at random (i.e., a uniformly at
    random element of $\FF_2^{[n]^k\times[d]}$) and let $\rn{C}\df\im(\rn{A})$ be the image of $\rn{A}$, which is clearly a
    (random) linear subspace of $\FF_2^{[n]^k}$ of dimension at most $d$.

    In fact, we can compute exactly the probability that the dimension of $\rn{C}$ is $d$ by simply counting in how many ways we
    can generate each row of $\rn{A}$ to not be in the span of the previous rows:
    \begin{equation*}
      \PP_{\rn{C}}\bigl[\dim_{\FF_2}(\rn{C}) = d\bigr]
      =
      2^{-d\cdot n^k}\prod_{j=0}^{d-1} (2^{n^k} - 2^j)
      =
      \prod_{j=0}^{d-1} (1 - 2^{j-n^k})
      \geq
      (1-2^{d-n^k})^d,
    \end{equation*}
    where the inequality follows since $d\leq n^k$.

    To prove the existence of the desired linear code, it then suffices to show that
    \begin{equation*}
      \PP_{\rn{C}}\left[
        \PP_{\rn{\gamma}}\left[
          \dist_{\rn{\gamma}}(\rn{C}) > \frac{\epsilon\cdot m^k}{s(\ell)}
          \right]
        >
        \delta
        \right]
      >
      1 - (1 - 2^{d-n^k})^d,
    \end{equation*}
    as we will then conclude (by union bound) that with positive probability $\rn{C}$ satisfies both the above and has dimension
    exactly $d$. Since the inner probability is at most $1$, by (reverse) Markov's Inequality, it suffices to show
    \begin{equation}\label{eq:linearcode:partite:goal}
      \begin{aligned}
        \EE_{\rn{C}}\left[
          \PP_{\rn{\gamma}}\left[
            \dist_{\rn{\gamma}}(\rn{C}) > \frac{\epsilon\cdot m^k}{s(\ell)}
            \right]
          \right]
        & >
        (1 - \delta)\cdot \bigl(1 - (1 - 2^{d-n^k})^d\bigr) + \delta
        \\
        & =
        1 - (1 - \delta)\cdot(1 - 2^{d-n^k})^d.
      \end{aligned}
    \end{equation}

    For each $i\in[k]$, let $E_i(\rn{\gamma}_i)$ be the event that $\rn{\gamma}_i$ has no repeated values (i.e., $\rn{\gamma}_i$
    is injective) and let $E(\rn{\gamma})$ be the conjunction of the $E_i(\rn{\gamma}_i)$. Note that
    \begin{align*}
      \PP_{\rn{\gamma}}\bigl[E(\rn{\gamma})\bigr]
      & =
      \prod_{i=1}^k \PP_{\rn{\gamma}_i}\bigl[E_i(\rn{\gamma}_i)\bigr]
      =
      \left(\frac{(n)_m}{n^m}\right)^k
      \\
      & >
      \left(1 - \frac{m}{n}\right)^{k\cdot m}
      >
      \left(1 - \frac{1}{m}\right)^{k\cdot m},
    \end{align*}
    where the last inequality follows since $n > m^2 > 0$.

    We now note that the left-hand side of our goal in~\eqref{eq:linearcode:partite:goal} can be bounded as:
    \begin{align*}
      \EE_{\rn{C}}\left[
        \PP_{\rn{\gamma}}\left[
          \dist_{\rn{\gamma}}(\rn{C}) > \frac{\epsilon\cdot m^k}{s(\ell)}
          \right]
        \right]
      & =
      \EE_{\rn{\gamma}}\left[
        \EE_{\rn{C}}\left[
          \One\left[
            \dist_{\rn{\gamma}}(\rn{C}) > \frac{\epsilon\cdot m^k}{s(\ell)}
            \right]
          \right]
        \right]
      \\
      & >
      \left(1 - \frac{1}{m}\right)^{k\cdot m}\cdot
      \EE_{\rn{\gamma}}\left[
        \EE_{\rn{C}}\left[
          \One\left[
            \dist_{\rn{\gamma}}(\rn{C}) > \frac{\epsilon\cdot m^k}{s(\ell)}
            \right]
          \right]
        \Given
        E(\rn{\gamma})
        \right].
    \end{align*}

    Thus, it suffices to show that for every fixed $\gamma$ in the event $E(\gamma)$, we have
    \begin{equation*}
      \PP_{\rn{C}}\left[
        \dist_\gamma(\rn{C}) > \frac{\epsilon\cdot m^k}{s(\ell)}
        \right]
      \geq
      \frac{1 - (1 - \delta)\cdot(1 - 2^{d-n^k})^d}{(1 - 1/m)^{k\cdot m}},
    \end{equation*}
    which in turn is equivalent to
    \begin{equation*}
      \PP_{\rn{C}}\left[
        \dist_\gamma(\rn{C}) \leq \frac{\epsilon\cdot m^k}{s(\ell)}
        \right]
      \leq
      1 - \frac{1 - (1 - \delta)\cdot(1 - 2^{d-n^k})^d}{(1 - 1/m)^{k\cdot m}}.
    \end{equation*}

    From the definition of $\rn{C}$, we know that the set $\rn{C}\setminus\{0\}$ is a subset\footnote{The only reason we say
    subset instead of equality is because we are \emph{not} restricting to the event in which $\rn{A}$ is full rank, so the set
    above might potentially have $0$.} of
    \begin{equation*}
      \{\rn{A}(z) \mid z\in\FF_2^{[d]}\setminus\{0\}\}.
    \end{equation*}

    By the union bound, it then suffices to show that for every $z\in\FF_2^{[d]}\setminus\{0\}$, we have\footnote{It would have
    been fine to put $2^d-1$ instead of $2^d$ in the denominator, but this leads to a slightly cleaner expression.}
    \begin{equation*}
      \PP_{\rn{A}}\left[
        \Bigl\lvert\gamma^*\bigl(\rn{A}(z)\bigr)^{-1}(1)\Bigr\rvert
        \leq
        \frac{\epsilon\cdot m^k}{s(\ell)}
        \right]
      \leq
      \frac{1}{2^d}\cdot\left(1 - \frac{1 - (1 - \delta)\cdot(1 - 2^{d-n^k})^d}{(1 - 1/m)^{k\cdot m}}\right),
    \end{equation*}

    Since $\rn{A}$ is picked uniformly at random in $\FF_2^{[n]^k\times[d]}$, for each fixed $z\in\FF_2^{[d]}\setminus\{0\}$, we
    know that $\rn{A}(z)$ is uniformly distributed on $\FF_2^{[n]^k}$, so the above is equivalent to
    \begin{equation*}
      \PP_{\rn{w}}\left[
        \lvert\gamma^*(\rn{w})^{-1}(1)\rvert
        \leq
        \frac{\epsilon\cdot m^k}{s(\ell)}
        \right]
      \leq
      \frac{1}{2^d}\cdot\left(1 - \frac{1 - (1 - \delta)\cdot(1 - 2^{d-n^k})^d}{(1 - 1/m)^{k\cdot m}}\right),
    \end{equation*}
    where $\rn{w}$ is picked uniformly at random in $\FF_2^{[n]^k}$.

    Since $\gamma$ is in the event $E(\gamma)$, it follows that the projection $\gamma^*$ is full rank; this means that the
    probability above is straightforward to compute: by counting how many ways $\rn{w}$ can project into a ball of radius
    $\epsilon\cdot m^k/s(\ell)$ around the origin (in $\FF_2^{[m]^k}$) and measuring the size of the kernel of $\gamma^*$; in
    formulas:
    \begin{equation*}
      \PP_{\rn{w}}\left[
        \bigl\lvert\gamma^*(\rn{w})^{-1}(1)\bigr\rvert
        \leq
        \frac{\epsilon\cdot m^k}{s(\ell)}
        \right]
      =
      \frac{1}{2^{n^k}}\cdot\left(\sum_{j=0}^{\floor{\epsilon\cdot m^k/s(\ell)}}\binom{m^k}{j}\right)\cdot 2^{n^k-m^k}
      \leq
      2^{(h_2(\epsilon/s(\ell))-1)\cdot m^k},
    \end{equation*}
    where the inequality is the standard upper bound on the size of the Hamming ball in terms of the binary entropy (see
    e.g.~\cite[Lemma~4.7.2]{Ash65}), using the fact that $\epsilon/s(\ell)\in(0,1/2)$ as $\epsilon\in(0,s(\ell)/2)$.

    Thus, it suffices to show that
    \begin{equation*}
      2^{(h_2(\epsilon/s(\ell))-1)\cdot m^k}
      <
      \frac{1}{2^d}\cdot\left(1 - \frac{1 - (1 - \delta)\cdot(1 - 2^{d-n^k})^d}{(1 - 1/m)^{k\cdot m}}\right),
    \end{equation*}
    which follows from the fact that
    \begin{equation*}
      n
      >
      \left(
      d
      - \log_2\left(
      1 - \left(\frac{1 - (1-1/m)^{k\cdot m}\cdot(1 - 2^{(h_2(\epsilon/s(\ell)) - 1)\cdot m^k + d})}{1-\delta}\right)^{1/d}
      \right)
      \right)^{1/k}
    \end{equation*}
    after a tedious but straightforward calculation.
  \end{proofof}

  \medskip

  We now prove the non-partite case. The proof is completely analogous, except for the following changes:
  \begin{itemize}
  \item The definition of $\VCN_{k,k}(\cH)\geq n$ in the non-partite is more complicated: it involves a point in
    $\cE_{kn}(\Omega)$ (as opposed to a point in $\cE_n(\Omega)$ in the partite) and not every $k$-subset of $[kn]$ contributes
    to the Natarajan-shattering, more precisely, the shattering happens exactly on the $k$-subsets in $T_{k,n}$.
  \item The structured projections $\gamma^*$ in the non-partite are of the form $\FF_2^{T_{k,n}}\to\FF_2^{\binom{m}{k}}$, where
    $T_{k,n}\subseteq\binom{[kn]}{k}$ is given by~\eqref{eq:Tkm} (as opposed to $\FF_2^{[n]^k}\to\FF_2^{[m]^k}$ in the partite);
    they also come from a single function $\gamma\colon[m]\to[kn]$ (as opposed to $k$ functions
    $\gamma_1,\ldots,\gamma_k\colon[m]\to[n]$ in the partite).
  \item Empirical losses are a (normalized) sum of $\binom{m}{k}$ terms (as opposed to $m^k$ terms in the partite), so all
    calculations have to change accordingly.
  \item Even though the set $T_{k,n}$ has a natural $k$-partition, sampling in the non-partite setting does not need to respect
    this partition; this means that in our calculation besides enforcing no repetition among the coordinates (by incurring some
    probability loss), we will also need to enforce that about $m/k$ points land on each of the parts of $T_{k,n}$ (incurring
    another probability loss).
  \end{itemize}

  First, we show that all calculations in~\eqref{eq:PHP->VCNkk:VCNkk} and~\eqref{eq:PHP->VCNkk:m} are valid.

  The condition $\delta\in(0,1/12)$ ensures that the logarithm in~\eqref{eq:PHP->VCNkk:m} is well-defined and the condition
  \begin{equation*}
    \rho\in\left(0,\frac{1 - h_2(\epsilon\cdot(2\cdot k)^k/(k!\cdot s(\ell)))}{(2\cdot k)^k}\right)
  \end{equation*}
  ensures that the denominator in~\eqref{eq:PHP->VCNkk:m} is positive, hence the $(1/k)$th power in~\eqref{eq:PHP->VCNkk:m} is
  also well-defined.

  Since $m\geq 8\cdot k\cdot\ln(4\cdot k) > 11$ and the function $(1-1/x)^x - k\cdot e^{-x/(8\cdot k)}$ is increasing for $x >
  1$, it follows that
  \begin{equation}\label{eq:between12ande}
    \frac{1}{12}
    \leq
    \left(1-\frac{1}{m}\right)^m - k\cdot e^{-m/(8\cdot k)}
    \leq
    \frac{1}{e}.
  \end{equation}
  This together with $d\geq\rho\cdot m^k$ means that in~\eqref{eq:PHP->VCNkk:VCNkk}, the expression under the $(1/d)$th power is
  at least
  \begin{equation*}
    \frac{
      1 - e^{-1}\cdot(1 - 2^{(h_2(\epsilon\cdot(2\cdot k)^k/(k!\cdot s(\ell))) - 1)(m/(2\cdot k))^k + \rho\cdot m^k})
    }{
      1-\delta
    },
  \end{equation*}
  which is non-negative since
  \begin{equation*}
    \rho\in\left(0,\frac{1 - h_2(\epsilon\cdot(2\cdot k)^k/(k!\cdot s(\ell)))}{(2\cdot k)^k}\right),
  \end{equation*}
  so the $(1/d)$th power is well-defined.

  Using the other inequality of~\eqref{eq:between12ande} and $d\leq\rho\cdot m^k + 1$, we also deduce that the expression under
  the logarithm in~\eqref{eq:PHP->VCNkk:VCNkk} is at least
  \begin{equation*}
    1
    - \left(\frac{
      1 - (1/12)\cdot(1 - 2^{(h_2(\epsilon\cdot(2\cdot k)^k/(k!\cdot s(\ell))) - 1)(m/(2\cdot k))^k + \rho\cdot m^k + 1})
    }{
      1-\delta
    }\right)^{1/d},
  \end{equation*}
  which is non-negative since
  \begin{equation*}
    m
    \geq
    2\cdot k\cdot
    \left(
    \frac{1 - \log_2(1 - 12\cdot\delta)}{1 - h_2(\epsilon\cdot(2\cdot k)^k/(k!\cdot s(\ell))) - \rho\cdot(2\cdot k)^k}
    \right)^{1/k}.
  \end{equation*}

  Thus, all expressions in~\eqref{eq:PHP->VCNkk:VCNkk} and~\eqref{eq:PHP->VCNkk:m} are well-defined.

  Furthermore, note that the minimum on the right-hand side of~\eqref{eq:PHP->VCNkk:VCNkk} is indeed attained due to the
  floor. Let then $(\epsilon,\delta,\rho)$ attain the minimum in~\eqref{eq:PHP->VCNkk:VCNkk} (and let $m$ and $d$ be defined as
  in~\eqref{eq:PHP->VCNkk:m} and~\eqref{eq:PHP->VCNkk:d}, respectively). Let $n\df\VCN_{k,k}(\cH)$ and suppose for a
  contradiction that
  \begin{multline*}
    n
    >
    \max\biggggl\{
    \frac{m^2}{k},
    \Bigggl(
    d
    - \log_2\bigggl(
    1
    \\
    - \left(\frac{
      1 - ((1-1/m)^m - k\cdot e^{-m/(8\cdot k)})\cdot(1 - 2^{(h_2(\epsilon\cdot(2\cdot k)^k/(k!\cdot s(\ell))) - 1)(m/(2\cdot k))^k + d})
    }{
      1-\delta
    }\right)^{1/d}
    \bigggr)
    \Bigggr)^{1/k}
    \biggggr\}.
  \end{multline*}
  (Note that we removed the floor as $n$ is an integer.)

  As per definition of $\VCN_{k,k}(\cH)$ in Defininition~\ref{def:SC:VCNkk}, let $z\in\cE_n(\Omega)$ be such that
  \begin{equation*}
    \cH_z \df \{H_z \mid H\in\cH\}\subseteq (\Lambda^{S_k})^{T_{k,n}}
  \end{equation*}
  Natarajan-shatters $T_{k,n}$, where
  \begin{gather*}
    H_z(U)_\tau \df H^*_{kn}(z)_{\iota_{U,kn}\comp\tau} \qquad (U\in T_{k,n}, \tau\in S_k),
    \\
    T_{k,n} \df \left\{U\in\binom{[kn]}{k} \;\middle\vert\; \lvert U\cap[(i-1)m+1, im]\rvert=1\right\}.
  \end{gather*}
  It will be convenient to index our witnesses to the shattering by $\FF_2^{T_{k,n}}$. Namely, we know that there exist
  $f_0,f_1\colon T_{k,n}\to\Lambda^{S_k}$ with $f_0(U)\neq f_1(U)$ for every $U\in T_{k,n}$ and $H^w\in\cH$
  ($w\in\FF_2^{T_{k,n}}$) such that for every $w\in\FF_2^{T_{k,n}}$ and every $U\in T_{k,n}$, we have $H^w_z(U) = f_{w_U}(U)$.

  Our goal is to show that there exists $C\subseteq\FF_2^{T_{k,n}}$ of size at least $2^{\rho\cdot m^k}$ and a probability
  template $\mu\in\Pr(\Omega)$ such that if $C = \{w_1,\ldots,w_{\lvert C\rvert}\}$ and $\rn{x}\sim\mu^m$, then
  $(H^{w_1},\ldots,H^{w_{\lvert C\rvert}})$ is $\epsilon$-separated on $\rn{x}$ with probability larger than $\delta$.

  Again, we will find a linear code $C\subseteq\FF_2^{T_{k,n}}$ with this property. For this, we define a ``structured
  projection'' as follows: given $\gamma\colon[m]\to[kn]$, we define a function
  $\gamma^*\colon\FF_2^{T_{k,n}}\to\FF_2^{\binom{[m]}{k}}$ given by
  \begin{equation*}
    \gamma^*(w)_U \df
    \begin{dcases*}
      w_{\gamma(U)}, & if $\gamma(U)\in T_{k,n}$,\\
      0, & otherwise.
    \end{dcases*}
  \end{equation*}
  Clearly $\gamma^*$ is a linear map. For a linear code $C\subseteq\FF_2^{T_{k,n}}$, define
  \begin{align*}
    \dist_\gamma(C)
    & \df
    \inf_{\substack{w_1,w_2\in C\\ w_1\neq w_2}}
    \left\lvert\left\{U\in\binom{[m]}{k} \;\middle\vert\; \gamma^*(w_1)_U\neq \gamma^*(w_2)_U\right\}\right\rvert
    \\
    & =
    \inf_{w\in C\setminus\{0\}}
    \lvert\gamma^*(w)^{-1}(1)\rvert
    \\
    & =
    \begin{dcases*}
      \dist(\gamma^*(C)), & if $\gamma^*$ is injective on $C$,\\
      0, & otherwise.
    \end{dcases*}
  \end{align*}

  We will show that a uniformly random linear code $\rn{C}$ of dimension $d$ is such that for most $\gamma$, we have
  $\dist_\gamma(\rn{C}) > \epsilon\cdot m^k/s(\ell)$.

  \begin{claim}\label{clm:linearcode}
    There exists a linear code $C\subseteq\FF_2^{T_{k,n}}$ of dimension $d\df\ceil{\rho\cdot m^k}$ such that if $\rn{\gamma}$ is
    a uniformly at random function $[m]\to[kn]$, then
    \begin{equation*}
      \PP_{\rn{\gamma}}\left[
        \dist_{\rn{\gamma}}(C) > \epsilon\cdot\frac{\binom{m}{k}}{s(\ell)}
        \right]
      >
      \delta.
    \end{equation*}
  \end{claim}

  Before proving the claim, let us use it to finish the proof.

  First note that since $\cH_z$ Natarajan-shatters $T_{k,n}$ and $\bigcup T_{k,n} = [kn]$, there cannot be repetitions among the
  coordinates of $z$. We can then define $\mu\in\Pr(\Omega)$ as the uniform measure on the set
  \begin{equation*}
    \{z_1,\ldots,z_{kn}\}
  \end{equation*}
  (which has size exactly $kn$).

  Let $C\subseteq\FF_2^{T_{k,n}}$ be as in Claim~\ref{clm:linearcode} and enumerate its elements as $C = \{w_1,\ldots,w_t\}$,
  where $t\df\lvert C\rvert = 2^d$, so $t = 2^d\geq 2^{\rho\cdot m^k}$.

  Note that if we show that there exists an order choice $\alpha$ for $[m]$ such that
  \begin{equation*}
    \PP_{\rn{x}\sim\mu^m}[(H^{w_1},\ldots,H^{w_t})\text{ is $\epsilon$-separated on $\rn{x}$ w.r.t.\ $\ell$ and $\alpha$}]
    >
    \delta,
  \end{equation*}
  then the proof is concluded as this is a contradiction with the probabilistic Haussler packing property guarantee as $m\geq
  m^{\hPHP[m^k]}_{\cH,\ell}(\epsilon,\delta,\rho)$.

  We will show that the above in fact holds for every order choice $\alpha$ for $[m]$.

  Define the random element $\rn{\gamma}$ of $[kn]^m$ by letting $\rn{\gamma}$ be the unique function $[m]\to[kn]$ such that
  \begin{equation*}
    \rn{x}_i = z_{\rn{\gamma}(i)}
  \end{equation*}
  and note that since $\mu$ is the uniform distribution on $\{z_1,\ldots,z_{kn}\}$, it follows that $\rn{\gamma}$ is uniformly
  distributed on $[kn]^m$. Note that the above is equivalent to $\rn{x} = \rn{\gamma}^*(z)$. In particular, from
  equivariance~\eqref{eq:F*Vequiv}, it follows that for every $H\in\cH$, we have
  \begin{equation}\label{eq:H*mgamma}
    H^*_m(\rn{x})
    =
    H^*_m\bigl(\rn{\gamma}^*(z)\bigr)
    =
    \rn{\gamma}^*\bigl(H^*_{kn}(z)\bigr).
  \end{equation}

  We now claim that for every $w,w'\in\FF_2^{T_{k,n}}$ and every $U\in\binom{[m]}{k}$, we have
  \begin{equation}\label{eq:gammatobalpha}
    \rn{\gamma}^*(w)_U\neq\rn{\gamma}^*(w')_U
    \implies
    b_\alpha\bigl((H^w)^*_m(\rn{x})\bigr)_U \neq b_\alpha\bigl((H^{w'})^*_m(\rn{x})\bigr)_U.
  \end{equation}
  Indeed, since $\rn{\gamma}^*(w)_U\neq\rn{\gamma}^*(w')_U$, we must have $\rn{\gamma}(U)\in T_{k,n}$. On the other hand, for
  every $\tau\in S_k$, we have
  \begin{align*}
    \Bigl(b_\alpha\bigl((H^w)^*_m(\rn{x})\bigr)_U\Bigr)_\tau
    & =
    (H^w)^*_m(\rn{x})_{\alpha_U\comp\tau}
    =
    \rn{\gamma}^*\bigl((H^w)^*_{kn}(z)\bigr)_{\alpha_U\comp\tau}
    =
    (H^w)^*_{kn}(z)_{\rn{\gamma}\comp\alpha_U\comp\tau}
    \\
    & =
    (H^w)^*_{kn}(z)_{\iota_{\rn{\gamma}(U),kn}\comp\iota_{\rn{\gamma}(U),kn}^{-1}\comp\rn{\gamma}\comp\alpha_U\comp\tau}
    =
    H^w_z\bigl(\rn{\gamma}(U)\bigr)_{\iota_{\rn{\gamma}(U),kn}^{-1}\comp\rn{\gamma}\comp\alpha_U\comp\tau}
    \\
    & =
    f_{w_{\rn{\gamma}(U)}}\bigl(\rn{\gamma}(U)\bigr)_{\iota_{\rn{\gamma}(U),kn}^{-1}\comp\rn{\gamma}\comp\alpha_U\comp\tau}
    =
    f_{\rn{\gamma}^*(w)_U}\bigl(\rn{\gamma}(U)\bigr)_{\iota_{\rn{\gamma}(U),kn}^{-1}\comp\rn{\gamma}\comp\alpha_U\comp\tau}
  \end{align*}

  Since an analogous computation holds for $w'$ and since $\rn{\gamma}^*(w)_U\neq\rn{\gamma}^*(w')_U$ and $f_0(V)\neq f_1(V)$
  for every $V\in T_{k,n}$, we conclude that
  \begin{equation*}
    b_\alpha\bigl((H^w)^*_m(\rn{x})\bigr)_U
    \neq
    b_\alpha\bigl((H^{w'})^*_m(\rn{x})\bigr)_U,
  \end{equation*}
  as desired.

  Now Claim~\ref{clm:linearcode} then says that with probability greater than $\delta$, we have
  \begin{equation}\label{eq:PHP->VCNkk:dist}
    \dist_{\rn{\gamma}}(C) > \epsilon\cdot\frac{\binom{m}{k}}{s(\ell)}.
  \end{equation}

  Since
  \begin{align*}
    \dist_{\rn{\gamma}}(C)
    & =
    \inf_{\substack{w,w'\in C\\ w\neq w'}}
    \left\lvert\left\{U\in\binom{[m]}{k} \;\middle\vert\;
    \rn{\gamma}^*(w)_U\neq\rn{\gamma}^*(w')_U\right\}\right\rvert
    \\
    & \leq
    \inf_{\substack{w,w'\in C\\ w\neq w'}}
    \left\lvert\left\{U\in\binom{[m]}{k} \;\middle\vert\;
    b_\alpha\bigl((H^w)^*_m(\rn{x})\bigr)_U
    \neq
    b_\alpha\bigl((H^{w'})^*_m(\rn{x})\bigr)_U
    \right\}\right\rvert
    \\
    & \leq
    \frac{\binom{m}{k}}{s(\ell)}\cdot
    \inf_{1\leq i < j\leq t} L_{\rn{x},(H^{w_i})^*_m(\rn{x}),\ell}(H^{w_j}),
  \end{align*}
  where the first inequality follows from~\eqref{eq:gammatobalpha}. Thus,~\eqref{eq:PHP->VCNkk:dist} implies that
  $(H^{w_1},\ldots,H^{w_t})$ is $\epsilon$-separated on $\rn{x}$ w.r.t.\ $\ell$ and $\alpha$.

  It remains then to prove Claim~\ref{clm:linearcode}.

  \begin{proofof}{Claim~\ref{clm:linearcode}}
    The initial setup is analogous to the one of Claim~\ref{clm:linearcode:partite}: let $\rn{A}$ be a random
    $T_{k,n}\times[d]$-matrix with entries in $\FF_2$, picked uniformly at random (i.e., a uniformly at random element of
    $\FF_2^{T_{k,n}\times[d]}$) and let $\rn{C}\df\im(\rn{A})$ be the image of $\rn{A}$, which is clearly a (random) linear
    subspace of $\FF_2^{T_{k,n}}$ of dimension at most $d$.

    In fact, since $\lvert T_{k,n}\rvert = n^k$, the probability that the dimension of $\rn{C}$ is exactly $d$ is
    \begin{equation*}
      \PP_{\rn{C}}[\dim_{\FF_2}(\rn{C}) = d]
      =
      2^{-d\cdot n^k}\prod_{j=0}^{d-1} (2^{n^k} - 2^j)
      =
      \prod_{j=0}^{d-1} (1 - 2^{j-n^k})
      \geq
      (1 - 2^{d-n^k})^d,
    \end{equation*}
    where the inequality follows since $d\leq n^k$.

    To prove the existence of the desired linear code, it suffices to show that
    \begin{equation*}
      \PP_{\rn{C}}\left[
        \PP_{\rn{\gamma}}\left[
          \dist_{\rn{\gamma}(\rn{C})} > \epsilon\cdot\frac{\binom{m}{k}}{s(\ell)}
          \right]
        >
        \delta
        \right]
      >
      1 - (1 - 2^{d-n^k})^d
    \end{equation*}
    as then the union bound shows that with positive probability $\rn{C}$ satisfies both the above and has dimension exactly
    $d$. Since the inner probability is at most $1$, by (reverse) Markov's Inequality, it suffices to show
    \begin{equation}\label{eq:linearcode:goal}
      \begin{aligned}
        \EE_{\rn{C}}\left[
          \PP_{\rn{\gamma}}\left[
            \dist_{\rn{\gamma}}(\rn{C}) > \epsilon\cdot\frac{\binom{m}{k}}{s(\ell)}
            \right]
          \right]
        & >
        (1 - \delta)\cdot \bigl(1 - (1 - 2^{d-n^k})^d\bigr) + \delta
        \\
        & =
        1 - (1 - \delta)\cdot(1 - 2^{d-n^k})^d.
      \end{aligned}
    \end{equation}

    This is the first point of meaningful divergence of this claim from its partite counterpart: for each $i\in[k]$, let
    $E'_i(\rn{\gamma})$ be the event that
    \begin{equation*}
      \biggl\lvert\rn{\gamma}^{-1}\Bigl(\bigl[(i-1)\cdot n + 1, i\cdot n\bigr]\Bigr)\biggr\rvert
      \geq
      \frac{m}{2\cdot k},
    \end{equation*}
    i.e., the event that at least $m/(2\cdot k)$ entries of $\rn{\gamma}$ are in $[(i-1)\cdot n + 1,i\cdot n]$. Since
    $\lvert\im(\rn{\gamma})\cap [(i-1)\cdot n + 1, i\cdot n]\rvert$ has binomomial distribution $\Bi(m,1/k)$, by Chernoff's
    Bound, we have
    \begin{equation*}
      \PP_{\rn{\gamma}}\bigl[E'_i(\rn{\gamma})\bigr]
      =
      \PP_{\rn{\gamma}}\left[
        \Bi\left(m,\frac{1}{k}\right)
        \geq
        \left(1 - \frac{1}{2}\right)\cdot\frac{m}{k}
      \right]
      \geq
      1 - \exp\left(-\frac{m}{8\cdot k}\right).
    \end{equation*}

    In particular, if $E'(\rn{\gamma})$ is the conjunction of the events $E'_i(\rn{\gamma})$, then the union bound gives
    \begin{equation*}
      \PP_{\rn{\gamma}}\bigl[E'(\rn{\gamma})\bigr]
      \geq
      1 - k\cdot\exp\left(-\frac{m}{8\cdot k}\right).
    \end{equation*}

    Let also $E''(\rn{\gamma})$ be the event that $\rn{\gamma}$ has no repeated values (i.e., $\rn{\gamma}$ is injective) and
    let $E(\rn{\gamma})$ be the conjunction of $E'(\rn{\gamma})$ and $E''(\rn{\gamma})$. Note that
    \begin{equation*}
      \PP_{\rn{\gamma}}\bigl[E''(\rn{\gamma})\bigr]
      =
      \frac{(kn)_m}{(kn)^m}
      >
      \left(1 - \frac{m}{kn}\right)^m
      >
      \left(1 - \frac{1}{m}\right)^m,
    \end{equation*}
    where the last inequality follows since $n > m^2/k > 0$. Thus, by the union bound, we get
    \begin{equation*}
      \PP_{\rn{\gamma}}\bigl[E(\rn{\gamma})\bigr]
      \geq
      \left(1 - \frac{1}{m}\right)^m - k\cdot\exp\left(-\frac{m}{8\cdot k}\right).
    \end{equation*}

    Using these events and probability estimates, the left-hand side of our goal in~\eqref{eq:linearcode:goal} can be bounded
    as:
    \begin{align*}
      \MoveEqLeft
      \EE_{\rn{C}}\left[
          \PP_{\rn{\gamma}}\left[
            \dist_{\rn{\gamma}}(\rn{C}) > \frac{\epsilon\cdot\binom{m}{k}}{s(\ell)}
            \right]
          \right]
      \\
      & =
      \EE_{\rn{\gamma}}\left[
        \EE_{\rn{C}}\left[
          \One\left[
            \dist_{\rn{\gamma}}(\rn{C}) > \frac{\epsilon\cdot\binom{m}{k}}{s(\ell)}
            \right]
          \right]
        \right]
      \\
      & >
      \left(
      \left(1 - \frac{1}{m}\right)^m - k\cdot\exp\left(-\frac{m}{8\cdot k}\right)
      \right)\cdot
      \EE_{\rn{\gamma}}\left[
        \EE_{\rn{C}}\left[
          \One\left[
            \dist_{\rn{\gamma}}(\rn{C}) > \frac{\epsilon\cdot\binom{m}{k}}{s(\ell)}
            \right]
          \right]
        \Given
        E(\rn{\gamma})
        \right].
    \end{align*}

    Thus, it suffices to show that for every fixed $\gamma$ in the event $E(\gamma)$, we have
    \begin{equation*}
      \PP_{\rn{C}}\left[
        \dist_\gamma(\rn{C}) > \frac{\epsilon\cdot\binom{m}{k}}{s(\ell)}
        \right]
      \geq
      \frac{1 - (1 - \delta)\cdot(1 - 2^{d-n^k})^d}{(1 - 1/m)^m - k\cdot\exp(-m/(8\cdot k))},
    \end{equation*}
    which in turn is equivalent to
    \begin{equation*}
      \PP_{\rn{C}}\left[
        \dist_\gamma(\rn{C}) \leq \frac{\epsilon\cdot\binom{m}{k}}{s(\ell)}
        \right]
      \leq
      1 - \frac{1 - (1 - \delta)\cdot(1 - 2^{d-n^k})^d}{(1 - 1/m)^m - k\cdot\exp(-m/(8\cdot k))}.
    \end{equation*}

    Since $\rn{C}\setminus\{0\}$ is a subset of $\{\rn{A}(z) \mid z\in\FF_2^{[d]}\setminus\{0\}\}$, by the union bound, it
    suffices to show that\footnote{Similarly to the partite case, we say subset instead of equality since $\rn{A}$ might not be
    full rank and it would have been perfectly fine to put $2^d-1$ instead of $2^d$ in the denominator, but this leads to a
    cleaner expression.}
    \begin{equation*}
      \PP_{\rn{A}}\left[
        \Bigl\lvert\gamma^*\bigl(\rn{A}(z)\bigr)^{-1}(1)\Bigr\rvert
        \leq
        \frac{\epsilon\cdot\binom{m}{k}}{s(\ell)}
        \right]
      \leq
      \frac{1}{2^d}\cdot\left(1 - \frac{1 - (1 - \delta)\cdot(1 - 2^{d-n^k})^d}{(1 - 1/m)^m - k\cdot\exp(-m/(8\cdot k))}\right).
    \end{equation*}

    Since $\rn{A}$ is picked uniformly at random in $\FF_2^{T_{k,n}\times[d]}$, for each fixed $z\in\FF_2^{[d]}\setminus\{0\}$,
    we know that $\rn{A}(z)$ is uniformly distributed on $\FF_2^{T_{k,n}}$, so we can replace $\rn{A}(z)$ in the above with
    $\rn{w}$ picked uniformly at random in $\FF_2^{T_{k,n}}$.

    Note that the fact that $\gamma$ is in the event $E(\gamma)$ implies it is injective, hence
    \begin{align*}
      \lvert\gamma^*(\rn{w})^{-1}(1)\rvert
      & =
      \left\lvert\left\{U\in\binom{[m]}{k}
      \;\middle\vert\;
      \gamma(U)\in T_{k,n}
      \land
      \rn{w}_{\gamma(U)}=1
      \right\}\right\rvert
      \\
      & =
      \left\lvert\left\{U\in T_{k,n}\cap\binom{\im(\gamma)}{k}
      \;\middle\vert\;
      \rn{w}_U = 1
      \right\}\right\rvert
      \\
      & =
      \left\lvert T_{k,n}\cap\binom{\im(\gamma)}{k}\cap\rn{w}^{-1}(1)\right\rvert.
    \end{align*}

    On the other hand, the fact that $\gamma$ is in the event $E(\gamma)$ also implies
    \begin{equation*}
      \lvert\gamma([(i-1)\cdot n + 1, i\cdot n])\rvert
      \geq
      \frac{m}{2\cdot k}
    \end{equation*}
    for every $i\in[k]$. Letting $r\df\lvert T_{k,n}\cap\binom{\im(\gamma)}{k}\rvert$, we get $r\geq (m/(2\cdot k))^k$.

    Letting $\rn{z}$ be the restriction of $\rn{w}$ to $T_{k,n}\cap\binom{\im(\gamma)}{k}$, we note that $\rn{z}$ is
    uniformly distributed on $\FF_2^{T_{k,n}\cap\binom{\im(\gamma)}{k}}$ (as $\rn{w}$ is uniformly distributed on
    $\FF_2^{T_{k,n}}$), so we get
    \begin{align*}
      \PP_{\rn{A}}\left[
        \Bigl\lvert\gamma^*\bigl(\rn{A}(z)\bigr)^{-1}(1)\Bigr\rvert
        \leq
        \frac{\epsilon\cdot\binom{m}{k}}{s(\ell)}
        \right]
      & =
      \PP_{\rn{z}}\left[
        \lvert\rn{z}^{-1}(1)\rvert
        \leq
        \frac{\epsilon\cdot\binom{m}{k}}{s(\ell)}
        \right]
      =
      \frac{1}{2^r}\cdot\sum_{j=0}^{\floor{\epsilon\cdot\binom{m}{k}/s(\ell)}}\binom{r}{j}
      \\
      & \leq
      2^{(h_2(\epsilon\cdot\binom{m}{k}/(s(\ell)\cdot r))-1)\cdot r}
      \leq
      2^{(h_2(\epsilon\cdot(2\cdot k)^k/(k!\cdot s(\ell)))-1)\cdot (m/(2\cdot k))^k}
    \end{align*}
    where the first inequality is the standard upper bound on the size of the Hamming ball in terms of the binary entropy (see
    e.g.~\cite[Lemma~4.7.2]{Ash65}), using
    \begin{equation*}
      \frac{\epsilon\cdot\binom{m}{k}}{s(\ell)}
      \leq
      \frac{\epsilon\cdot m^k}{k!\cdot s(\ell)}
      =
      \frac{\epsilon\cdot (2\cdot k)^k}{k!\cdot s(\ell)}\cdot r
      <
      \frac{1}{2}\cdot r,
    \end{equation*}
    where the last inequality follows since $\epsilon\in(0,k!\cdot s(\ell)/(2\cdot(2\cdot k)^k))$.

    Thus, it suffices to show that
    \begin{equation*}
      2^{(h_2(\epsilon\cdot(2\cdot k)^k/(k!\cdot s(\ell)))-1)\cdot (m/(2\cdot k))^k}
      <
      \frac{1}{2^d}\cdot\left(1 - \frac{1 - (1 - \delta)\cdot(1 - 2^{d-n^k})^d}{(1 - 1/m)^m - k\cdot\exp(-m/(8\cdot k))}\right),
    \end{equation*}
    which follows from the fact that
    \begin{multline*}
      n
      >
      \Bigggl(
      d
      - \log_2\bigggl(
      1
      \\
      - \left(\frac{
        1 - ((1-1/m)^m - k\cdot e^{-m/(8\cdot k)})\cdot(1 - 2^{(h_2(\epsilon\cdot(2\cdot k)^k/(k!\cdot s(\ell))) - 1)(m/(2\cdot k))^k + d})
      }{
        1-\delta
      }\right)^{1/d}
      \bigggr)
      \Bigggr)^{1/k},
    \end{multline*}
    after a tedious but straightforward calculation.
  \end{proofof}

  This concludes the non-partite case.
\end{proof}

\end{document}